\documentclass[twoside,11pt]{article}

\usepackage[preprint]{jmlr2e}

\usepackage[utf8]{inputenc} % allow utf-8 input
\usepackage[T1]{fontenc}    % use 8-bit T1 fonts
\usepackage{hyperref}       % hyperlinks
\usepackage{url}            % simple URL typesetting
\usepackage{booktabs}       % professional-quality tables
\usepackage{amsmath,amssymb,amsfonts}       % blackboard math symbols
\usepackage{nicefrac}       % compact symbols for 1/2, etc.
\usepackage{microtype}      % microtypography
\usepackage{algorithm,algorithmic}
\usepackage{subcaption}
\usepackage{bbm}
\usepackage{wrapfig}

\newcommand{\reals}{\mathbb{R}}

\newcommand{\simplex}{\Delta}

\newcommand{\ball}{\mathcal{B}}

\newcommand{\Proj}{\operatorname{Proj}}

\newcommand{\dom}{\operatorname{dom}}

\newcommand{\Ord}{\mathcal{O}}

\newcommand{\dx}{\mathrm{d}}

\newcommand{\expect}{\operatorname{\mathbf{E}}}
\newcommand{\supp}{\operatorname{supp}}

\newcommand{\bigbrace}[1]{\left\{\begin{array}{lr} #1 \end{array} \right.}
\newcommand{\otherwise}{\text{otherwise}}

\DeclareMathOperator*{\argmin}{argmin}
\DeclareMathOperator*{\argmax}{argmax}

\DeclareMathOperator*{\minimize}{minimize}
\DeclareMathOperator*{\maximize}{maximize}

\newcommand{\Bset}{\mathcal{B}}

\newcommand{\Mset}{\mathcal{M}}

\newcommand{\Qset}{\mathcal{Q}}

\newcommand{\pvec}{\mathbf{p}}
\newcommand{\Ppr}{\mathbb{P}}
\newcommand{\Pset}{\mathcal{P}}

\newcommand{\Sspace}{\mathcal{S}}

\newcommand{\Uset}{\mathcal{U}}

\newcommand{\xvec}{\mathbf{x}}
\newcommand{\Xspace}{\mathcal{X}}
\newcommand{\Xrv}{X}

\newcommand{\yvec}{\mathbf{y}}
\newcommand{\Yspace}{\mathcal{Y}}
\newcommand{\Yrv}{Y}

\newcommand{\zvec}{\mathbf{z}}
\newcommand{\Zspace}{\mathcal{Z}}
\newcommand{\Zrv}{Z}

\newcommand{\Wass}{\mathcal{W}}

\usepackage{color}

\ShortHeadings{Incorporating Unlabeled Data into Distributionally-Robust Learning}{Frogner, Claici, Chien, and Solomon}
\firstpageno{1}

\begin{document}

\title{Incorporating Unlabeled Data into\\Distributionally-Robust Learning}

\author{\name Charlie Frogner \email frogner@mit.edu
       \AND
       \name Sebastian Claici \email sclaici@mit.edu
       \AND
       \name Edward Chien \email edchien@mit.edu
       \AND
       \name Justin Solomon \email jsolomon@mit.edu \\
       \addr Computer Science \& Artificial Intelligence Laboratory (CSAIL)\\
       Massachusetts Institute of Technology\\
       Cambridge, MA 02139, USA}

\editor{Kevin Murphy and Bernhard Sch{\"o}lkopf}

\maketitle

\begin{abstract}
We study a robust alternative to empirical risk minimization called distributionally robust learning (DRL), in which one learns to perform against an adversary who can choose the data distribution from a specified set of distributions. We illustrate a problem with current DRL formulations, which rely on an overly broad definition of allowed distributions for the adversary, leading to learned classifiers that are unable to predict with any confidence. We propose a solution that incorporates unlabeled data into the DRL problem to further constrain the adversary. We show that this new formulation is tractable for stochastic gradient-based optimization and yields a computable guarantee on the future performance of the learned classifier, analogous to---but tighter than---guarantees from conventional DRL. We examine the performance of this new formulation on $14$ real datasets and find that it often yields effective classifiers with nontrivial performance guarantees in situations where conventional DRL produces neither. Inspired by these results, we extend our DRL formulation to active learning with a novel, distributionally-robust version of the standard model-change heuristic.  Our active learning algorithm often achieves superior learning performance to the original heuristic on real datasets.
\end{abstract}


\section{Introduction}

Human learning is robust in ways that statistical learning struggles to replicate. Small changes to image pixel values and audio waveforms, for example, can dramatically alter the outputs of classifiers trained by conventional empirical risk minimization, while remaining imperceptible to human observers \citep{szegedy2013intriguing,carlini2018audio}. Robustness to artificial and natural variations, however, is critical when learning systems are deployed ``in the wild,'' such as in self-driving vehicles \citep{huval2015empirical,bojarski2016end} and speech recognition systems \citep{junqua2012robustness,hannun2014deep}. Hence, the design of robust learning techniques is a key focus of recent machine learning research \citep{eykholt2017robust,madry2017towards,raghunathan2018semidefinite,singh2018fast,sinha2017certifying,cohen2019certified,yuan2019adversarial}.

Distributionally robust learning (DRL) \citep{delage2010distributionally,abadeh2015distributionally,chen2018robust} offers an alternative to empirical risk minimization in which one learns to perform against an adversary who chooses the data distribution from a specified set of distributions. This approach offers several benefits, including robust performance with respect to perturbations of the data distribution and computable guarantees on the generalization of the learned model---provided the adversary's decision set includes the true data distribution.

The robustness guarantees offered by DRL rely on selection of the adversary's decision set; if the set does not include the true data distribution, the guarantees do not necessarily hold. Most previous work has chosen the decision set to be a norm ball around the empirical distribution of the training data \citep{abadeh2015distributionally,chen2018robust,esfahani2018data,sinha2017certifying}. As we show in Section \ref{sec:empirical-performance-of-drl-with-unlabeled}, however, in many cases this ball must be extremely large to contain the true data distribution. As a result, the distributionally-robust learner attempts to be robust to an overly broad set of data distributions, preventing it from making a prediction with any confidence. As a result, it can do no better than assigning equal probability to all of the classes.

In this paper, we address the problem of overwhelmingly-large decision sets by using unlabeled data to further constrain the adversary. In essence, we can remove from the decision set distributions that are ``unrealistic'' in the sense that their marginals in feature space do not resemble the unlabeled data. With a smaller decision set, the distributionally-robust learner can provide a tighter bound on the generalization performance, yielding nontrivial predictors with non-vacuous performance guarantees in situations where conventional DRL offers neither.

Our mechanism for optimizing against a adversary constrained by unlabeled data is general-purpose and applicable beyond supervised learning. We use this same mechanism to formulate a novel distributionally-robust method for active learning; this method frequently outperforms both uniform random sampling and standard methods for active learning.

\section{Background}

\subsection{Notation}

For any Polish space $\Sspace$, we use $\Bset(\Sspace)$ to denote the associated Borel $\sigma$-algebra and $\Mset(\Sspace)$ to denote set of Radon measures on $(\Sspace, \Bset(\Sspace))$. $\Mset_+(\Sspace)$ is the set of nonnegative Radon measures on $(\Sspace, \Bset(\Sspace))$, and $\Qset(\Sspace)$
is the set of probability measures: $\Qset(\Sspace) = \{\pi \in \Mset_+(\Sspace) : \pi(\Sspace) = 1\}$. $C_b(\Sspace)$ is the set of continuous, bounded functions from $\Sspace$ into $\reals$.

\subsection{Statistical learning}

Let $\Xspace$ be an input space and $\Yspace$ a label space, and let $\Ppr$ be the true data distribution, a probability measure over $\Zspace = \Xspace \times \Yspace$. We focus on a classification setting, in which $\Yspace = \{\yvec^k\}_{k=1}^{N_{\Yspace}}$ is a finite collection of discrete labels, while $\Xspace$ can be any compact Polish space. The learning problem chooses a hypothesis $h_{\theta} : \Xspace \rightarrow \Qset(\Yspace)$, parameterized by $\theta \in \Theta \subseteq \reals^q$, that minimizes the expected risk, $\expect^{\Ppr} \ell(h_{\theta}(\Xrv), \Yrv)$, where $\ell : \Qset(\Yspace) \times \Yspace \rightarrow \reals$ is a loss function measuring deviation of the prediction $h_{\theta}(\Xrv)$ from the true label $\Yrv$.\footnote{We will always write $(\Xrv, \Yrv)$ for the pair of random variables $\Xrv : \Omega \rightarrow \Xspace$ and $\Yrv : \Omega \rightarrow \Yspace$ over which we are taking the expectation. Their distributions are to be understood from the context.}

We cannot directly evaluate the expected risk, however, since $\Ppr$ is unknown. 
We instead have a labeled sample $\hat{\Zspace}_l = \{(\xvec_l^i, \yvec_l^i)\}_{i=1}^{N_l} \subset \Xspace \times \Yspace$ consisting of $N_l$ i.i.d.\ samples from $\Ppr$. If $\hat{\Ppr}_l = \frac{1}{N_l} \sum_{i=1}^{N_l} \delta_{(\xvec_l^i, \yvec_l^i)}$ is the empirical distribution of the labeled data, traditional empirical risk minimization substitutes $\hat{\Ppr}_l$ for $\Ppr$ in the statistical learning problem, solving 
\begin{equation}
\label{eq:empirical-risk-minimization}
\minimize_{\theta \in \Theta}\ \expect^{\hat{\Ppr}_l} \ell(h_{\theta}(\Xrv), \Yrv) = \frac{1}{N_l} \sum_{i=1}^{N_l} \ell(h_{\theta}(\xvec_l^i), \yvec_l^i).
\end{equation}
To reduce variance of this approximation and promote generalization, often a regularization term (e.g., penalizing model complexity) is added to the loss.

\subsection{Distributional robustness}

Distributionally-robust learning (DRL) \citep{delage2010distributionally,abadeh2015distributionally,chen2018robust} is an alternative to empirical risk minimization that attempts to learn a predictor with minimal worst-case expected risk, against an adversary who chooses the distribution of the data from a specified decision set $\Pset$:
\begin{equation}
\label{eq:distributionally-robust-learning}
\minimize_{\theta \in \Theta} \sup_{\mu \in \Pset} \expect^{\mu} \ell(h_{\theta}(\Xrv), \Yrv).
\end{equation}
$\Pset$ is typically a norm ball centered at the empirical distribution of the labeled data $\hat{\Ppr}_l$. If $\Pset$ is chosen such that it contains the true data distribution $\Ppr$, the objective in \eqref{eq:distributionally-robust-learning} upper-bounds the expected risk of the hypothesis.

In this paper, we focus on Wasserstein distributional robustness \citep{abadeh2015distributionally,chen2018robust}, in which the adversary's decision set $\Pset$ is a norm ball with respect to the Wasserstein distance:
\begin{definition}[Wasserstein distance]
\label{def:wasserstein-distance}
Let $c : \Zspace \times \Zspace \rightarrow \reals_+$ be a lower-semicontinuous cost function. For any $\mu, \nu \in \Qset(\Zspace)$, the Wasserstein distance between $\mu$ and $\nu$ is
\begin{equation}
\Wass_c(\mu, \nu) = \inf_{\pi \in \Pi(\mu, \nu)} \int_{\Zspace \times \Zspace} c(\zvec, \zvec^{\prime})\,\dx\pi(\zvec, \zvec^{\prime}),
\end{equation}
with $\Pi(\mu, \nu) = \{ \pi \in \mathcal{M}_+(\Zspace \times \Zspace) : \pi(A \times \Zspace) = \mu(A), \pi(\Zspace \times B) = \nu(B), \forall A, B \in \Bset(\Zspace)\}$, i.e. the set of all joint distributions on $\Zspace \times \Zspace$ having marginals $\mu$ and $\nu$. $\pi$ is sometimes also called a ``transportation plan'' for moving the mass in $\mu$ to match $\nu$.
\end{definition}
The Wasserstein distance differs from other common divergences on probability measures, such as the KL divergence, in that it takes into account the geometry of the domain $\Zspace$, via the transport cost $c$. For this reason, it can compare measures with disjoint support, for example.

\subsection{Related work}

Distributionally robust optimization \citep{calafiore2006distributionally} has been explored extensively beyond the learning setting, for a broad variety of objective functions and decision sets. Often decision sets are defined by moment or support conditions \citep{delage2010distributionally,goh2010distributionally,wiesemann2014distributionally} or divergences on probability measures such as the Prokhorov metric \citep{erdougan2006ambiguous} or $f$-divergences \citep{ben2013robust,duchi2016statistics,namkoong2016stochastic,bertsimas2018data,miyato2015distributional}. Directional deviation conditions have also been explored \citep{chen2007robust}. \cite{kuhn2019wasserstein} gives a recent review of applications in machine learning.

Distributionally robust learning over a Wasserstein ball was proposed for logistic regression \citep{abadeh2015distributionally}, regularized linear regression \citep{chen2018robust}, and more general losses \citep{blanchet2016robust,gao2016distributionally,sinha2017certifying,dziugaite2017computing,esfahani2018data}. An equivalence to regularization, under various assumptions on the loss, was shown by \citet{gao2017wasserstein}.

In Section \ref{sec:active}, we discuss an application of the proposed method to active learning, which is a well-studied topic that has inspired a wide variety of algorithms \citep{yang2018benchmark}. We focus on a class of heuristics that seek to maximize the change in the learned model resulting from obtaining a labeled example \citep{settles2008multiple,freytag2014selecting,cai2017active}.

\section{Distributionally-robust learning with unlabeled data}
\label{sec:drl-with-unlabeled-data}

\subsection{A problem with the existing approach}

\begin{wrapfigure}[16]{r}{0.5\textwidth}
  \centering
    \vspace{-.35in}
    \includegraphics[width=\linewidth]{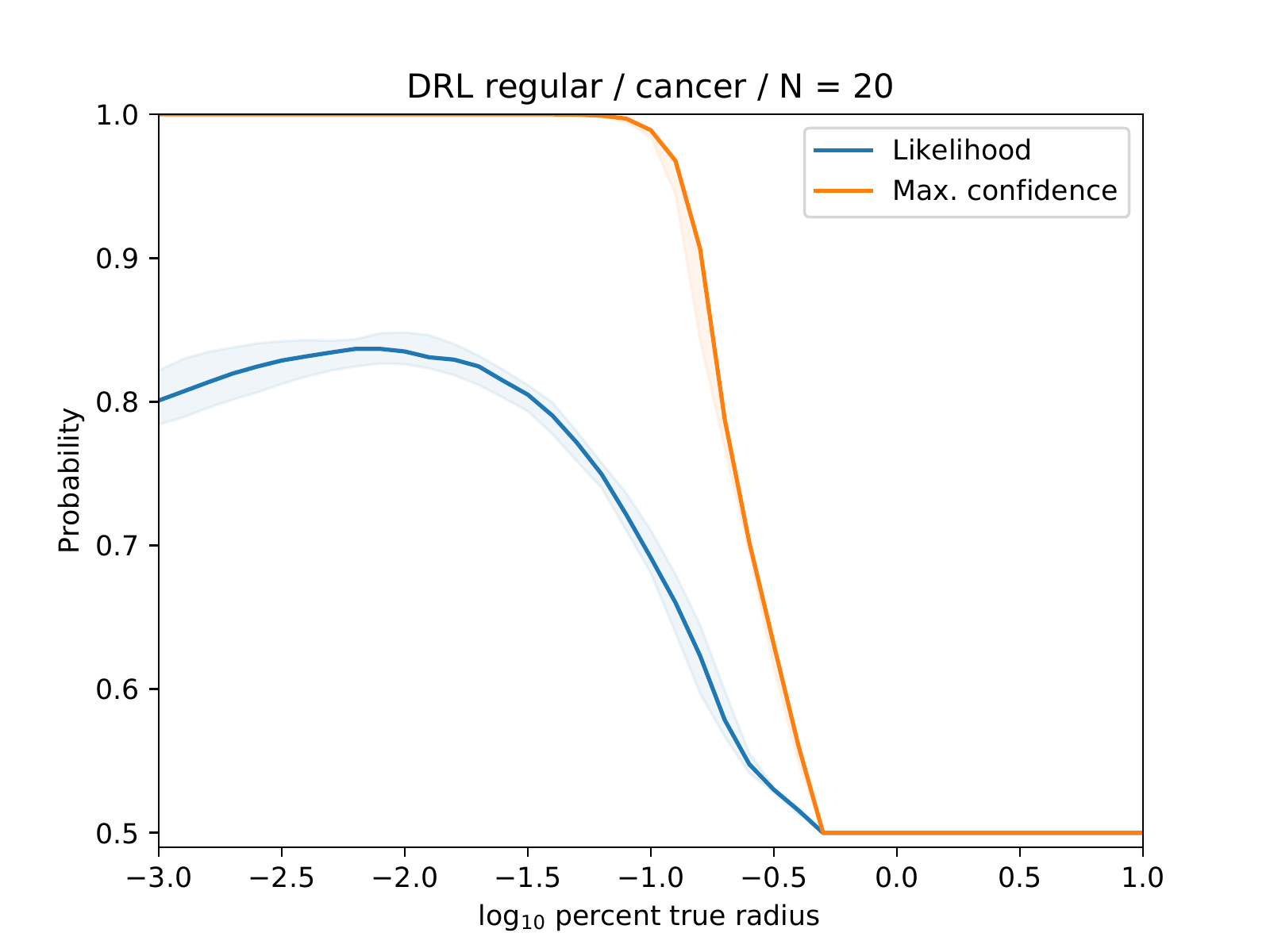}
    \vspace{-.3in}
  \caption{Wasserstein distributionally robust learning yields a no-confidence predictor at radius $\varepsilon$ much smaller than the distance to the true data distribution.}
  \label{fig:likelihood-and-conf-vs-pct-true-radius}
\end{wrapfigure}

In the ``medium-data'' regime, where the labeled sample may be far from the true data distribution $\Ppr$ with respect to Wasserstein distance, Wasserstein distributionally-robust learning suffers from imprecision of the decision set $\Pset$, which is a Wasserstein ball centered at the empirical distribution of the labeled sample. The volume of this ball grows rapidly in its radius, requiring the learner to be robust to an enormous variety of data distributions. This problem manifests as low confidence of the distributionally robust learner, even when the radius $\varepsilon$ is chosen to be much smaller than the true distance to the data distribution---thereby foregoing the performance guarantee implied by \eqref{eq:distributionally-robust-learning}. 

Figure \ref{fig:likelihood-and-conf-vs-pct-true-radius} shows an example; additional illustrations are in Section \ref{sec:empirical-performance-of-drl-with-unlabeled}. We train a Wasserstein distributionally robust logistic regression model using $20$ labeled samples from the Wisconsin breast cancer dataset \citep{dua2019:uci}.
We plot both the test set likelihood and the maximum confidence\footnote{The confidence of a hypothesis $h_{\theta}$ at a point $\xvec \in \Xspace$ we define by $\max\{h_{\theta}(\xvec), 1 - h_{\theta}(\xvec)\}$.} of the learner over input samples as the radius $\varepsilon$ of the decision set is varied. We see that the confidence goes to $0.5$---i.e., the classes are assigned equal probability---at a radius much smaller than the distance to the empirical distribution of the test set. Notably, the radius that maximizes the likelihood is approximately $1\%$ of the distance to the test distribution. This maximum is often suggested as an appropriate target when choosing the radius $\varepsilon$ in practice, as we will discuss in Section \ref{sec:how-important-is-radius}.

\subsection{Constraining the adversary using unlabeled data}
\label{sec:constraining-the-adversary}

We propose to deal with the overwhelming size of the decision set by constraining it further, pruning unrealistic potential data distributions while still allowing the set to contain the true data distribution. Specifically, we intersect two additional constraints with the Wasserstein ball.

The first constraint uses unlabeled data to constrain the marginal in $\Xspace$ of the data distribution. As is common in many learning settings, we assume that unlabeled data are acquired much more readily than labeled data, giving the learner access to large set of unlabeled examples. Let $\Ppr_{\Xspace}$ be the {\bf $\Xspace$-marginal}, defined by $\Ppr_{\Xspace}(A) = \Ppr(A \times \Yspace)$ for all Borel subsets $A \in \Bset(\Xspace)$. Then our {\bf unlabeled data} is a set $\hat{\Xspace}_u \subset \Xspace$ drawn i.i.d.\ from $\Ppr_{\Xspace}$.

The second constraint restricts the {\bf $\Yspace$-marginal} of the data distribution, by defining intervals on the individual label probabilities. Let $\Ppr_{\Yspace}$ be the $\Yspace$-marginal, which in a classification setting is discrete $\Ppr_{\Yspace} = \sum_{k=1}^{N_{\Yspace}} \pvec_{\Yspace}^k \delta_{\yvec^k}$ for $\Yspace = \{\yvec^k\}_{k=1}^{N_{\Yspace}}$ the set of labels and $\pvec_{\Yspace}^k$ the corresponding label probabilities. The {\bf interval} for each label is $[\underline{\pvec}_{\Yspace}^k, \overline{\pvec}_{\Yspace}^k]$. These interval constraints might come from prior knowledge, another dataset as in the ecological inference setting \citep{king2013solution,frogner2019fast}, or directly from the training data.

\subsection{Problem formulation and duality}
\label{sec:drl-problem-formulation}

\begin{table}[t]
\begin{center}
\begin{tabular}{r|l}
$\Xspace$ & Feature space \\
$\Yspace = \{\yvec^k\}_{k=1}^{N_{\Yspace}}$ & Label space \\
$\Zspace$ & $\Xspace \times \Yspace$ \\
$h_{\theta}$ & Hypothesis function, parameterized by $\theta$ \\
$\ell$ & Loss function \\
$\Ppr$ & True data distribution (over $\Zspace$) \\
$\Ppr_{\Xspace}$ & $\Xspace$-marginal of $\Ppr$ \\
$\Ppr_{\Yspace}$ & $\Yspace$-marginal of $\Ppr$ \\
$\hat{\Ppr}_l$ & Labeled data distribution (over $\Zspace$) \\
$\ball_{\varepsilon}(\hat{\Ppr}_l)$ & Wasserstein ball of radius $\varepsilon$ about $\hat{\Ppr}_l$ \\
$\overline{\pvec}_{\Yspace}^k$ ($\underline{\pvec}_{\Yspace}^k)$ & Upper (lower) bound on marginal probability of $\yvec^k$ \\
$\Uset(\Ppr_{\Xspace}, \underline{\pvec}_{\Yspace}, \overline{\pvec}_{\Yspace})$ & Set of probability measures $\pi \in \Qset(\Zspace \times \Zspace)$ whose first \\
& $\Xspace$-marginal is $\Ppr_{\Xspace}$ and first $\Yspace$-marginal satisfies \\
& $\pi((\Xspace \times \{\yvec^k\}) \times \Zspace) \in [\underline{\pvec}_{\Yspace}^k, \overline{\pvec}_{\Yspace}^k]$, $\forall k$.
\end{tabular}
\end{center}
\caption{Notation.}
\end{table}

If we restrict the decision set as described in Section \ref{sec:constraining-the-adversary}, we need to establish that the distributionally robust learning problem is still tractable, particularly since one of the constraints we have added is infinite-dimensional. Recall that $\Ppr_\Xspace$ is the marginal of the unlabeled data on the feature space, and that $\underline{p}_\Yspace$ and $\overline{p}_\Yspace$ are the lower and upper bounds on the marginal on the label. 
We can define a set of possible joint distributions on $\Xspace \times \Yspace$ that are consistent with this data,
\begin{equation}
\begin{aligned}
\Uset(\Ppr_{\Xspace}, \underline{\pvec}_{\Yspace}, \overline{\pvec}_{\Yspace}) = \Bigl\{\Ppr \in \Mset_+(\Xspace \times \Yspace) &: \Ppr(A \times \Yspace) = \Ppr_{\Xspace}(A), \\
& \quad \Ppr(\Xspace \times B) \in [\underline{\Ppr}_{\Yspace}(B), \overline{\Ppr}_{\Yspace}(B)], \\
& \quad \forall A \in \Bset(\Xspace), B \subseteq \Yspace \Bigr\},
\end{aligned}
\end{equation}
with $\underline{\Ppr}_{\Yspace} = \sum_{k=1}^{N_{\Yspace}} \underline{\pvec}_{\Yspace}^k \delta_{\yvec^k}$ and $\overline{\Ppr}_{\Yspace} = \sum_{k=1}^{N_{\Yspace}} \overline{\pvec}_{\Yspace}^k \delta_{\yvec^k}$.

Suppose in addition we observe labeled data $\hat{\Zspace}_l = \{\zvec_{\ell}^i\}_{i=1}^{N_l} \subset \Xspace \times \Yspace$, with $\zvec_{\ell}^i = (\xvec_l^i, \yvec_{\ell}^i)$, that define the empirical distribution $\hat{\Ppr}_l = \frac{1}{N_l} \sum_{i=1}^{N_l} \delta_{\zvec_{\ell}^i}$. We define the adversary's decision set to be the intersection of the set of distributions $\Uset(\Ppr_{\Xspace}, \underline{\pvec}_{\Yspace}, \overline{\pvec}_{\Yspace})$ with a Wasserstein ball of radius $\varepsilon$ around the empirical distribution $\hat{\Ppr}_l$:
\begin{equation}
\Pset = \Uset(\Ppr_{\Xspace}, \underline{\pvec}_{\Yspace}, \overline{\pvec}_{\Yspace}) \cap \ball_{\varepsilon}(\hat{\Ppr}_l),
\end{equation}
where $\ball_{\varepsilon}(\hat{\Ppr}_l) = \{\mu : \Wass_c(\mu, \hat{\Ppr}_l) \leq \varepsilon\}$. Thus, our feasible set contains all distributions that have the correct data and label marginals (and are thus in $\Uset(\Ppr_{\Xspace}, \underline{\pvec}_{\Yspace}, \overline{\pvec}_{\Yspace})$), but which are also close to the known labeled distribution (and thus contained in $\ball_{\varepsilon}(\hat{\Ppr}_l)$).

The resulting distributionally-robust problem is defined identically to \eqref{eq:distributionally-robust-learning}, using this decision set $\Pset$. The inner problem with fixed $\theta$ is that of evaluating a {\bf worst-case expected loss}, 
\begin{equation}
\label{eq:worst-case-primal}
f(\theta) = \sup_{\mu \in \Pset} \expect^{\mu} \ell(h_{\theta}(\Xrv), \Yrv).
\end{equation}
For marginals $\Ppr_{\Xspace}$ with infinite support this an infinite-dimensional linear program as the marginal on $\Xspace$ of the solution $\mu$ must contain the support of $\Ppr_{\Xspace}$.

We can rewrite \eqref{eq:worst-case-primal} by casting it as an optimal transportation problem over the space $\Zspace=\Xspace\times \Yspace$ between our unknown distribution $\mu$ and the given data distribution such that the transport plan $\pi$ satisfies the marginal constraints on $\mu$:
\begin{equation}
f(\theta)=
\left\{
\begin{array}{rll}
\sup_{\pi \in \Mset(\Zspace \times \Zspace)} & \int_{(\Xspace \times \Yspace) \times \Zspace} \ell(h_{\theta}(\xvec), \yvec)\,\dx\pi((\xvec, \yvec), \zvec^{\prime}) \\
\text{s.t.}
 & \int_{\Zspace \times \Zspace} c(\zvec, \zvec^{\prime})\,\dx\pi(\zvec, \zvec^{\prime}) \leq \varepsilon \\
& \int_{\Zspace \times \Zspace} \delta_{\zvec_{\ell}^i}(\zvec^{\prime})\, \dx\pi(\zvec, \zvec^{\prime}) = \frac{1}{N_l}
&\forall i \in \{1, \dots, N_l\}, \\
& \pi((A \times \Yspace) \times \Zspace) = \Ppr_{\Xspace}(A) & \forall A \in \Bset(\Xspace), \\
& \int_{(\Xspace \times \Yspace) \times \Zspace} \delta_{\yvec^k}(\yvec)\, \dx\pi((\xvec, \yvec), \zvec^{\prime}) \leq \overline{\pvec}_{\Yspace}^k
& \forall k \in \{1, \dots, N_{\Yspace}\}, \\
& \int_{(\Xspace \times \Yspace) \times \Zspace} \delta_{\yvec^k}(\yvec) \,\dx\pi((\xvec, \yvec), \zvec^{\prime}) \geq \underline{\pvec}_{\Yspace}^k 
& \forall k \in \{1, \dots, N_{\Yspace}\}, \\
& \pi(A) \geq 0
& \forall A \in \Bset(\Zspace \times \Zspace). 
\end{array}
\right.
\label{eq:worst-case-primal-transport}
\end{equation}
Here the variable $\zvec = (\xvec, \yvec)$ indexes the support of the worst-case measure while $\zvec^{\prime}$ indexes the support of $\hat{\Ppr}_l$. $\pi$ is a transport plan that joins these two measures. 
Observe that only the constraint on the $\Xspace$ marginal is infinite dimensional. We will show that this constraint corresponds in the dual problem to an expectation under $\Ppr_{\Xspace}$ of a finite dimensional cost.

While the program \eqref{eq:worst-case-primal} is infinite dimensional, its dual is a problem in finite dimensions:
\iffalse
\begin{equation}
\label{eq:worst-case-dual}
\begin{aligned}
& g(\theta)
& = & \inf_{\alpha, \beta, \underline{\lambda}, \overline{\lambda}} \alpha \varepsilon + \frac{1}{N_l} \sum_{i=1}^{N_l} \beta^i + \sum_{k=1}^{N_{\Yspace}} \left(\overline{\lambda}^k \overline{\pvec}_{\Yspace}^k - \underline{\lambda}^k \underline{\pvec}_{\Yspace}^k \right) \\
&&& \quad\quad + \expect^{\Ppr_{\Xspace}} \left[ \max_{\substack{k \in \{1, \dots, N_{\Yspace}\}\\i \in \{1, \dots, N_l\}}} \ell(h_{\theta}, (\xvec, \yvec^k)) - \left(\alpha c\left((\xvec, \yvec^k), \zvec_{\ell}^i\right) + \beta^i)\right) - (\overline{\lambda}^k - \underline{\lambda}^k)\right] \\
& \text{s.t.}
& & \alpha, \underline{\lambda}^k, \overline{\lambda}^k \geq 0, \quad \forall k \in \{1, \dots, N_{\Yspace}\}.
\end{aligned}
\end{equation}
\fi
%
\begin{equation}
\label{eq:worst-case-dual}
g(\theta)
\!=\! 
\left\{\hspace{-.1in}
\begin{array}{r@{\ }l}
{\displaystyle\inf_{\alpha, \beta, \underline{\lambda}, \overline{\lambda}}}
&
\alpha \varepsilon + \frac{1}{N_l} \sum_{i=1}^{N_l} \beta^i + \sum_{k=1}^{N_{\Yspace}} \left(\overline{\lambda}^k \overline{\pvec}_{\Yspace}^k - \underline{\lambda}^k \underline{\pvec}_{\Yspace}^k \right) \\
&
+ \expect^{\Ppr_{\Xspace}}\!\left[\!
\vcenter{\hbox{$
{\displaystyle \max_{\substack{k \in \{1, \dots, N_{\Yspace}\}\\i \in \{1, \dots, N_l\}}}}
\hspace{-.1in}
\ell(h_{\theta}, (\Xrv, \yvec^k)) - 
\left(\alpha c\left((\Xrv, \yvec^k), \zvec_{\ell}^i\right) + \beta^i)\right) - (\overline{\lambda}^k - \underline{\lambda}^k)
$}}
\!\right] 
\vspace{.1in}
\\
\text{s.t.}
 & \alpha, \underline{\lambda}^k, \overline{\lambda}^k \geq 0 \quad \forall k \in \{1, \dots, N_{\Yspace}\}.
\end{array}
\right.
\end{equation}
Here $\alpha \in \reals$, $\beta \in \reals^{N_{\ell}}$, and $\underline{\lambda}, \overline{\lambda} \in \reals^{N_{\Yspace}}$. Each of these dual variables corresponds to a primal constraint: $\alpha$ corresponds to the constraint on the transport cost, $\beta$ the constraint that the second marginal be $\hat{\Ppr}_l$, $\underline{\lambda}$ the lower bound on the worst-case label probabilities, and $\overline{\lambda}$ the upper bound. The infinite-dimensional constraint that the first marginal of the primal transport plan have $\Xspace$-marginal $\Ppr_{\Xspace}$ corresponds here to the expectation in the objective. This correspondence is established in much more detail in the proof of Theorem \ref{thm:strong-duality} (Appendix \S\ref{sec:appendix-strong-duality}), which shows that the two problems $g(\theta)$ and $f(\theta)$ are in fact equivalent. 

We state our main theoretical result, whose proof is deferred to the appendix (\S\ref{sec:appendix-strong-duality}):
\begin{theorem}[Strong duality]
\label{thm:strong-duality}
Let $\Xspace$ be a compact Polish space and $\Yspace = \{\yvec^k\}_{k=1}^{N_{\Yspace}}$ any finite set. Let $\Ppr_{\Xspace}$ be a probability measure over $\Xspace$ and $\hat{\Ppr}_l = \frac{1}{N_l} \sum_{i=1}^{N_l} \delta_{\zvec_{\ell}^i}$ an empirical probability measure over $\Zspace = \Xspace \times \Yspace$, and define intervals $[\underline{\pvec}_{\Yspace}^k, \overline{\pvec}_{\Yspace}^k] \subseteq [0, 1]$, $k \in \{1, \dots, N_{\Yspace}\}$. Let the transportation cost $c : \Zspace \times \Zspace \rightarrow [0, +\infty)$ be nonnegative and upper semicontinuous with $c(\zvec, \zvec^{\prime}) = 0  \Leftrightarrow \zvec = \zvec^{\prime}$. Assume $\ell(h_{\theta}(\cdot), \cdot) : \Zspace \rightarrow \reals$ is upper semicontinuous. Define $f$ as in \eqref{eq:worst-case-primal} and $g$ as in \eqref{eq:worst-case-dual}. If $\Uset(\Ppr_{\Xspace}, \underline{\pvec}_{\Yspace}, \overline{\pvec}_{\Yspace}) \cap \ball_{\varepsilon}(\hat{\Ppr}_l) \neq \emptyset$, then
\begin{equation}
f(\theta) = g(\theta),\quad \forall \theta \in \Theta.
\end{equation}
Furthermore, if $\operatorname{relint} (\Uset(\Ppr_{\Xspace}, \underline{\pvec}_{\Yspace}, \overline{\pvec}_{\Yspace}) \cap \ball_{\varepsilon}(\hat{\Ppr}_l)) \neq \emptyset$, then there exists a minimizer $(\alpha_{\ast}, \beta_{\ast}, \overline{\lambda}_{\ast}, \underline{\lambda}_{\ast}) \in \reals_+ \times \reals^{N_l} \times \reals_+^{N_{\Yspace}} \times \reals_+^{N_{\Yspace}}$ attaining the infimum in \eqref{eq:worst-case-dual}.
\end{theorem}
The take-away message of Theorem \ref{thm:strong-duality} is that distributionally-robust learning under the model proposed here amounts to minimizing $g$ with respect to $\theta$, a finite dimensional problem that can be tackled with stochastic gradient approaches.

\section{Algorithm and analysis}

\subsection{Optimization by SGD}
\label{sec:optimization-by-sgd}

Problem \eqref{eq:worst-case-dual} is a convex, finite dimensional optimization problem in $\alpha, \beta, \underline{\lambda}, \overline{\lambda}$ that is the sum of a linear term and an expectation under $\Ppr_\Xspace$. To apply stochastic gradient descent, we first need to compute derivatives under the variables $\alpha, \beta, \underline{\lambda}, \overline{\lambda}$.

We first compute derivatives of the term under the expectation. Define $\Phi^{i,k}(\cdot; \theta, \alpha, \beta, \underline{\lambda}, \overline{\lambda})$ as the function 
\begin{align*}
    \Phi^{i,k}(\xvec; \theta, \alpha, \beta, \underline{\lambda}, \overline{\lambda}) = \ell(h_{\theta}, (\Xrv, \yvec^k)) -  \left(\alpha c\left((\Xrv, \yvec^k), \zvec_{\ell}^i\right) + \beta^i\right) - (\overline{\lambda}^k - \underline{\lambda}^k).
\end{align*}
The dual objective can be expressed as a function of $\theta, \alpha, \beta, \underline{\lambda}, \overline{\lambda}$ as 
\begin{align*}
    \alpha \varepsilon + \frac{1}{N_l} \sum_{i=1}^{N_l} \beta^i + \sum_{k=1}^{N_{\Yspace}} \left(\overline{\lambda}^k \overline{\pvec}_{\Yspace}^k - \underline{\lambda}^k \underline{\pvec}_{\Yspace}^k \right) + \expect^{\Ppr_{\Xspace}}\left[
    \vcenter{\hbox{$
    \max_{\substack{k \in \{1, \dots, N_{\Yspace}\}\\i \in \{1, \dots, N_l\}}} \Phi^{i,k}(\xvec; \theta, \alpha, \beta, \underline{\lambda}, \overline{\lambda})$}}\right].
\end{align*}
For a given choice of $i$ and $k$, there is a set of points $\xvec$ where the maximum is achieved at $(i, k)$. We can define a collection of subsets of $\Xspace$ 
\begin{equation}
\label{eq:voronoi-sets}
V^{ik} = \left\{\xvec \in \Xspace : \Phi^{i,k}(\xvec;\theta, \alpha, \beta, \underline{\lambda}, \overline{\lambda}) \geq \Phi^{i^\prime,k^\prime}(\xvec;\theta, \alpha, \beta, \underline{\lambda}, \overline{\lambda})\ \forall\ i^{\prime},
k^{\prime}
\right\},
\end{equation}
The sets $V^{ik}$ partition $\Xspace$, up to boundary points where the sets meet one another.
We can decompose the expectation as a finite sum of integrals over domains $V^{ik}$, i.e.
\begin{equation}
\expect^{\Ppr_{\Xspace}}\left[
\vcenter{\hbox{$
\max_{\substack{k \in \{1, \dots, N_{\Yspace}\}\\i \in \{1, \dots, N_l\}}} \Phi^{i,k}(\xvec; \theta, \alpha, \beta, \underline{\lambda}, \overline{\lambda})
$}}
\right] = \sum_{i=1}^{N_l} \sum_{k=1}^{N_{\Yspace}} \int_{V^{ik}} \Phi^{i,k}(\xvec; \theta, \alpha, \beta, \underline{\lambda}, \overline{\lambda}) \,\dx\Ppr_{\Xspace}(\xvec).
\end{equation}
Note that $V^{ik}$ changes depending on the parameters $(\theta, \alpha, \beta, \overline{\lambda}, \underline{\lambda})$. To evaluate a derivative with respect to one of these parameters, then, we need to differentiate under the integral sign. Applying Reynolds' Transport Theorem, we obtain that 
\begin{equation}
\frac{\partial}{\partial \alpha} \sum_{i=1}^{N_l} \sum_{k=1}^{N_{\Yspace}} \int_{V^{ik}} \Phi^{i,k}(\xvec;\theta, \alpha, \beta, \underline{\lambda}, \overline{\lambda}) \,\dx\Ppr_{\Xspace}(\xvec) = \sum_{i=1}^{N_l} \sum_{k=1}^{N_{\Yspace}} \int_{V^{ik}} \frac{\partial}{\partial \alpha} \Phi^{i,k}(\xvec;\theta, \alpha, \beta, \underline{\lambda}, \overline{\lambda}) \,\dx\Ppr_{\Xspace}(\xvec),
\end{equation}
and the same holds for the other parameters $\theta, \beta, \overline{\lambda}, \underline{\lambda}$.\footnote{Reynold's Theorem specifies terms that are boundary integrals for the boundaries of the sets $V^{ik}$. In our case, these terms sum to zero, as almost every boundary point $\xvec \in \operatorname{int} \Xspace$ is shared between exactly two sets $V^{ik}$, $ V^{i^{\prime} k^{\prime}}$ and the integrands at $\xvec$ for the corresponding boundary integrals exactly cancel.}

The exact forms for the derivatives of the dual objective
are given in the appendix (\S\ref{sec:appendix-dual-gradients}). To simplify notation further, we define 
\begin{equation}
\label{eq:big-phi}
\Phi(\xvec;\theta, \alpha, \beta, \underline{\lambda}, \overline{\lambda}) = \max_{\substack{k \in \{1, \dots, N_{\Yspace}\}\\i \in \{1, \dots, N_l\}}} \Phi^{i,k}(\xvec; \theta, \alpha, \beta, \underline{\lambda}, \overline{\lambda}).
\end{equation}
We can optimize for the optimal dual parameters $\theta, \alpha, \beta, \underline{\lambda}, \overline{\lambda}$ by sampling $\xvec^1, \ldots, \xvec^{N_b}$ from $\Ppr_\Xspace$ and computing gradients of $\Phi(\xvec^j; \theta, \alpha, \beta, \underline{\lambda}, \overline{\lambda})$ with respect to the dual variables while maintaining constraints. This approach is summarized in Algorithm \ref{alg:sgd}. 

\begin{algorithm}[h]
\caption{SGD for distributionally robust learning with unlabeled data}
\label{alg:sgd}
\begin{algorithmic}
\STATE {\bf Given}: $\varepsilon \geq 0$, $\overline{\pvec}_{\Yspace}, \underline{\pvec}_{\Yspace} \in [0, 1]^{N_{\Yspace}}$, $\theta_0 \in \Theta$, step size $\gamma > 0$, batch size $N_b$.
\STATE $\theta \gets \theta_0$, $\alpha, \beta, \overline{\lambda}, \underline{\lambda} \gets \mathbf{0}$.
\WHILE{not converged}
\STATE \# \COMMENT{Computation of subgradients of $\Phi$ is described in Appendix \ref{sec:appendix-dual-gradients}.}
\STATE Sample $\xvec^1, \dots, \xvec^{N_b} \sim \Ppr_{\Xspace}$.
\STATE $\theta \gets \Proj_{\Theta}\left[\theta - \frac{\gamma}{N_b} \sum_{j=1}^{N_b} \nabla_\theta \Phi(\xvec^j; \theta, \alpha, \beta, \underline{\lambda}, \overline{\lambda}) \right]$.
\STATE $\alpha \gets \max\left(0, \alpha - \gamma \left[\varepsilon + \frac{1}{N_b} \sum_{j=1}^{N_b} \nabla_\alpha \Phi(\xvec^j; \theta, \alpha, \beta, \underline{\lambda}, \overline{\lambda})\right]\right)$.
\STATE $\beta \gets \beta - \gamma \left[\frac{1}{N_l} + \frac{1}{N_b} \sum_{j=1}^{N_b} \nabla_\beta \Phi(\xvec^j; \theta, \alpha, \beta, \underline{\lambda}, \overline{\lambda})\right]$.
\STATE $\overline{\lambda} \gets \max\left(\mathbf{0}, \overline{\lambda} - \gamma \left[\overline{\pvec}_{\Yspace} + \frac{1}{N_b} \sum_{j=1}^{N_b} \nabla_{\overline{\lambda}} \Phi(\xvec^j; \theta, \alpha, \beta, \underline{\lambda}, \overline{\lambda})\right]\right)$.
\STATE $\underline{\lambda} \gets \max\left(\mathbf{0}, \underline{\lambda} - \gamma \left[-\underline{\pvec}_{\Yspace} + \frac{1}{N_b} \sum_{j=1}^{N_b} \nabla_{\underline{\lambda}} \Phi(\xvec^j; \theta, \alpha, \beta, \underline{\lambda}, \overline{\lambda})\right]\right)$.
\ENDWHILE 
\end{algorithmic}
\end{algorithm}

\subsection{A computable performance guarantee}
\label{sec:performance-guarantee}

An attractive feature of traditional Wasserstein DRL is that the optimal value of the objective upper-bounds the true expected risk $\expect^{\Ppr} \ell(h_{\theta}(\Xrv), \Yrv)$, provided that the adversary's decision set contains the true data distribution $\Ppr$. 

The proposed formulation using unlabeled data provides a similar guarantee. Specifically, for all $\theta \in \Theta$, $(\alpha, \beta, \phi, \overline{\lambda}, \underline{\lambda}) \in \Lambda_{\theta}$, we have
\begin{equation}
\label{eq:dual-problem-empirical}
\expect^{\Ppr} \ell(h_{\theta}(\Xrv), \Yrv) \leq \alpha \varepsilon + \frac{1}{N_l} \sum_{i=1}^{N_l} \beta^i + \sum_{k=1}^{N_{\Yspace}} \left(\overline{\lambda}^k \overline{\pvec}_{\Yspace}^k - \underline{\lambda}^k \underline{\pvec}_{\Yspace}^k \right) + \expect^{\hat{\Ppr}_{\Xspace}}\left[\Phi(\xvec)\right] + \epsilon(N_u),
\end{equation}
with high probability, where $\hat{\Ppr}_{\Xspace}$ is an empirical sample on $N_u$ points from the distribution on unlabeled data $\Ppr_\Xspace$.
In other words, the objective value for the dual problem provides a computable guarantee on the generalization error of the learned predictor $h_\theta$.

This result follows from weak duality as the value of \eqref{eq:dual-problem-empirical} upper bounds $\expect^{\Ppr} \ell(h_{\theta}(\Xrv), \Yrv)$ if $\Ppr$ is feasible, and an application of the Berry--Esseen theorem  \citep{berry1941accuracy,esseen1942liapunoff} implies
\begin{equation*}
\expect^{\Ppr_{\Xspace}} \Phi(\Xrv) \leq \expect^{\hat{\Ppr}_{\Xspace}} \Phi(\Xrv) + \epsilon(N_u).
\end{equation*}
with $\epsilon(N_u) = \Ord(\frac{1}{\sqrt{N_u}})$. 
This result gives a guarantee on the generalization error of $h_\theta$, provided that $\Ppr \in \Pset$. In \S\ref{sec:empirical-performance-of-drl-with-unlabeled}, we compute this bound for a number of datasets and compare it to the equivalent bound from traditional DRL (Figure \ref{fig:true-radius-fval}). Since the bound relies on the true distribution $\Ppr$ being included in the adversary's decision set, the choice of the radius $\varepsilon$ of the Wasserstein ball around the labeled data distribution $\hat{\Ppr}_l$ is very important. We comment on the impact of $\varepsilon$ in \S\ref{sec:how-important-is-radius}. 

\section{Empirical results}

In this section, we investigate the empirical performance of our proposed formulation of distributionally robust learning in the particular case of logistic regression. First, we demonstrate an important limitation of the previously-proposed distributionally robust logistic regression \citep{abadeh2015distributionally}: There is often no choice of the radius $\varepsilon$ of the adversary's decision set that yields a classifier that is both robust and non-trivial. We then demonstrate that the formulation proposed here, which uses unlabeled data to restrict the adversary, can yield non-trivial classifiers with non-vacuous bounds on the generalization error.

\subsection{How important is the choice of $\varepsilon$?}
\label{sec:how-important-is-radius}

In practice, we do not know the radius necessary to include the true data distribution $\Ppr$ in the Wasserstein ball $\ball_{\varepsilon}(\hat{\Ppr}_l)$. Standard practice for DRL is to choose $\varepsilon$ by cross-validation, attempting to maximize a proxy for the out-of-sample performance. Implicitly, however, doing so relies on a \emph{regularization} effect of traditional DRL, documented by \citet{gao2017wasserstein}, which generates an inverted U-shaped
out-of-sample performance curve with respect to $\varepsilon$. Maximizing cross-validation performance does not necessarily yield a robust classifier in the DRL sense: As we demonstrate in Section \ref{sec:empirical-performance-of-drl-with-unlabeled}, for some datasets there is \emph{no choice of $\varepsilon$} that both includes $\Ppr$ in the adversary's decision set and yields a non-trivial classifier using traditional DRL. In other words, the $\varepsilon$ that maximizes generalization performance is much smaller than the distance between the labeled data $\hat{\Ppr}_l$ and the true data distribution $\Ppr$.

Here, we verify that the choice of $\varepsilon$ matters critically for robustness in the sense of traditional DRL, meaning that a learned classifier that is robust to distributions within an $\varepsilon$-ball $\ball_{\varepsilon}(\hat{\Ppr}_l)$ is not robust to distributions even slightly outside the $\varepsilon$-ball. In this sense, choosing $\varepsilon$ by cross-validation in traditional DRL can yield a classifier that is not robust to perturbations on the order of the distance between the labeled data $\hat{\Ppr}_l$ and the true data distribution $\Ppr$.

We make two empirical observations:
\begin{enumerate}
\item The generalization performance curve for traditional DRL has an inverted U shape, with a maximum at a much smaller radius $\varepsilon$ than is required to include the true data distribution $\Ppr$ in the adversary's decision set. This is shown for several datasets in Figure \ref{fig:likelihood-and-confidence-vs-radius} and further in Appendix \ref{sec:appendix-regular-drl-generalization-vs-radius}.
\item Wasserstein distributional robustness out to radius $\varepsilon$ does not confer robustness to distributions even slightly outside $\ball_{\varepsilon}(\hat{\Ppr}_l)$, at distance $\varepsilon + \Delta$, in the sense that there exists a data distribution in the ball $\ball_{\varepsilon + \Delta}(\hat{\Ppr}_l)$ that yields poor performance for the traditional Wasserstein DRL predictor trained with radius of robustness $\varepsilon$. This is shown for several datasets in Figure \ref{fig:worst-case-likelihood-vs-radius-plus-delta} and further in Appendix \ref{sec:appendix-beyond-decision-set}.
\end{enumerate}

\subsubsection{Choosing $\varepsilon$ in the proposed method}

\begin{figure}
\centering
\begin{subfigure}[b]{0.49\textwidth}
\includegraphics[width=\textwidth]{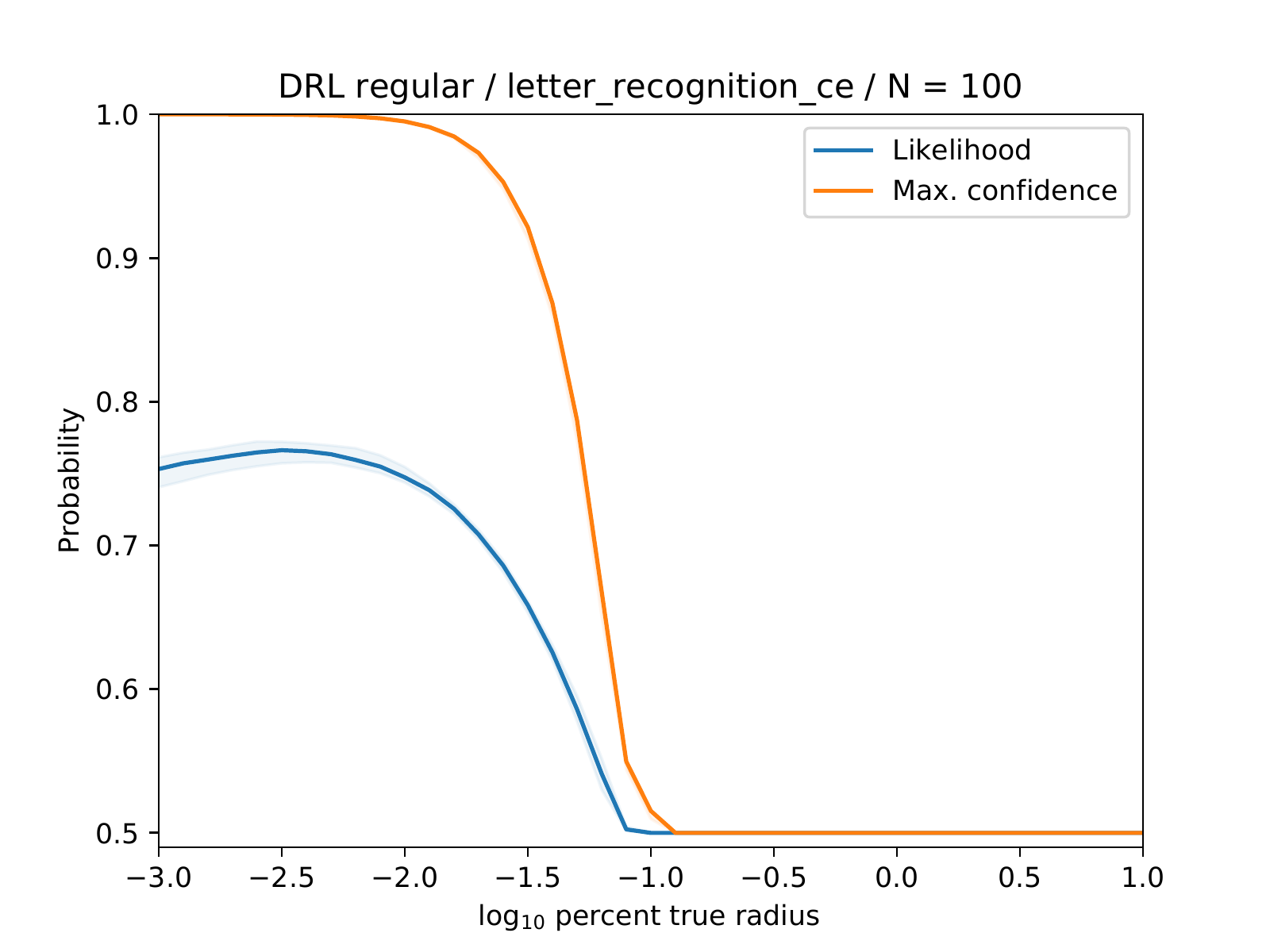}
\caption{Letter recognition (C-E)}
\end{subfigure}
\begin{subfigure}[b]{0.49\textwidth}
\includegraphics[width=\textwidth]{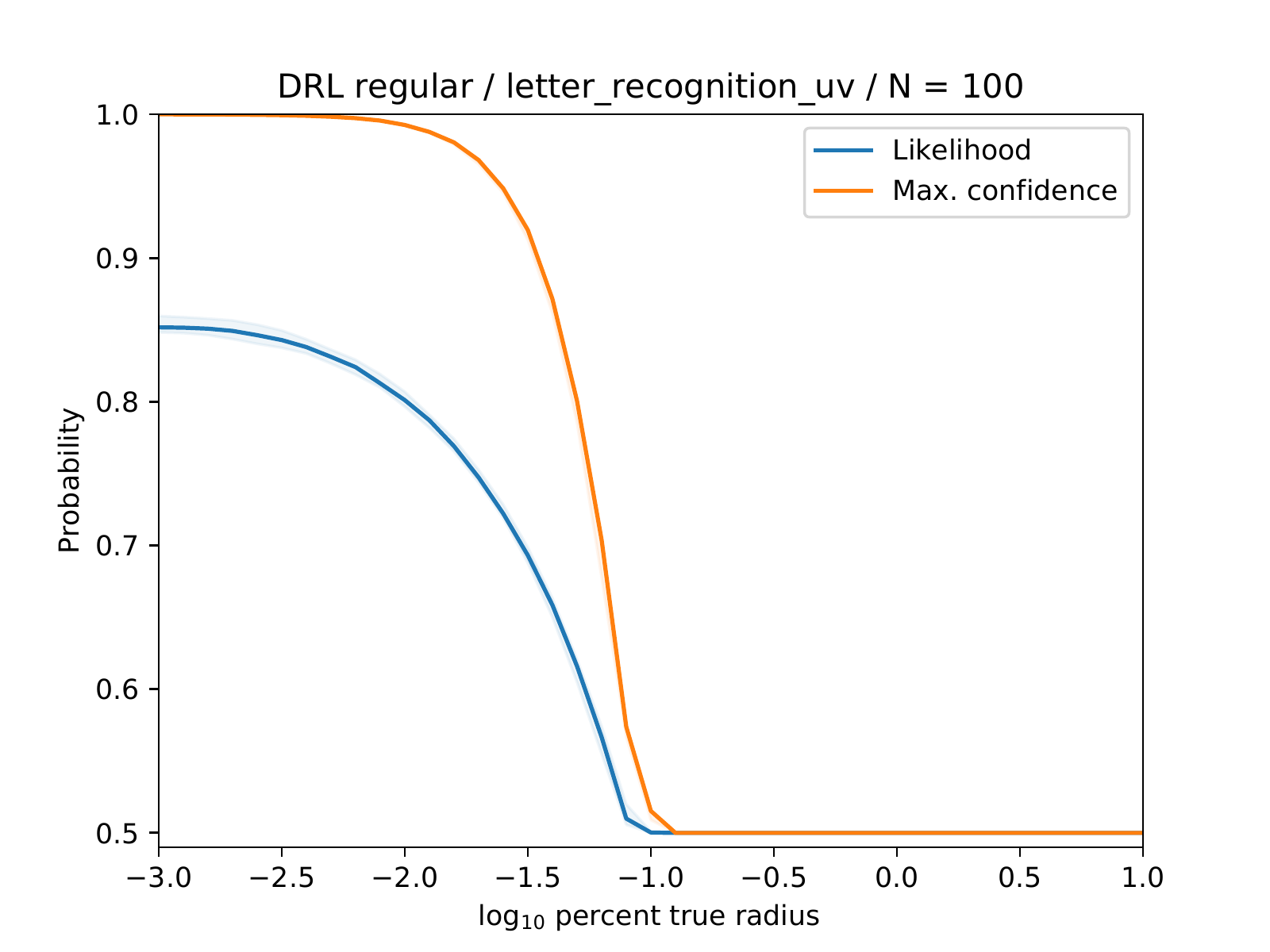}
\caption{Letter recognition (U-V)}
\end{subfigure}
\begin{subfigure}[b]{0.49\textwidth}
\includegraphics[width=\textwidth]{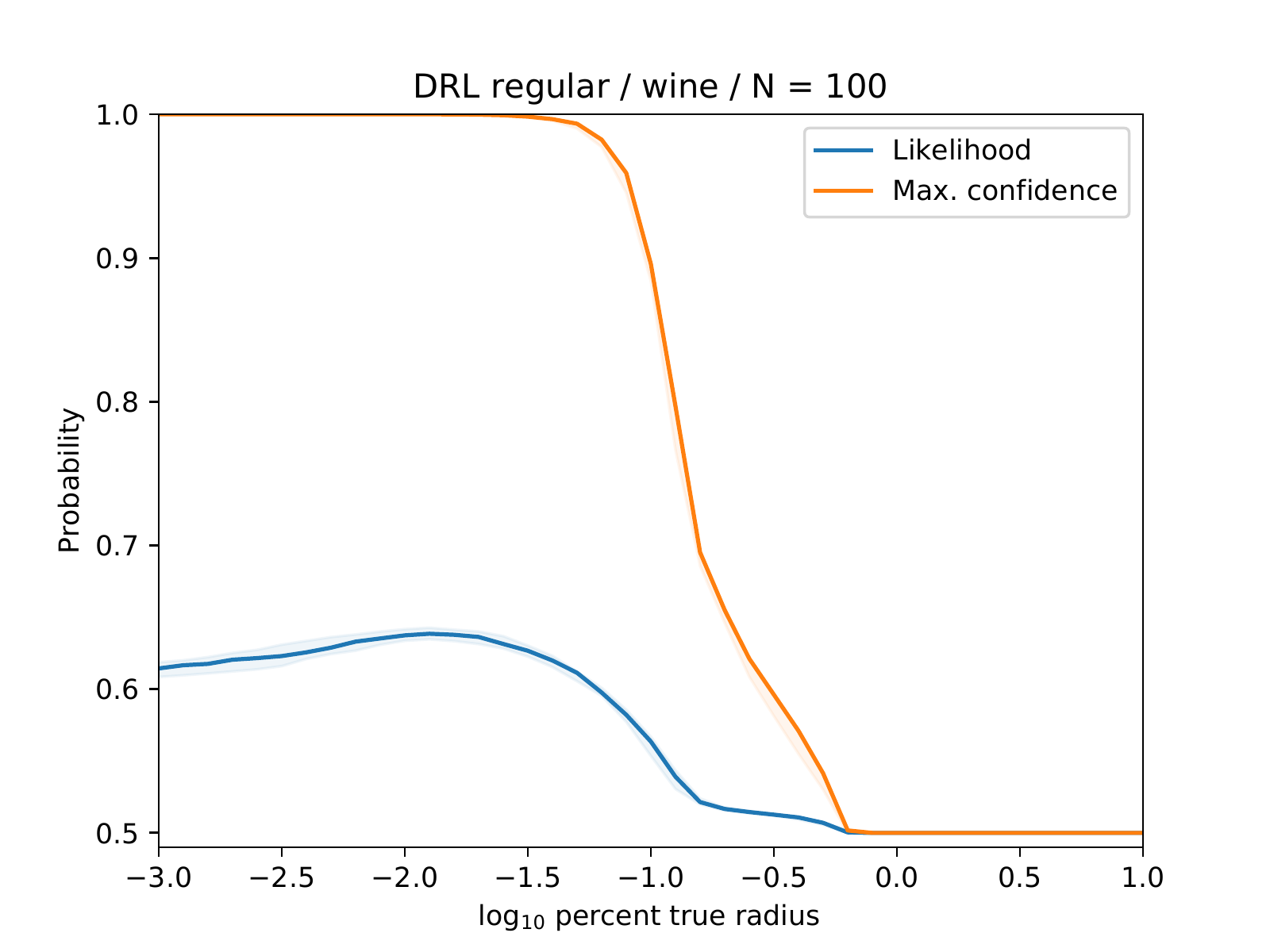}
\caption{Wine}
\end{subfigure}
\begin{subfigure}[b]{0.49\textwidth}
\includegraphics[width=\textwidth]{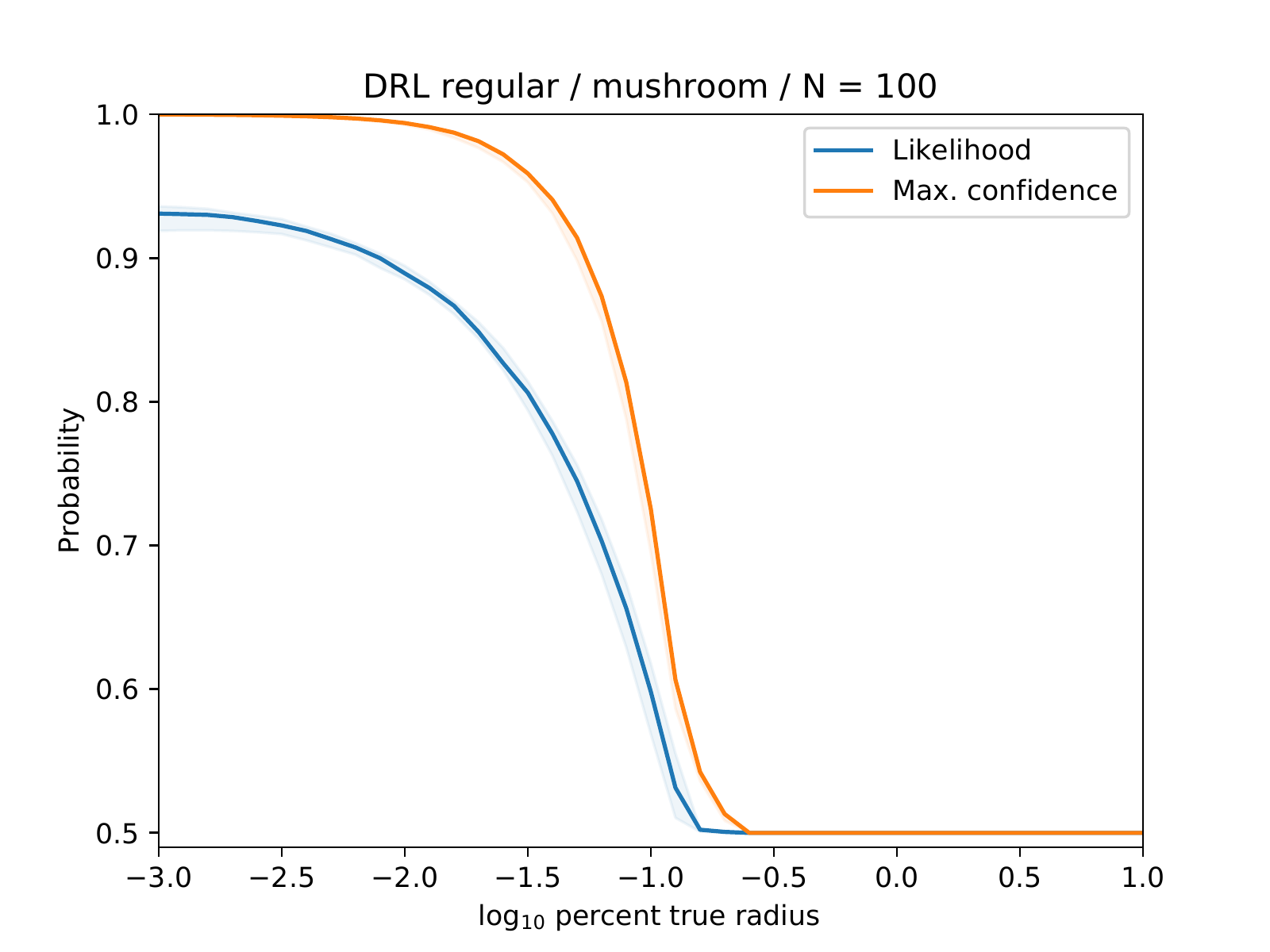}
\caption{Mushroom}
\end{subfigure}
\begin{subfigure}[b]{0.49\textwidth}
\includegraphics[width=\textwidth]{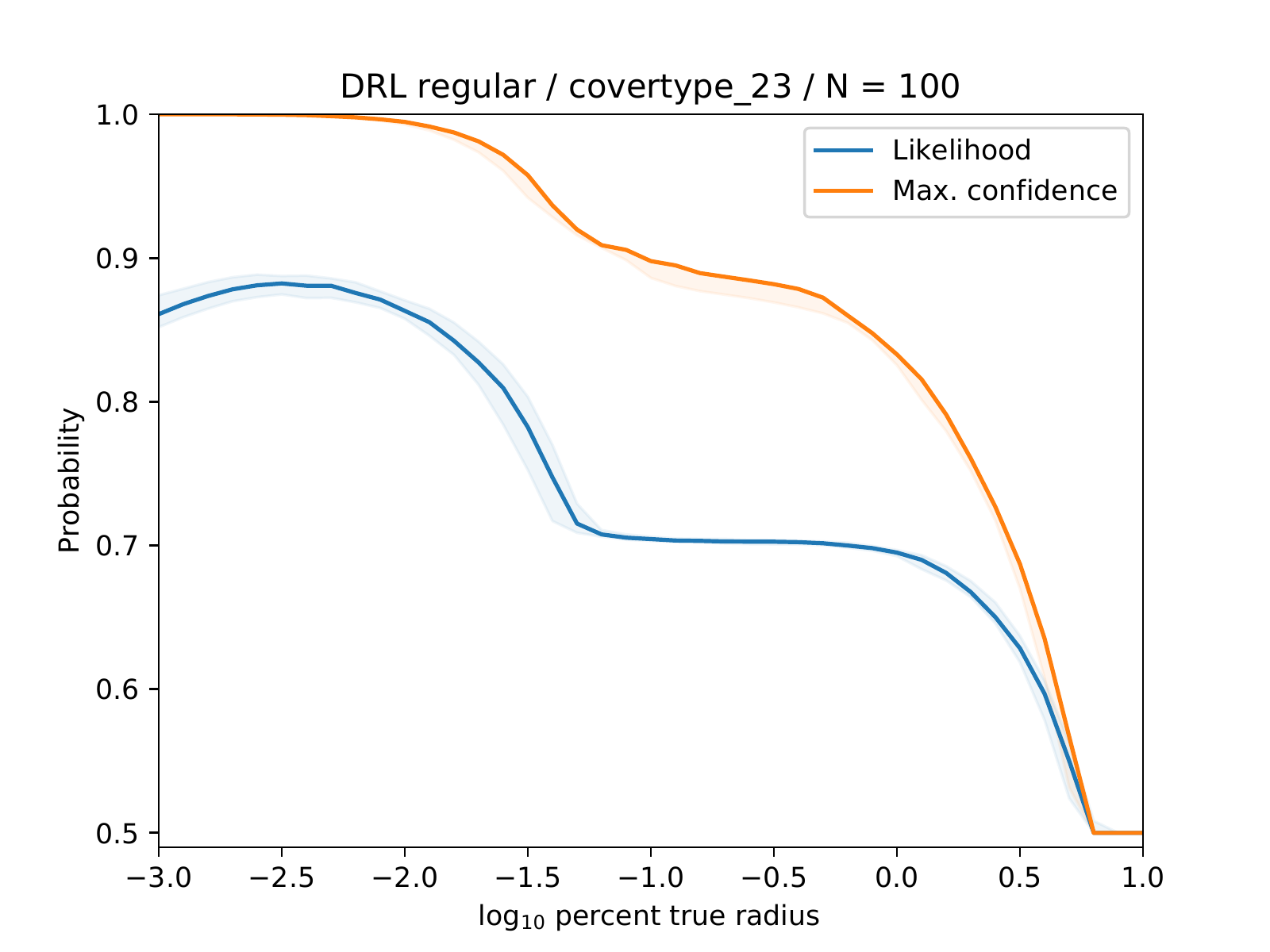}
\caption{Cover type (2-3)}
\end{subfigure}
\begin{subfigure}[b]{0.49\textwidth}
\includegraphics[width=\textwidth]{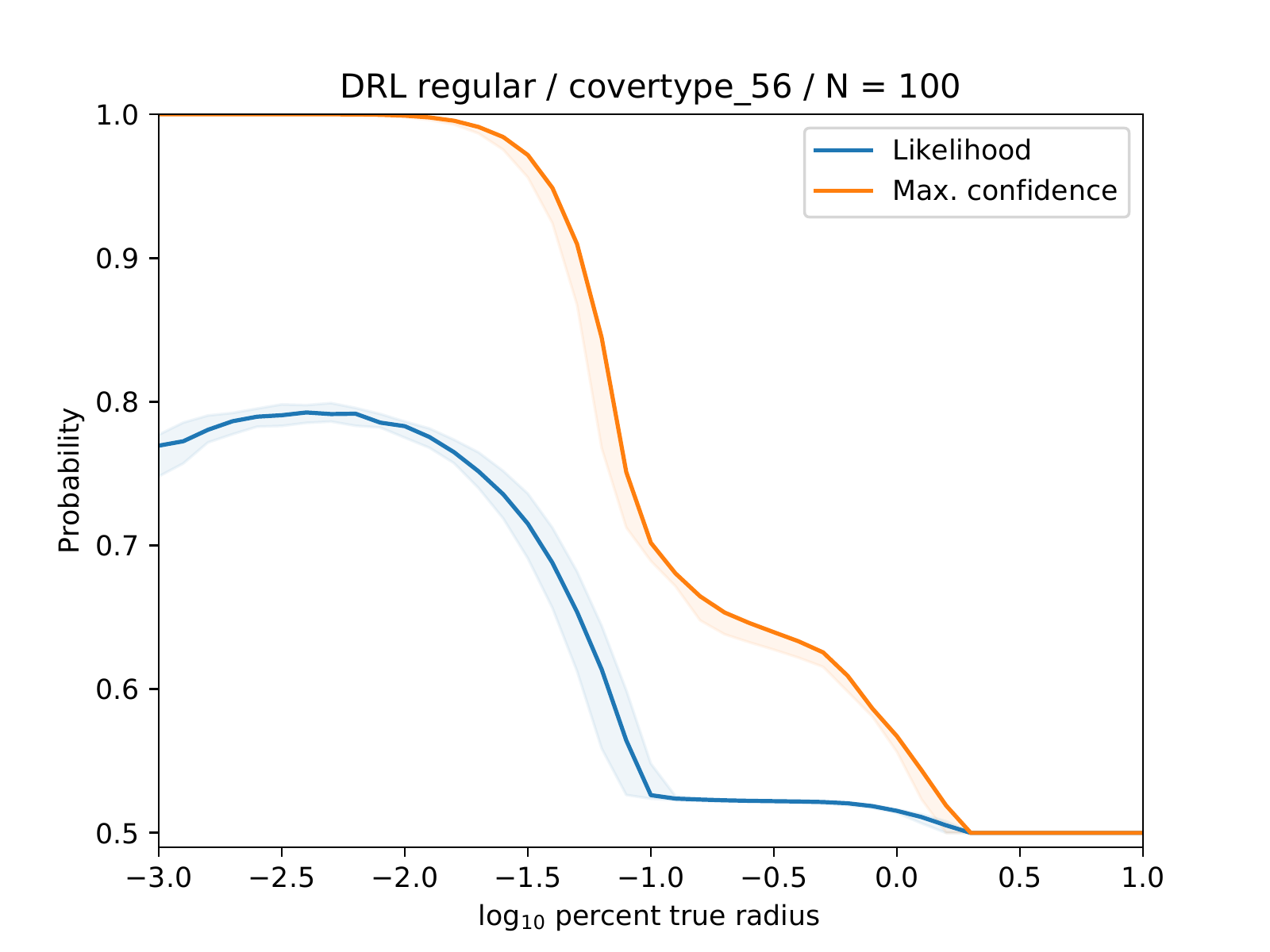}
\caption{Cover type (5-6)}
\end{subfigure}
\caption{Traditional Wasserstein DRL. Out-of-sample performance (likelihood) and maximum confidence vs. radius of robustness $\varepsilon$ as a percentage of the distance to the true data distribution $\Ppr$. Performance shows a bias-variance tradeoff with peak at $\varepsilon$ much smaller than the distance to $\Ppr$. Confidence often drops sharply at a radius much smaller than the distance to $\Ppr$.}
\label{fig:likelihood-and-confidence-vs-radius}
\end{figure}

\begin{figure}
\centering
\begin{subfigure}[b]{0.49\textwidth}
\includegraphics[width=\textwidth]{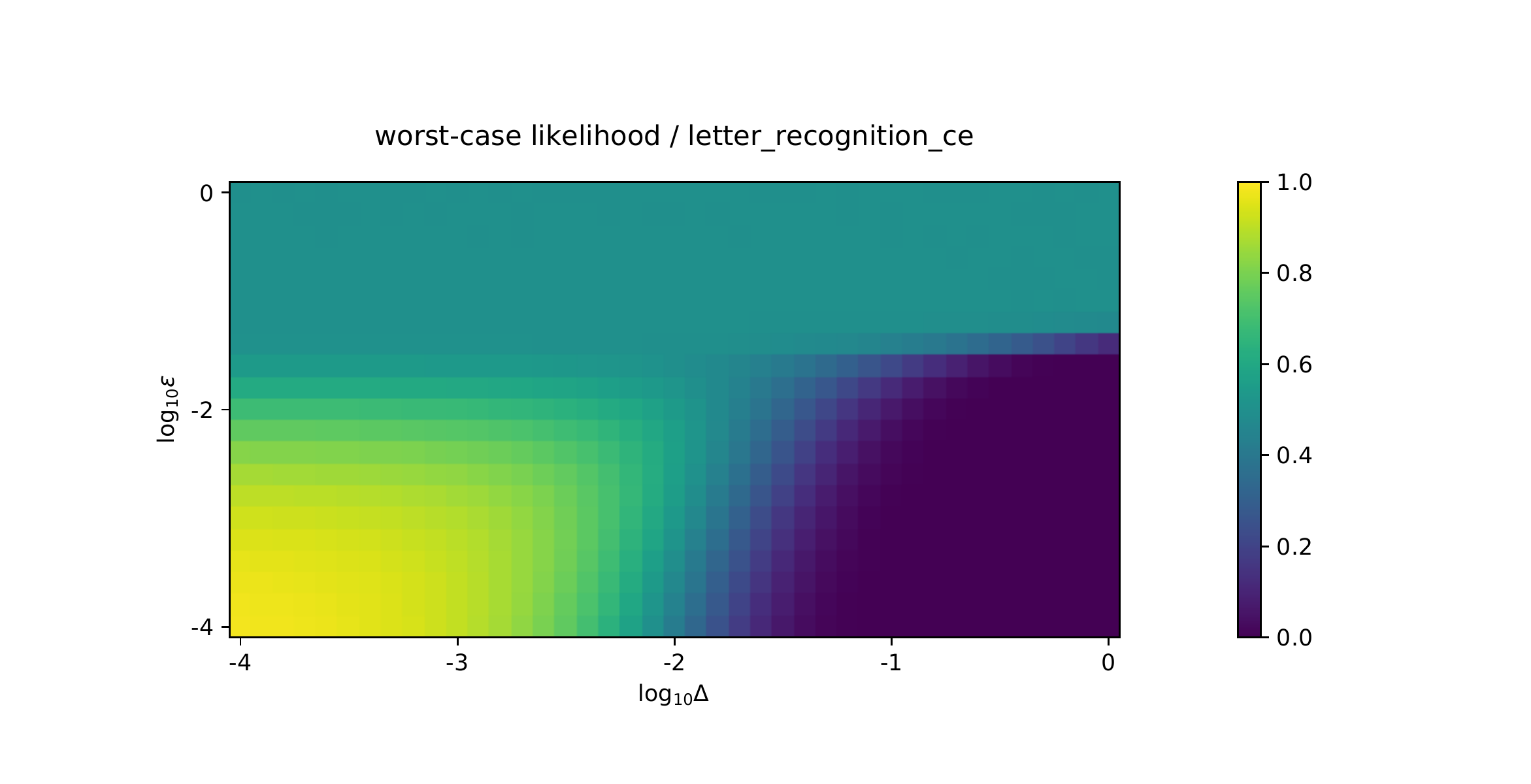}
\caption{Letter recognition (C-E)}
\end{subfigure}
\begin{subfigure}[b]{0.49\textwidth}
\includegraphics[width=\textwidth]{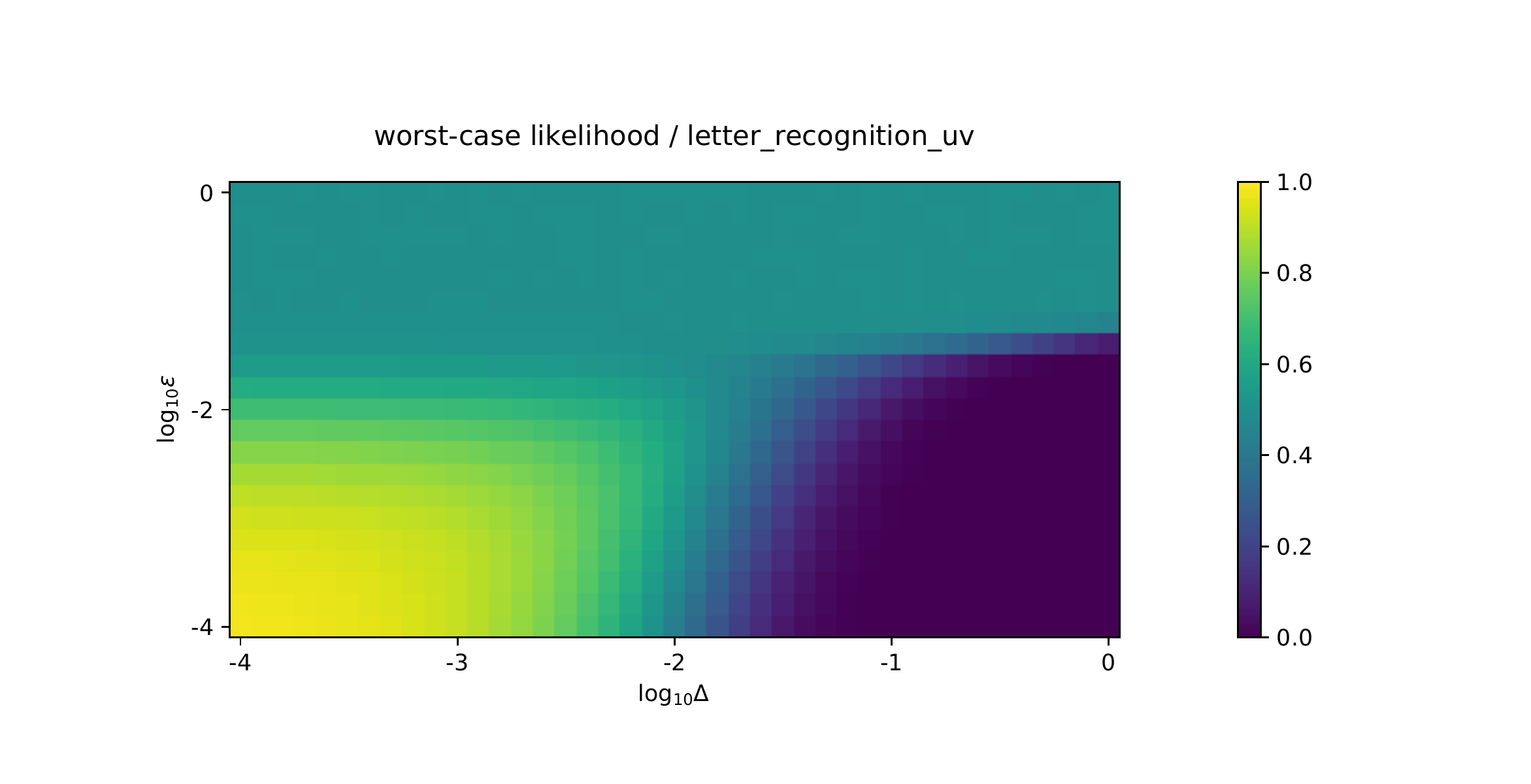}
\caption{Letter recognition (U-V)}
\end{subfigure}
\begin{subfigure}[b]{0.49\textwidth}
\includegraphics[width=\textwidth]{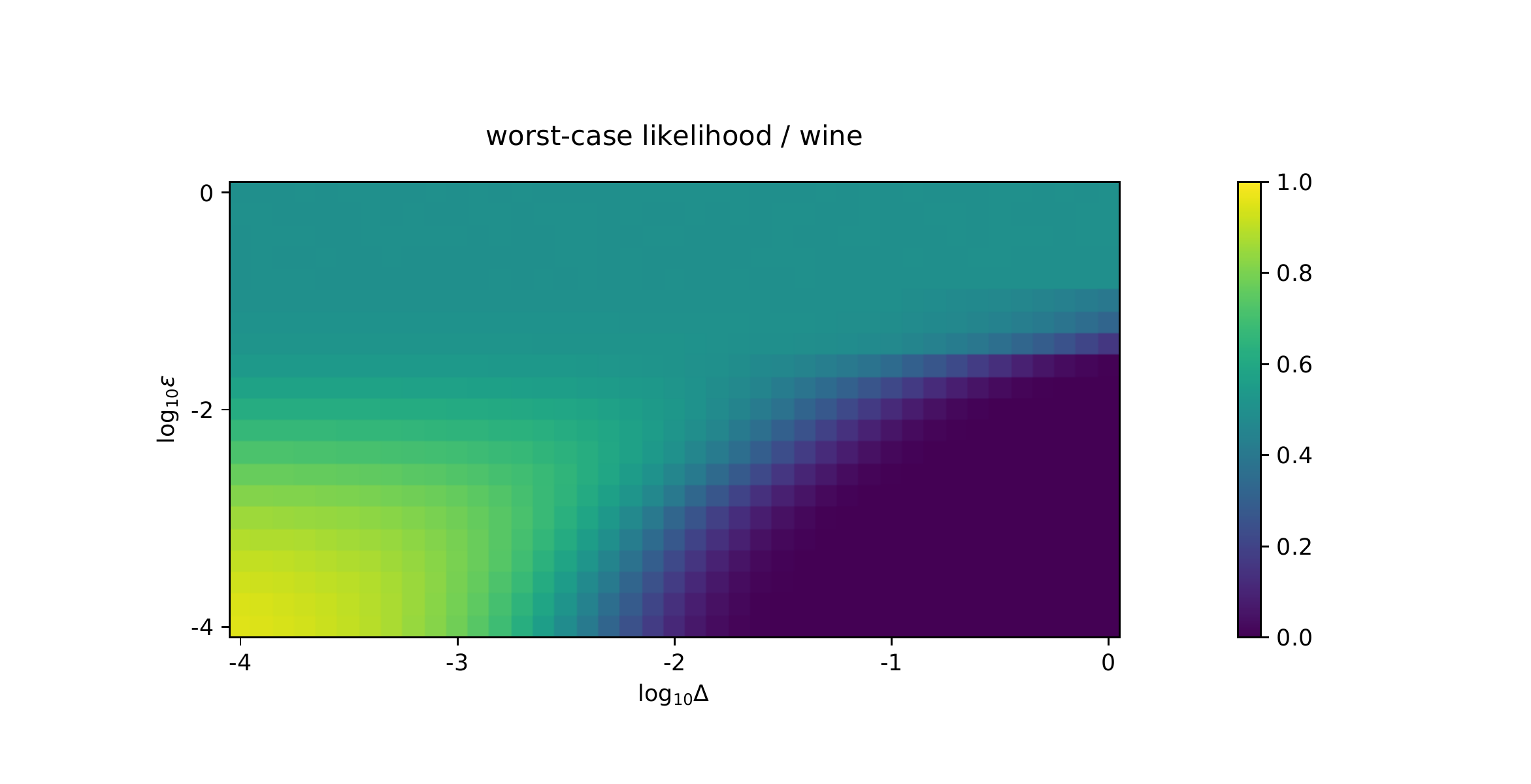}
\caption{Wine}
\end{subfigure}
\begin{subfigure}[b]{0.49\textwidth}
\includegraphics[width=\textwidth]{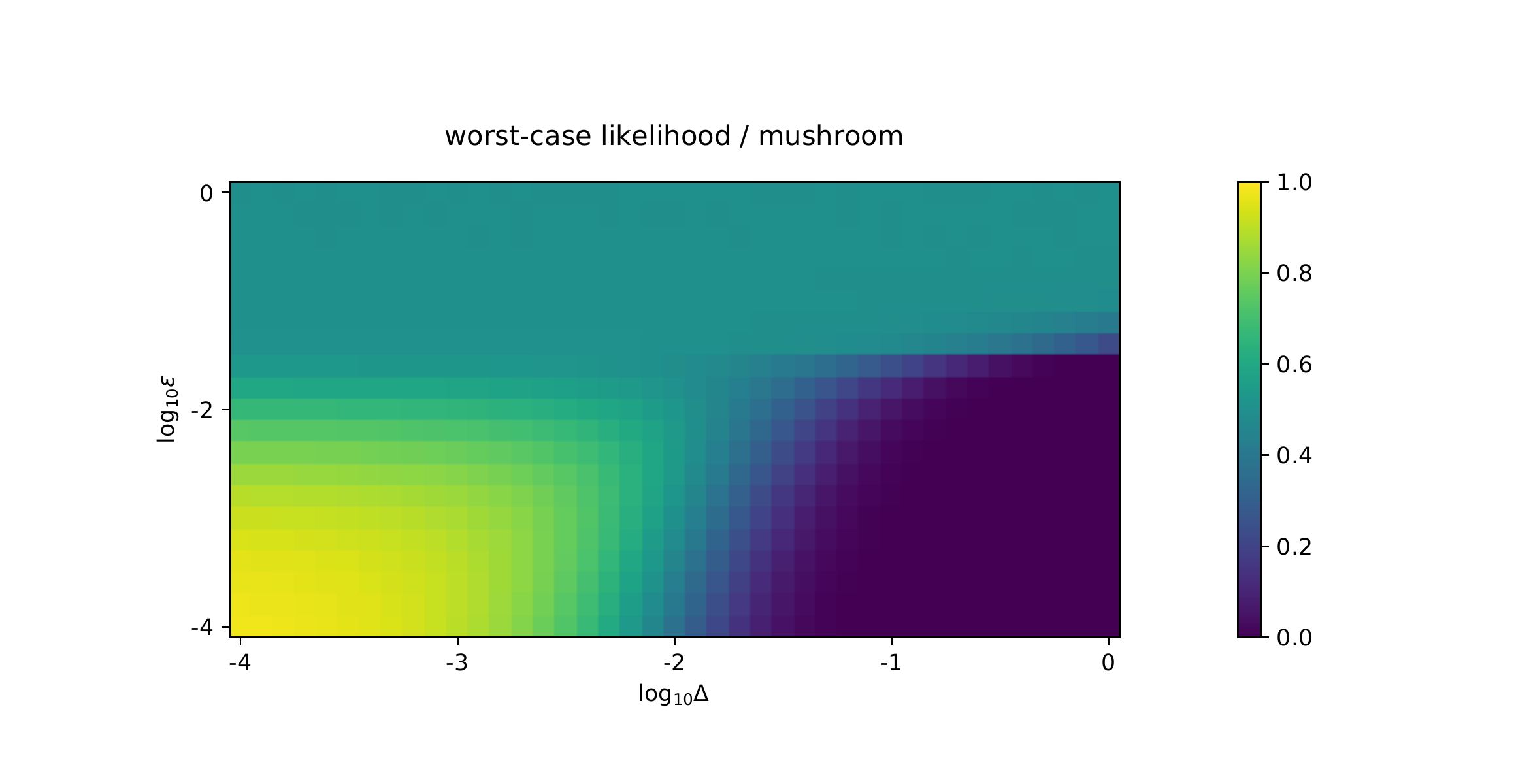}
\caption{Mushroom}
\end{subfigure}
\begin{subfigure}[b]{0.49\textwidth}
\includegraphics[width=\textwidth]{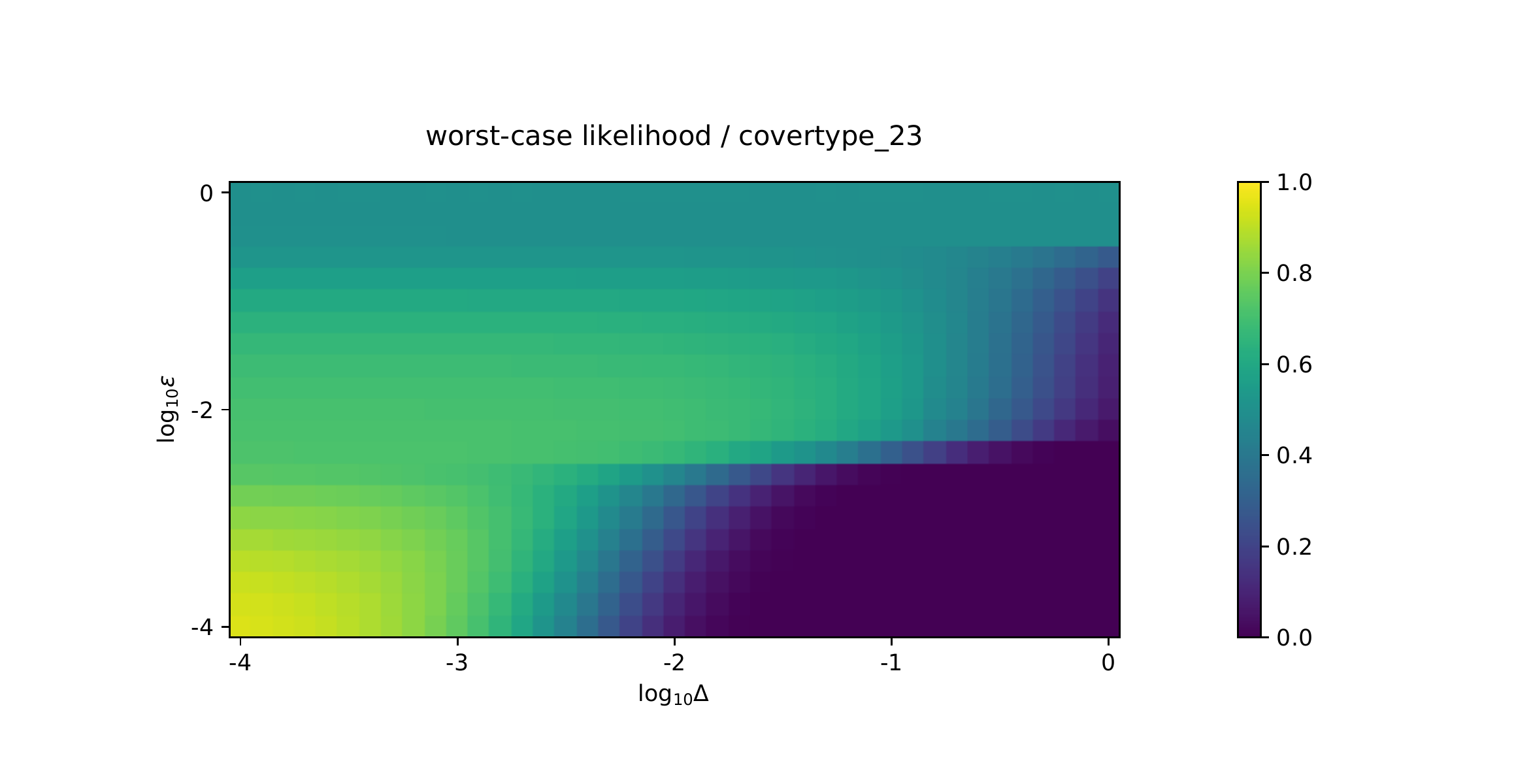}
\caption{Cover type (2-3)}
\end{subfigure}
\begin{subfigure}[b]{0.49\textwidth}
\includegraphics[width=\textwidth]{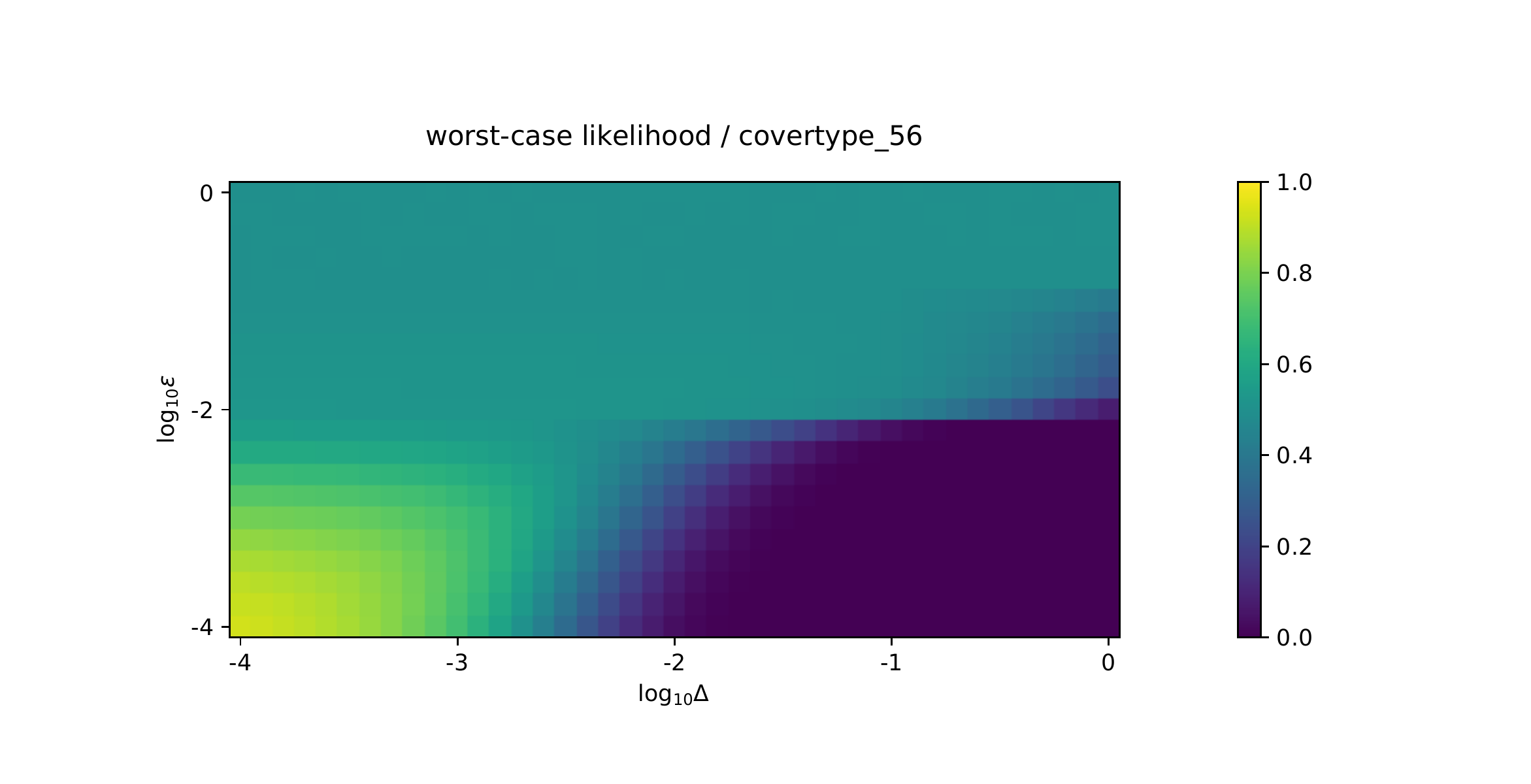}
\caption{Cover type (5-6)}
\end{subfigure}
\caption{Traditional Wasserstein DRL. Worst-case performance (likelihood) vs. radius of robustness $\varepsilon$ and test-time data radius $\varepsilon + \Delta$. Yellow indicates perfectly correct prediction (likelihood $1$), blue perfectly incorrect (likelihood $0$), and green perfectly indecisive prediction (likelihood $0.5$). Training with radius $\varepsilon$ confers little robustness beyond $\varepsilon$.}
\label{fig:worst-case-likelihood-vs-radius-plus-delta}
\end{figure}

\begin{figure}
\centering
\begin{subfigure}[b]{0.49\textwidth}
\includegraphics[width=\textwidth]{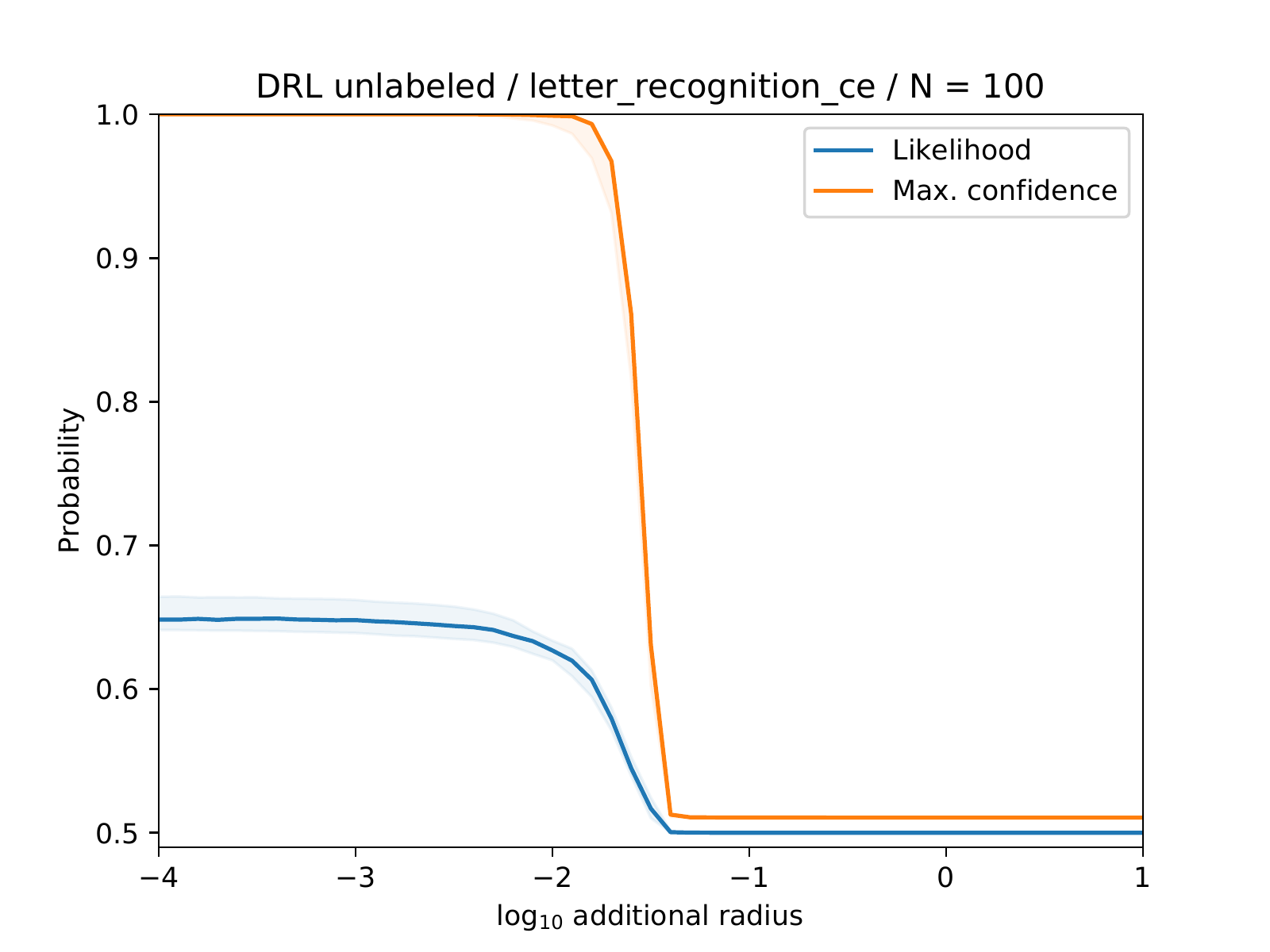}
\caption{Letter recognition (C-E)}
\end{subfigure}
\begin{subfigure}[b]{0.49\textwidth}
\includegraphics[width=\textwidth]{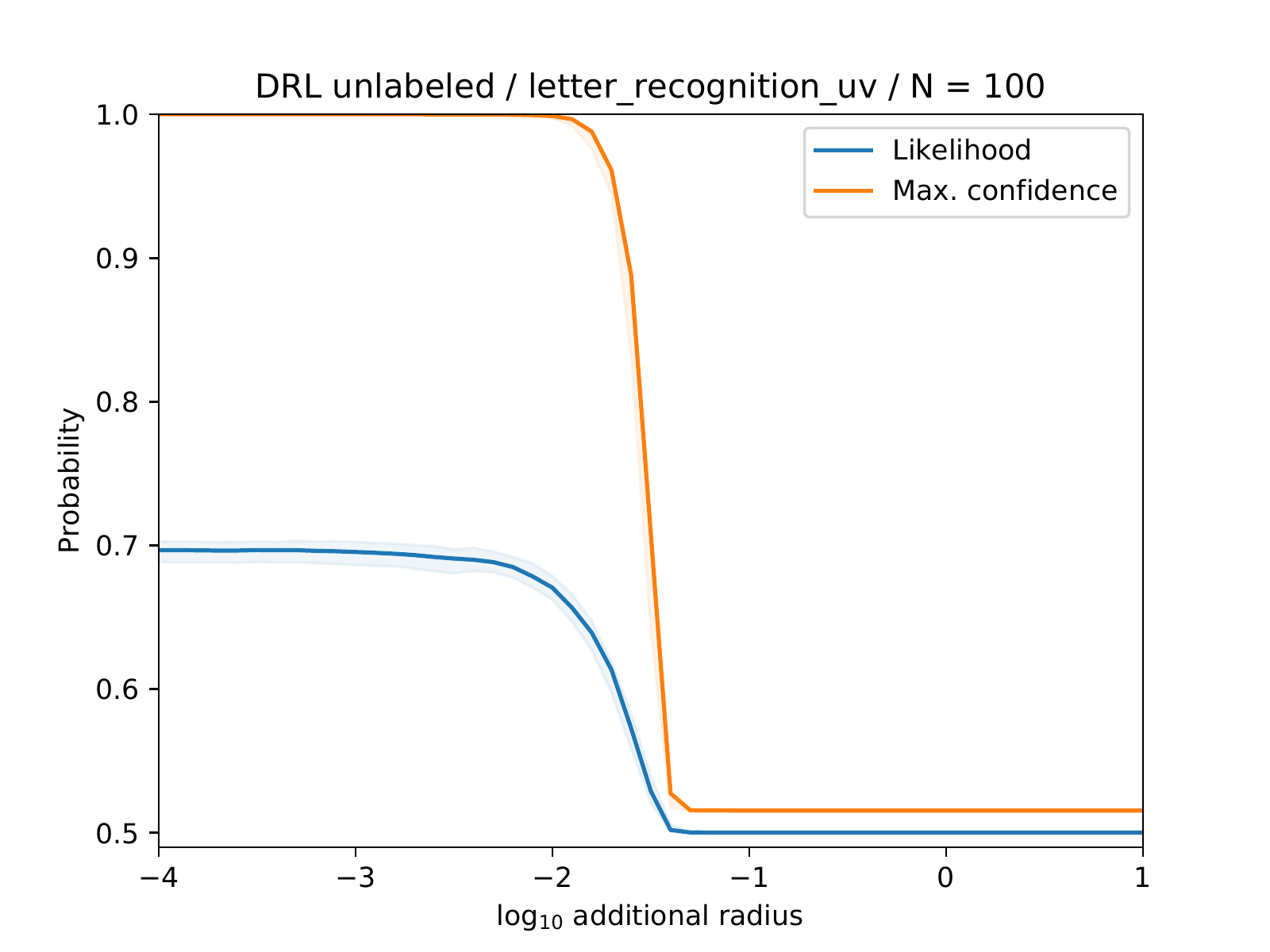}
\caption{Letter recognition (U-V)}
\end{subfigure}
\begin{subfigure}[b]{0.49\textwidth}
\includegraphics[width=\textwidth]{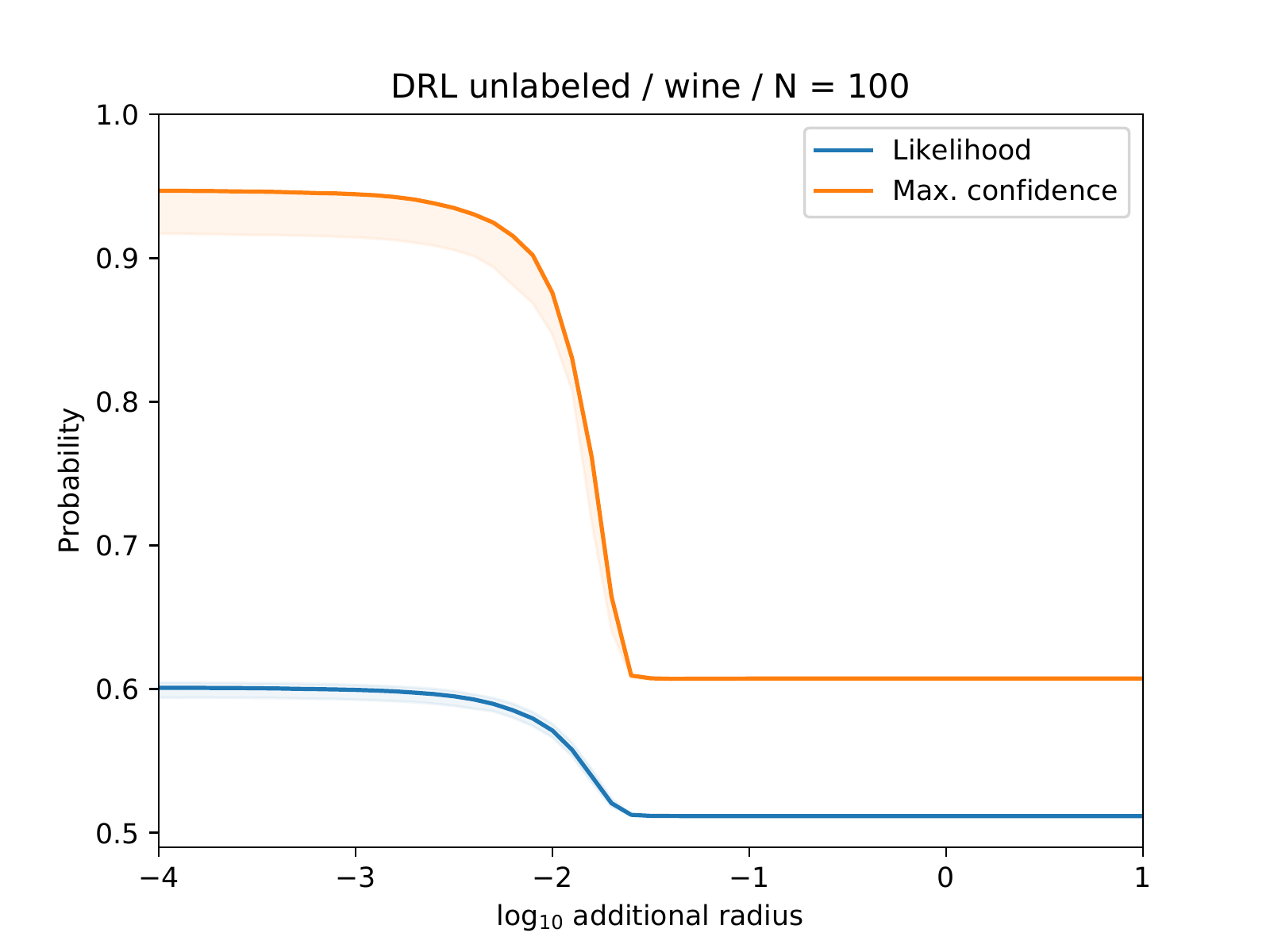}
\caption{Wine}
\end{subfigure}
\begin{subfigure}[b]{0.49\textwidth}
\includegraphics[width=\textwidth]{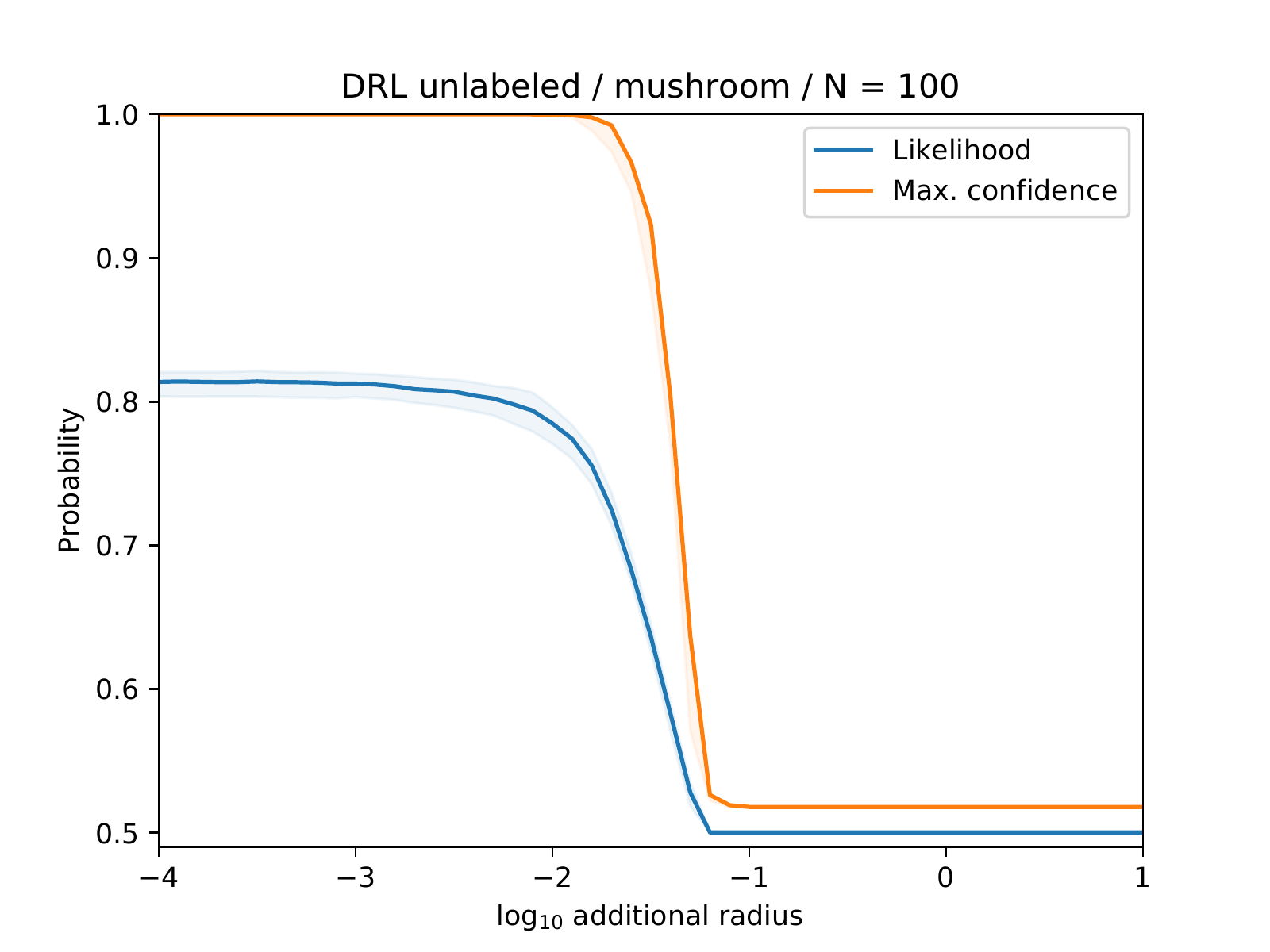}
\caption{Mushroom}
\end{subfigure}
\begin{subfigure}[b]{0.49\textwidth}
\includegraphics[width=\textwidth]{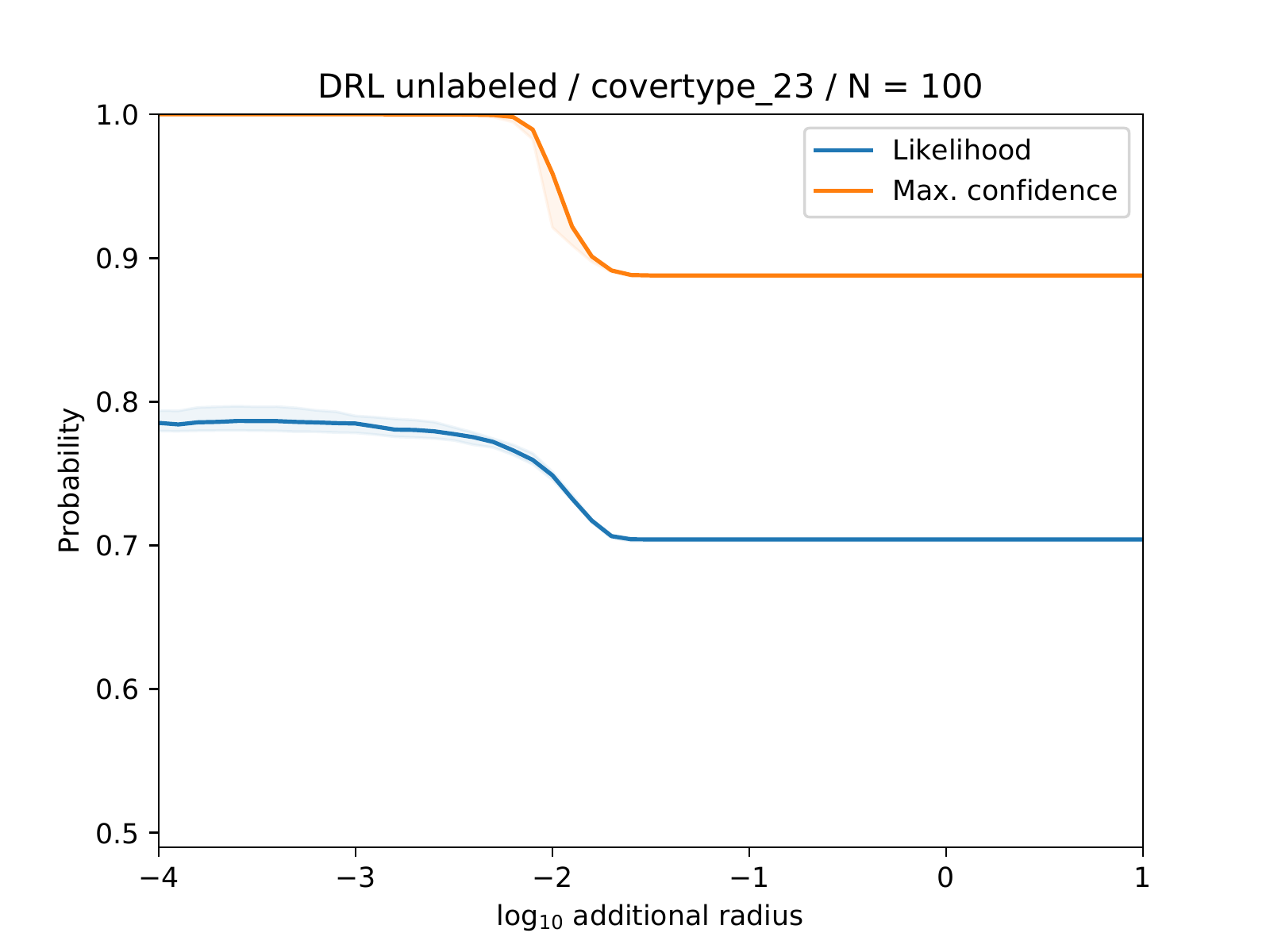}
\caption{Cover type (2-3)}
\end{subfigure}
\begin{subfigure}[b]{0.49\textwidth}
\includegraphics[width=\textwidth]{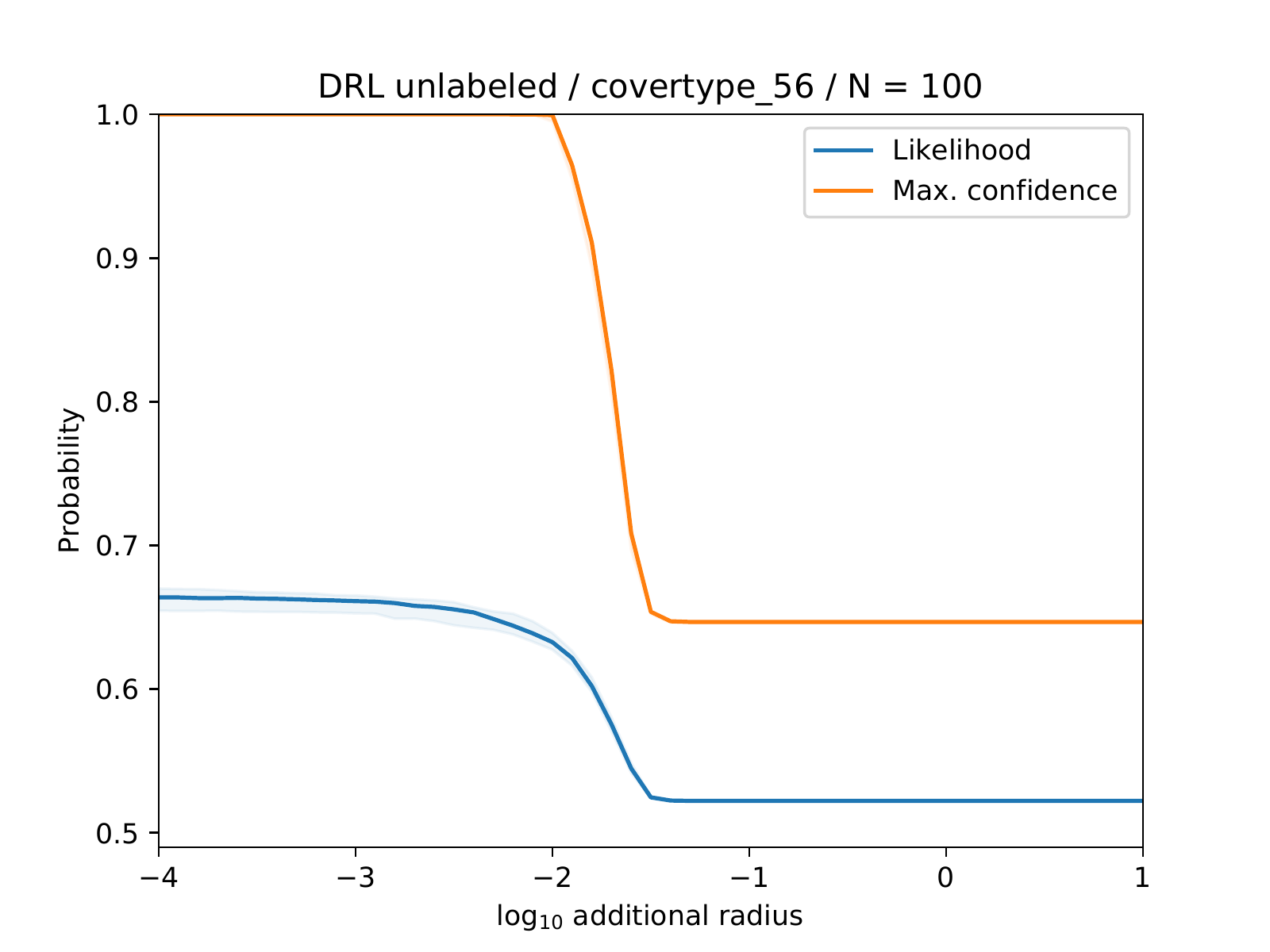}
\caption{Cover type (5-6)}
\end{subfigure}
\caption{DRL with unlabeled data. Out-of-sample performance (likelihood) and maximum confidence vs. difference between the radius of robustness $\varepsilon$ and the minimal radius necessary for the decision set to be nonempty. Unlike with traditional Wasserstein DRL, here there is no apparent bias-variance tradeoff. Performance is flat out to a radius at which the confidence drops sharply.}
\label{fig:likelihood-and-confidence-vs-additional-radius}
\end{figure}

The choice of $\varepsilon$ for the proposed method is constrained by the fact that {\bf the feasible set is empty} for $\varepsilon$ below a threshold, as there might be no distribution in the ball $\ball_{\varepsilon}(\hat{\Ppr}_l)$ having the desired marginals $\Ppr_{\Xspace}$ and $\Ppr_{\Yspace}$. This situation is easily detected in practice, as the value of the dual $g(\theta)$ becomes unbounded below.

Empirically, with the proposed method, we find no evidence of a bias-variance tradeoff as the radius $\varepsilon$ is varied, unlike traditional Wasserstein DRL. Figure \ref{fig:likelihood-and-confidence-vs-additional-radius} shows out-of-sample performance as we vary the difference between the radius $\varepsilon$ and the minimal such radius for which the feasible set is nonempty. The performance is flat out to a radius beyond which the confidence of the learner decreases quickly. Appendix \ref{sec:appendix-no-bias-variance} contains further examples.

This last observation suggests suggests a criterion for choosing $\varepsilon$ under the proposed DRL model: One chooses the maximum $\varepsilon$ such that the confidence of the learned classifier is above a threshold. This is the {\bf as-robust-as-possible} selection, as opposed to the maximum-cross-validation-performance selection often used in traditional DRL. There are multiple possible implementations of this criterion. In our experiments, for example, thresholding the median confidence on the unlabeled set at $0.7$ often suffices to ensure that $\varepsilon$ is large enough to ensure $\Ppr \in \ball_{\varepsilon}(\hat{\Ppr}_l)$, for reasonable values of $N_l$ (Figure \ref{fig:true-radius-conf}).

\subsection{Empirical performance of learning with unlabeled data}
\label{sec:empirical-performance-of-drl-with-unlabeled}

In this section, we demonstrate the impact of the proposed method for constraining the adversary's decision set using unlabeled data. We evaluate the performance guarantee offered by the previously-proposed distributionally robust logistic regression model \citep{abadeh2015distributionally} on several binary classification datasets,\footnote{Datasets are from the UCI machine learning repository \citep{dua2019:uci}.} and we compare it to the guarantee offered by our model under the assumption that the radius $\varepsilon$ of the Wasserstein ball $\ball_{\varepsilon}(\hat{\Ppr}_l)$ defining the adversary's decision set is chosen to include the true data distribution $\Ppr$.

For each dataset, we sample a small number $N_l$ of labeled examples and compute the radius $\varepsilon$ that is required to include the true (empirical) data distribution $\Ppr$ in the Wasserstein ball $\ball_{\varepsilon}(\hat{\Ppr}_l)$. This is the smallest $\varepsilon$ for which the performance guarantee from DRL holds. We use the labeled examples to compute the distributionally robust logistic regression under the traditional model \citep{abadeh2015distributionally} and additionally use the set of unlabeled examples to compute the same regression under the proposed model. We compare the performance guarantee (i.e.\ the dual objective value) computed under each DRL model. Identical values of $\varepsilon$ are used for both methods, but a different value of $\varepsilon$ is computed for each sampled set of labeled examples $\hat{\Zspace}_l$.

We examine two settings for the proposed method. The first assumes a {\bf strong prior} that specifies the exact (true) label probabilities, such that $\underline{\pvec}_{\Yspace} = \overline{\pvec}_{\Yspace}$. In practice such a strong prior might come from auxiliary data, such as in ecological inference or with domain knowledge. The second setting assumes a {\bf weak prior} that specifies only $95\%$ confidence intervals for the label probabilities, estimated directly from the from labeled data $\hat{\Zspace}_l$ \citep{clopper1934use}. 

We vary the number of labeled examples and examine the computed performance guarantee, shown in Figure \ref{fig:true-radius-fval}, as well as the median confidence of the learned predictor, shown in Figure \ref{fig:true-radius-conf}. The former is the worst-case guarantee \eqref{eq:worst-case-primal} and not the actual generalization performance of the learned classifier. We make three observations:
\begin{enumerate}
\item For all but one of the datasets, the performance bound computed by traditional DRL is vacuous (guaranteeing only likelihood greater than or equal to $0.5$), while the learned classifier is trivial (assigning equal probability to both classes), for all tested numbers of labeled examples (maximum $N_l = 1000$).
\item For all datasets, the proposed DRL using unlabeled data and either a strong prior and a weak prior on the label probabilities yields a non-vacuous performance bound and a non-trivial classifier, for $N_l$ at which traditional DRL is vacuous.
\item The strong prior on label probabilities can yield highly non-trivial performance bounds, for smaller numbers of labeled examples $N_l$ than the weak prior.
\end{enumerate}
We have chosen $\varepsilon$ as small as possible while ensuring the computed performance guarantee holds, and the performance bound computed by either algorithm gets monotonically worse as $\varepsilon$ increases.

\begin{figure}
\centering
\begin{subfigure}[b]{0.49\textwidth}
\includegraphics[width=\textwidth]{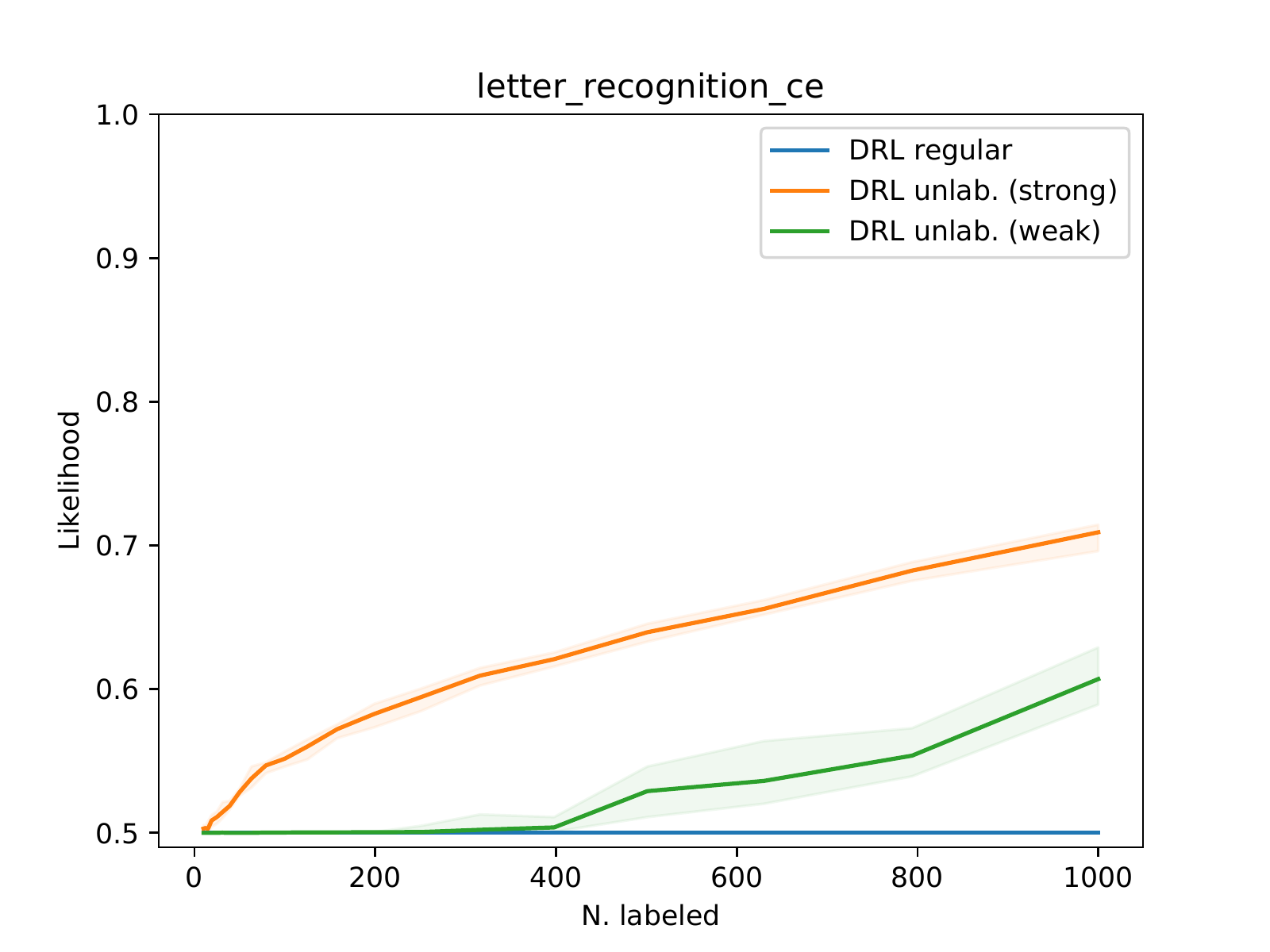}
\caption{Letter recognition (C-E)}
\end{subfigure}
\begin{subfigure}[b]{0.49\textwidth}
\includegraphics[width=\textwidth]{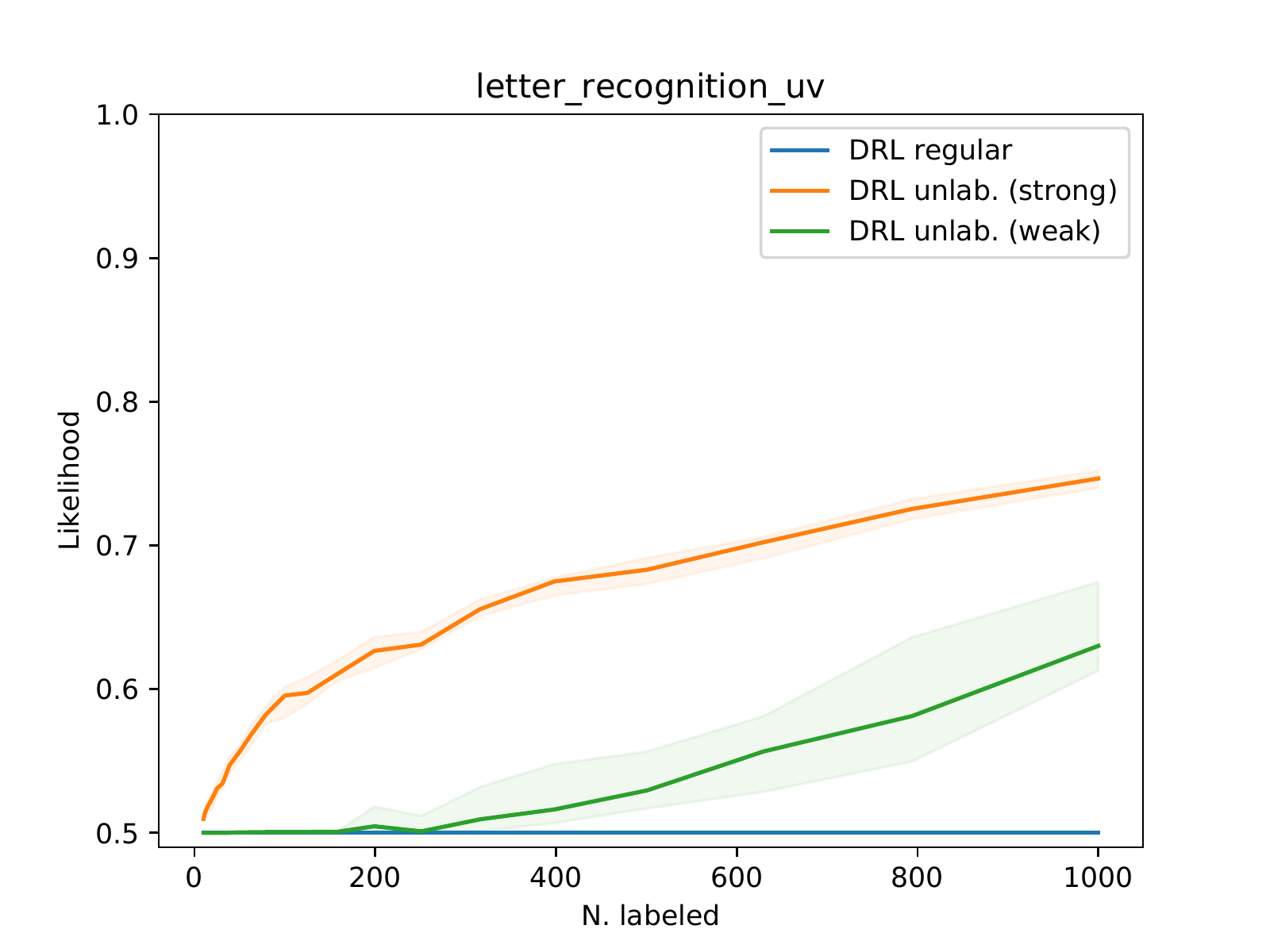}
\caption{Letter recognition (U-V)}
\end{subfigure}
\begin{subfigure}[b]{0.49\textwidth}
\includegraphics[width=\textwidth]{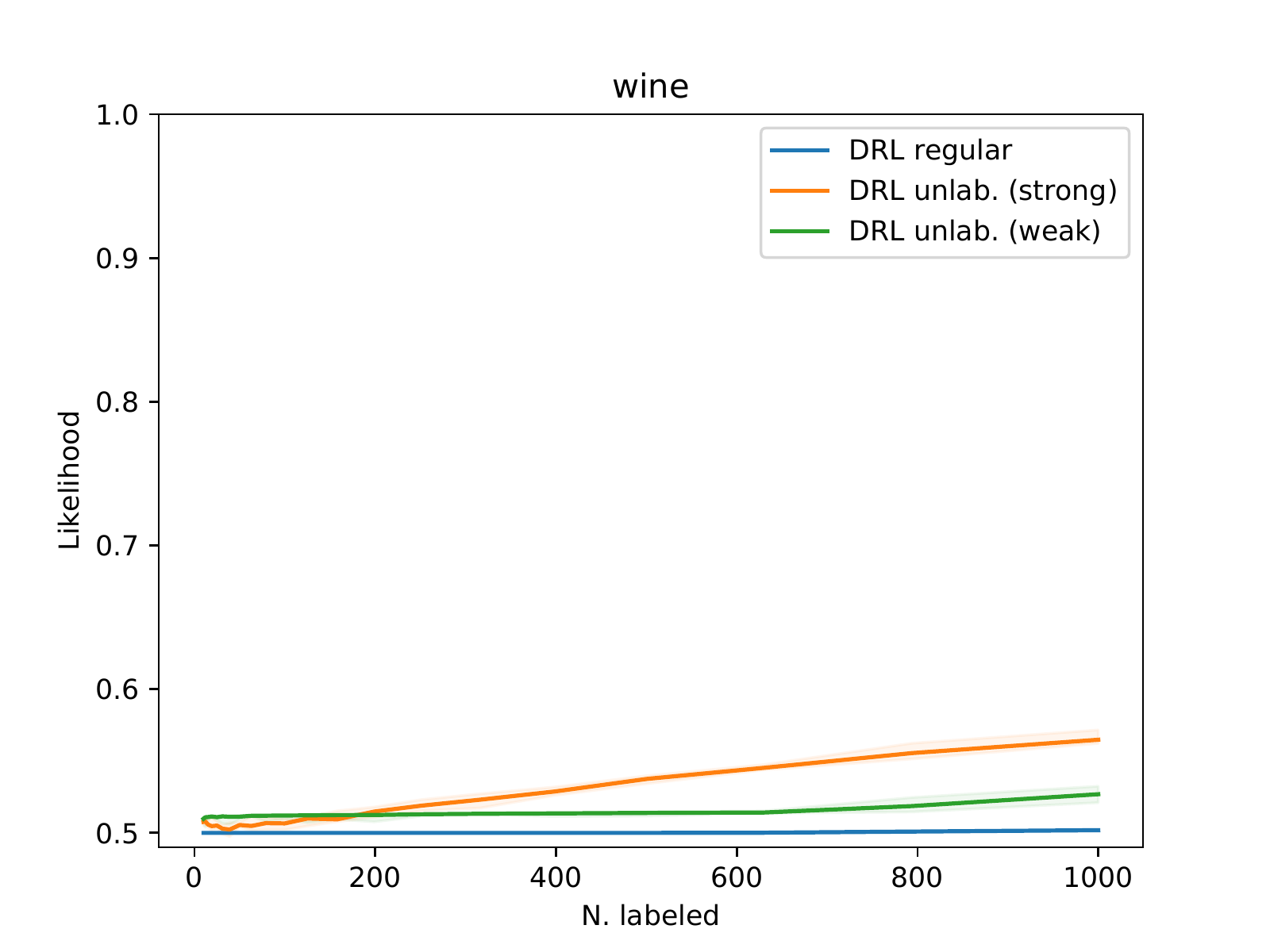}
\caption{Wine}
\end{subfigure}
\begin{subfigure}[b]{0.49\textwidth}
\includegraphics[width=\textwidth]{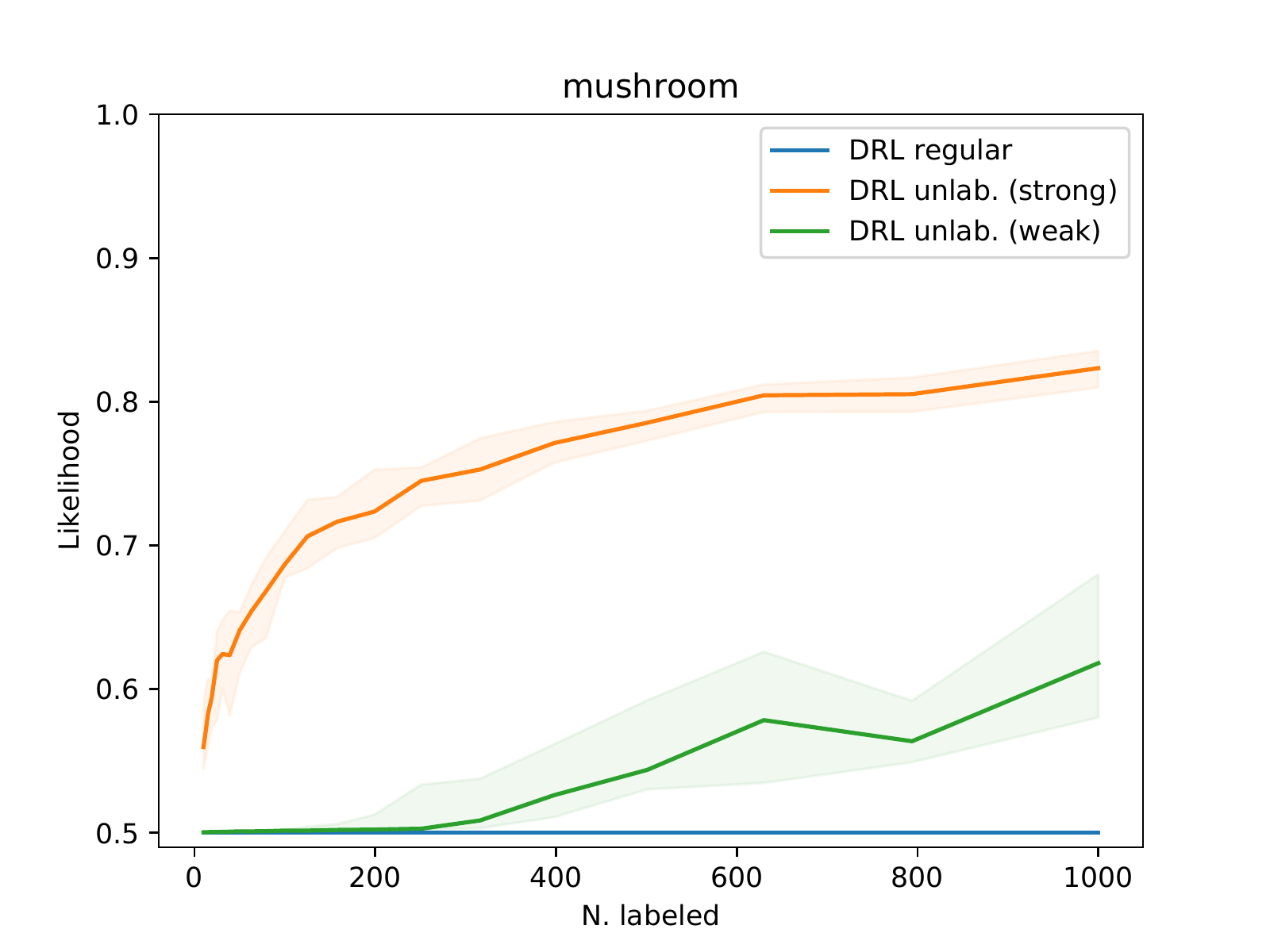}
\caption{Mushroom}
\end{subfigure}
\begin{subfigure}[b]{0.49\textwidth}
\includegraphics[width=\textwidth]{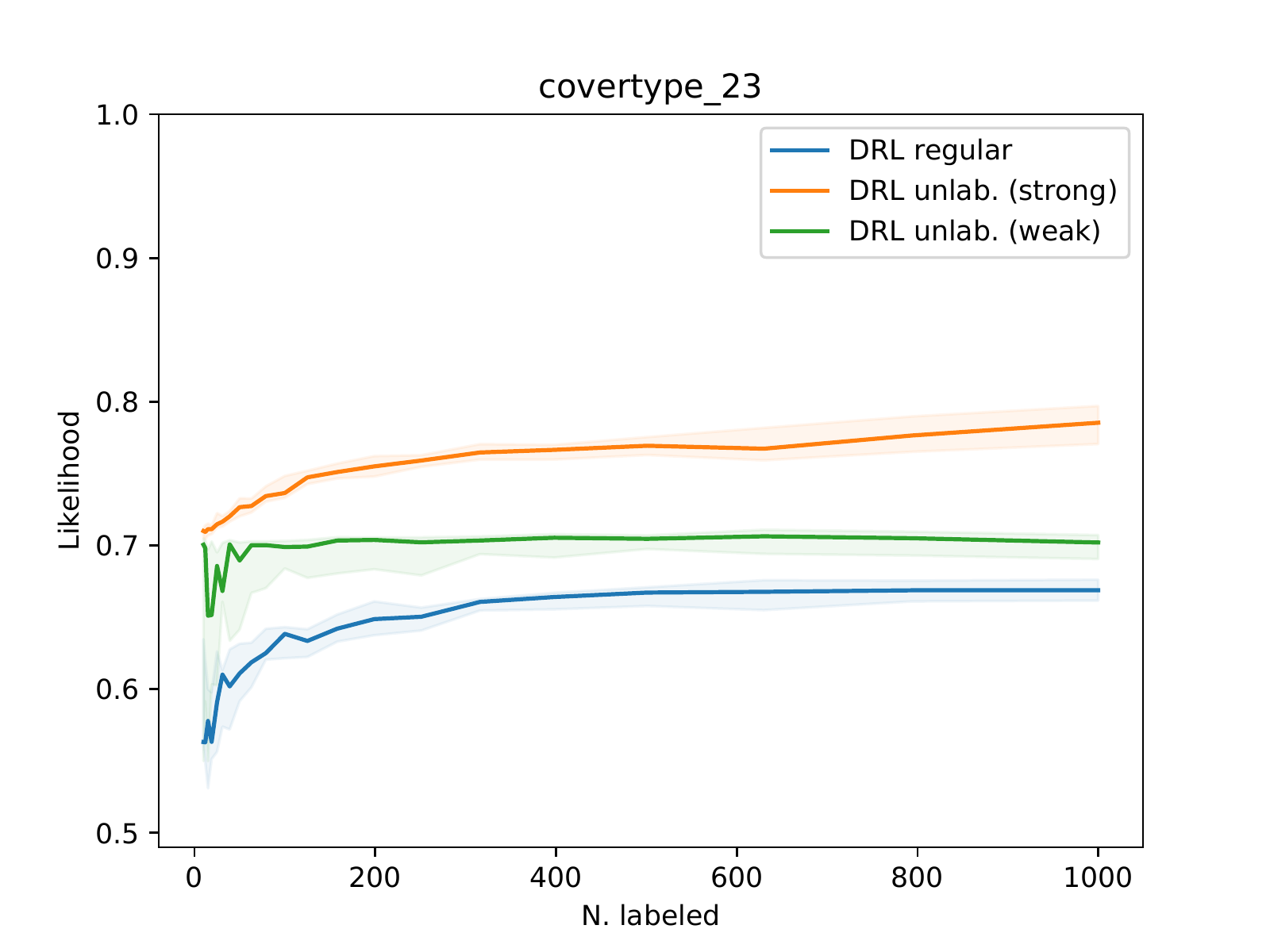}
\caption{Cover type (2-3)}
\end{subfigure}
\begin{subfigure}[b]{0.49\textwidth}
\includegraphics[width=\textwidth]{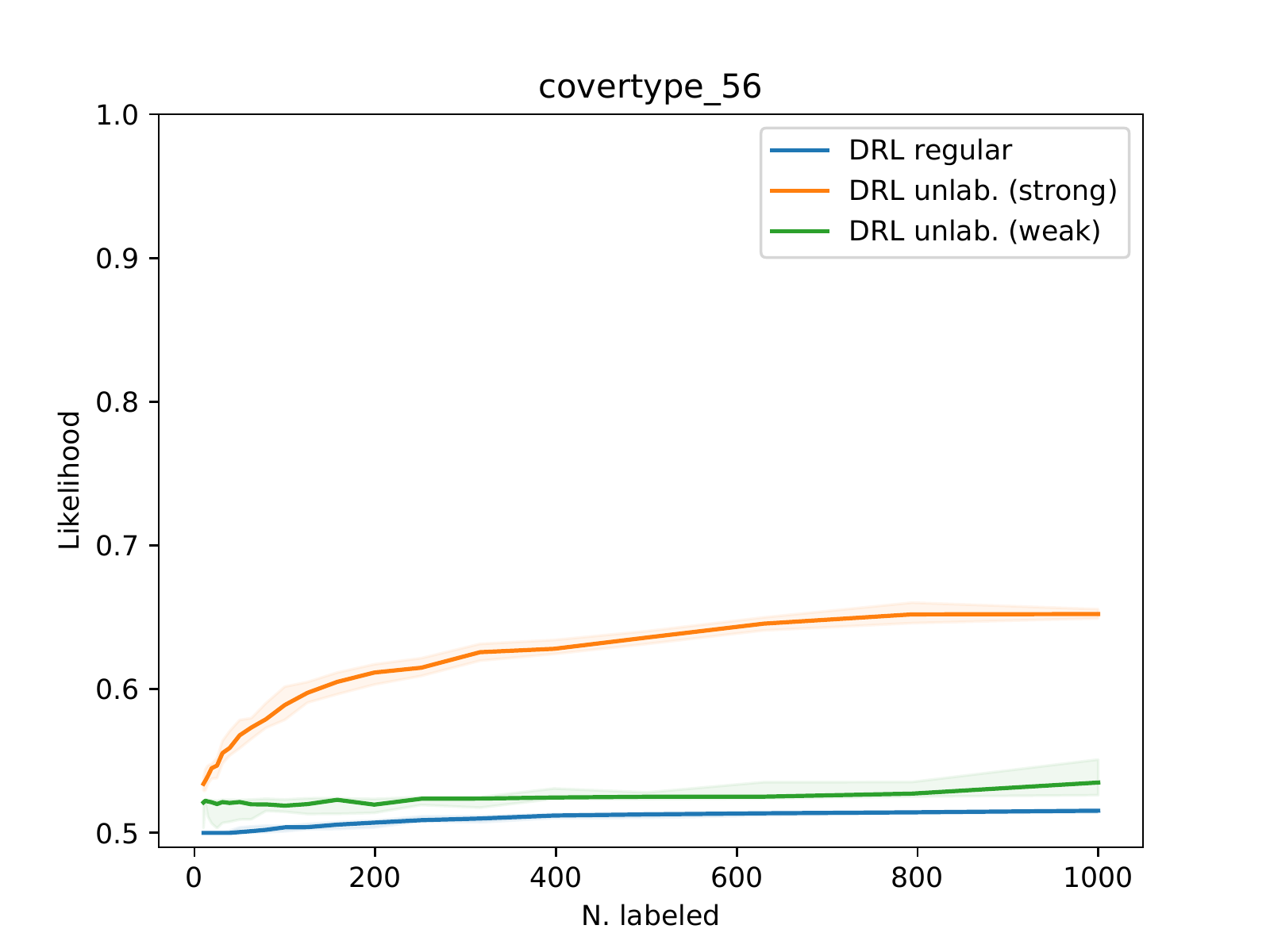}
\caption{Cover type (5-6)}
\end{subfigure}
\caption{Worst-case performance bound (likelihood) vs. number of labeled data, setting $\varepsilon$ to include the true test distribution. The regular DRL bound is often vacuous through $N_l = 1000$ while both DRL methods with unlabeled data yield non-vacuous bounds.}
\label{fig:true-radius-fval}
\end{figure}

\begin{figure}
\centering
\begin{subfigure}[b]{0.49\textwidth}
\includegraphics[width=\textwidth]{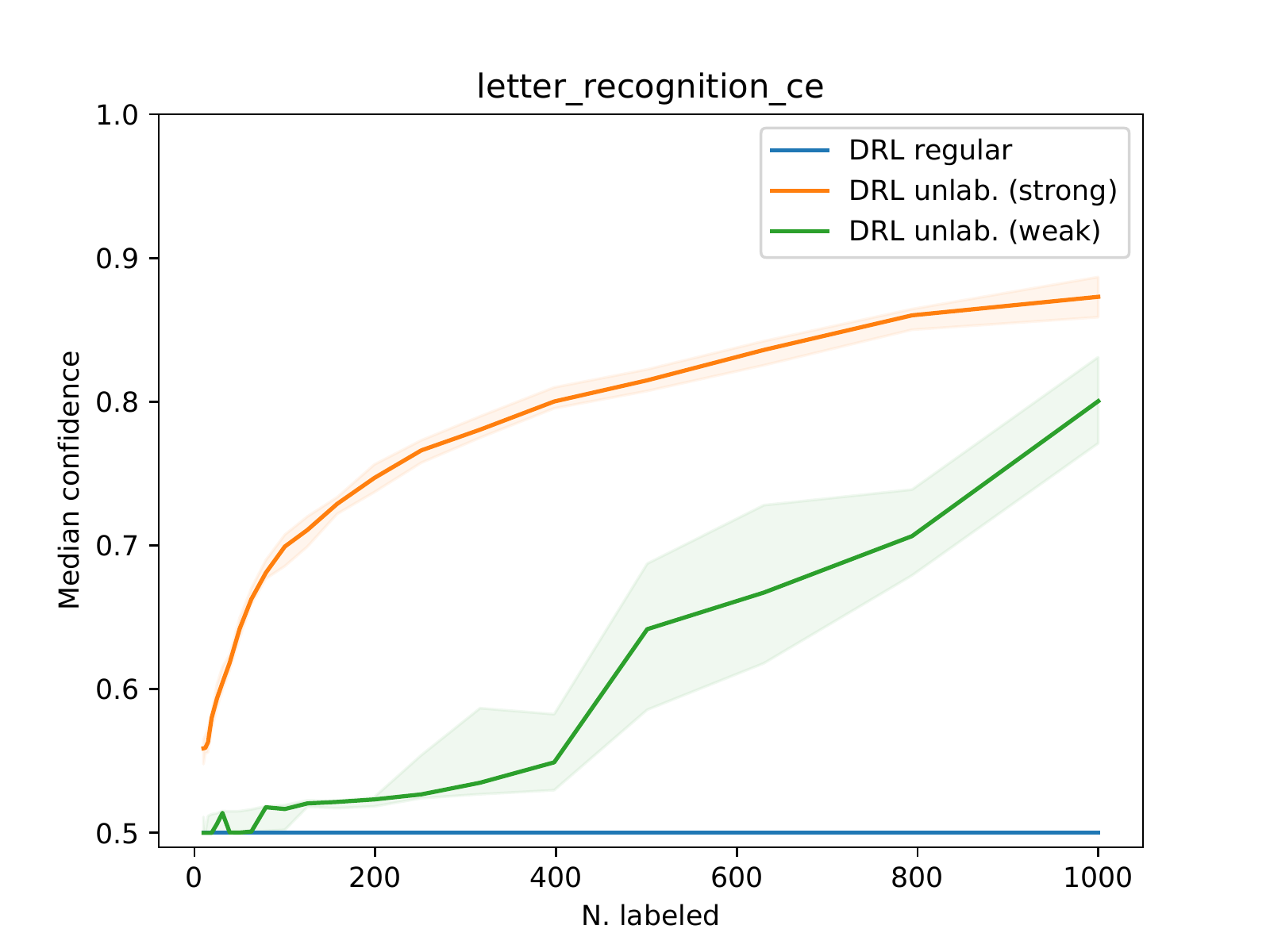}
\caption{Letter recognition (C-E)}
\end{subfigure}
\begin{subfigure}[b]{0.49\textwidth}
\includegraphics[width=\textwidth]{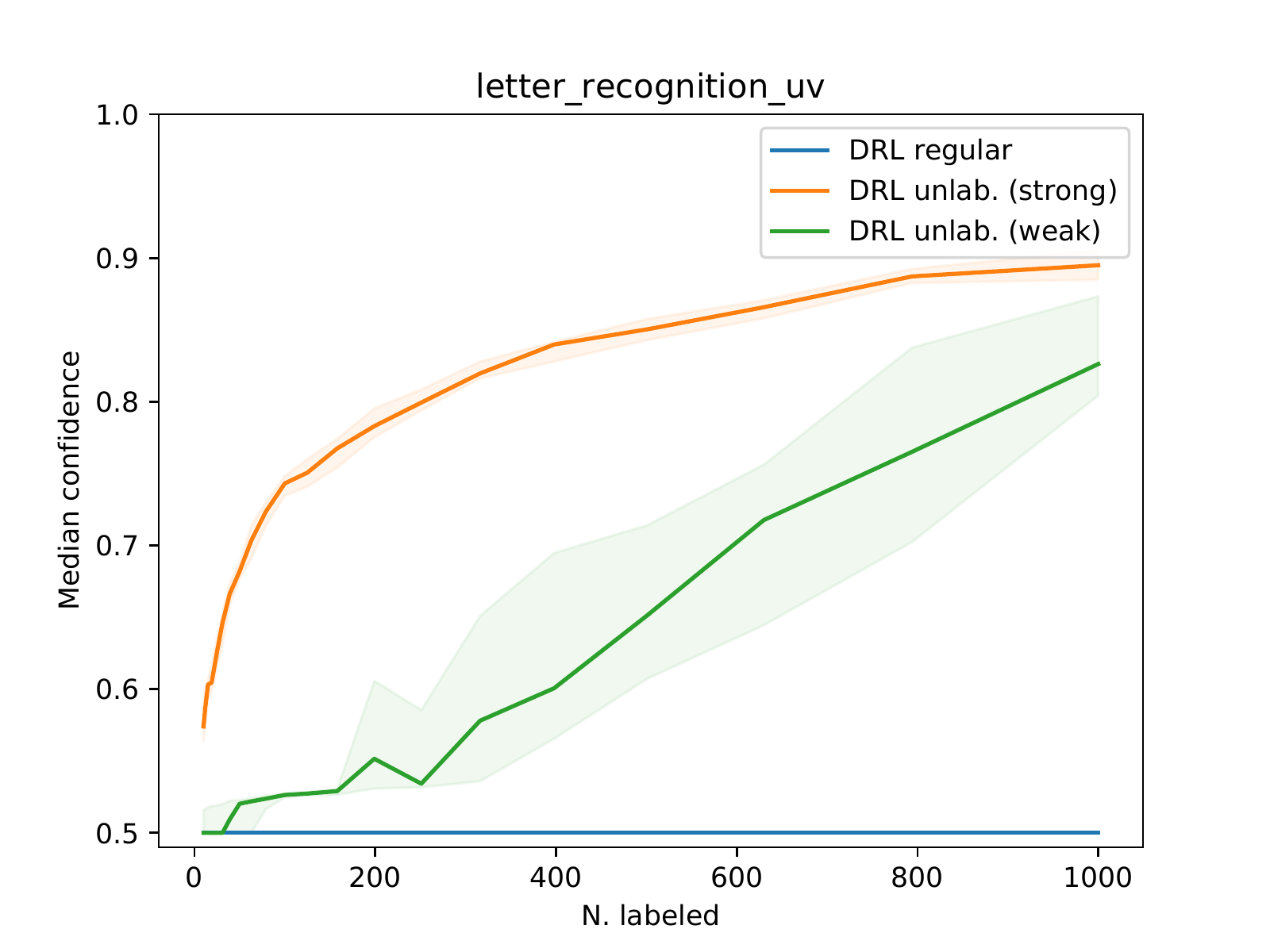}
\caption{Letter recognition (U-V)}
\end{subfigure}
\begin{subfigure}[b]{0.49\textwidth}
\includegraphics[width=\textwidth]{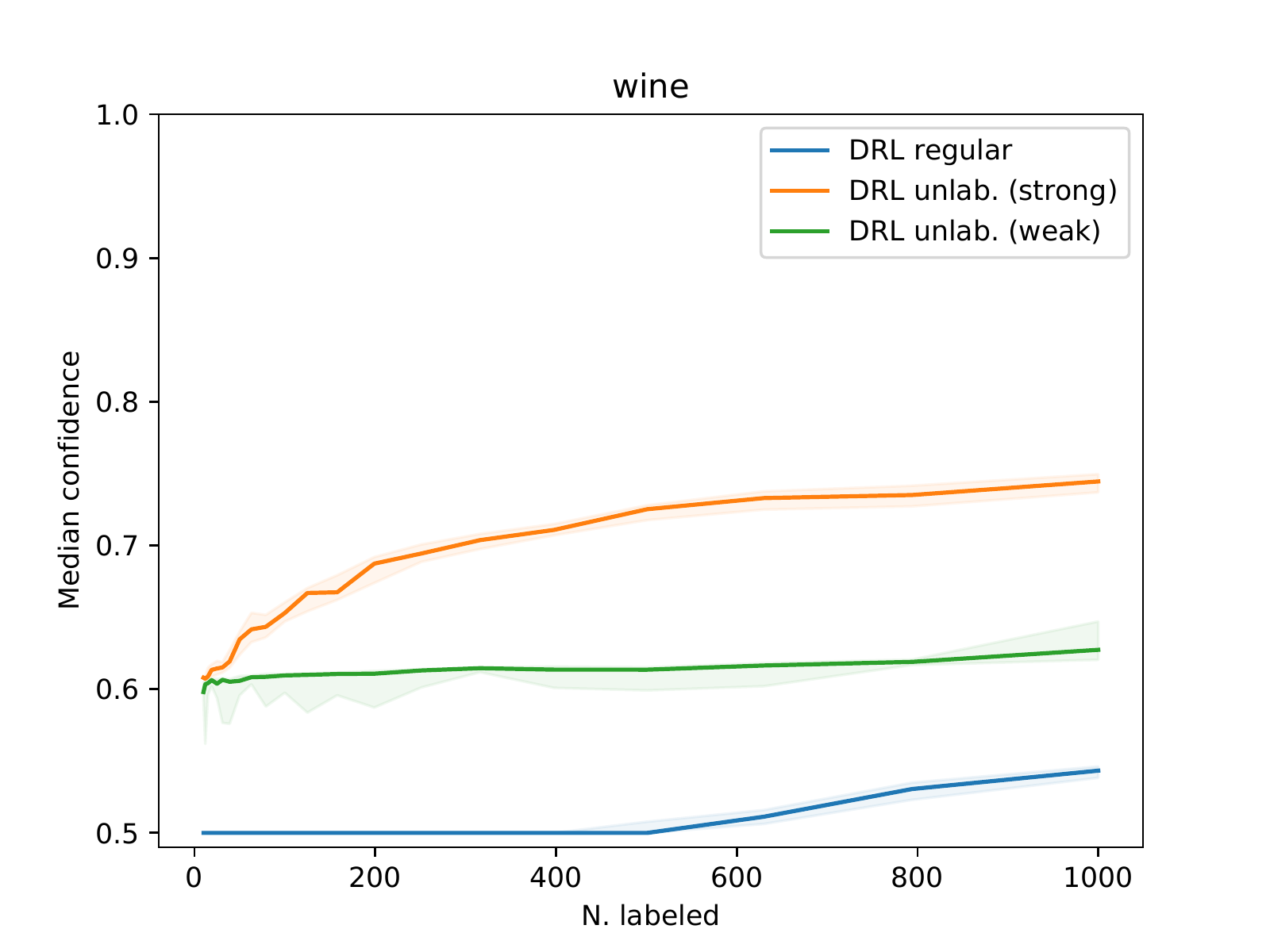}
\caption{Wine}
\end{subfigure}
\begin{subfigure}[b]{0.49\textwidth}
\includegraphics[width=\textwidth]{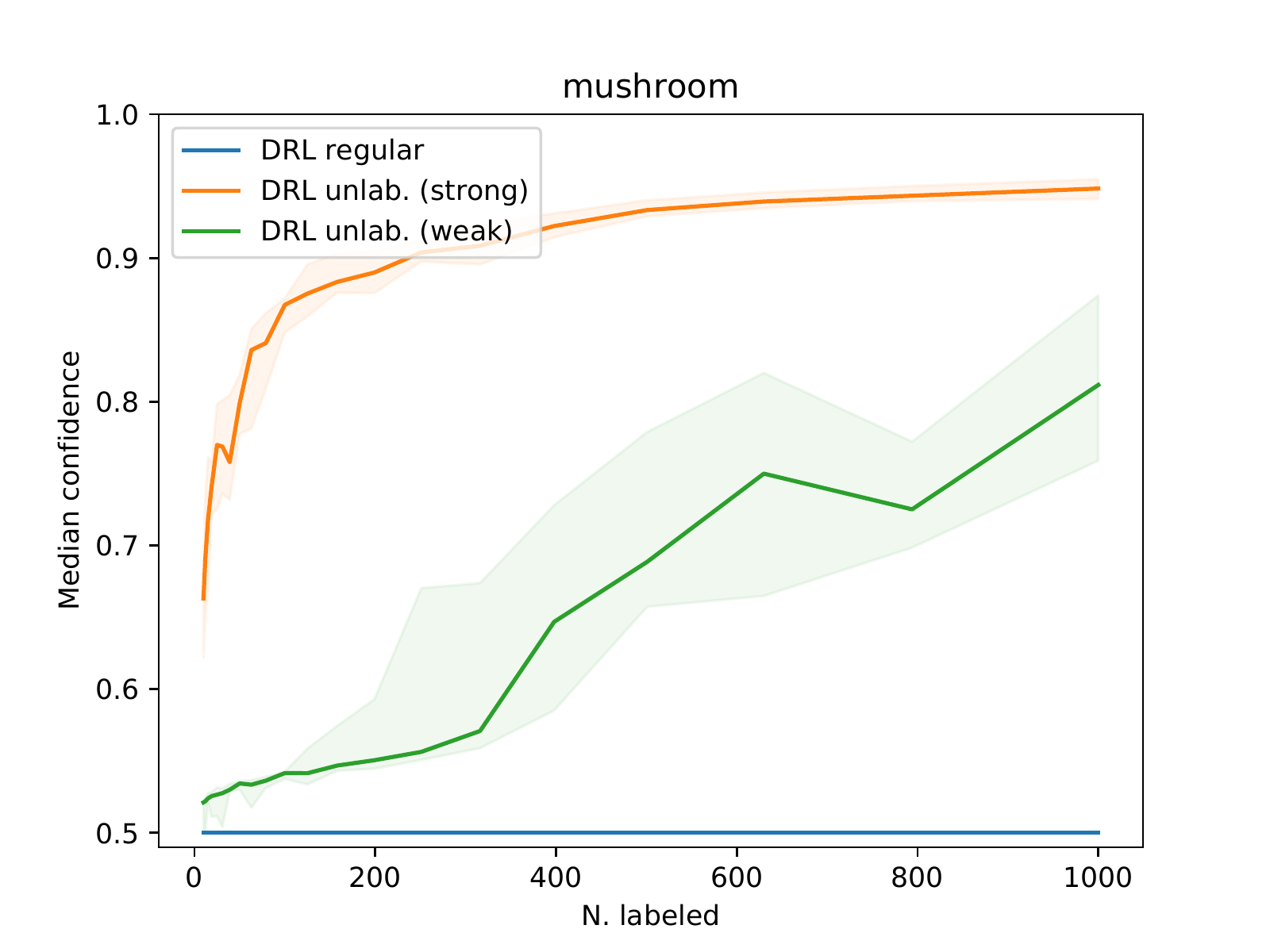}
\caption{Mushroom}
\end{subfigure}
\begin{subfigure}[b]{0.49\textwidth}
\includegraphics[width=\textwidth]{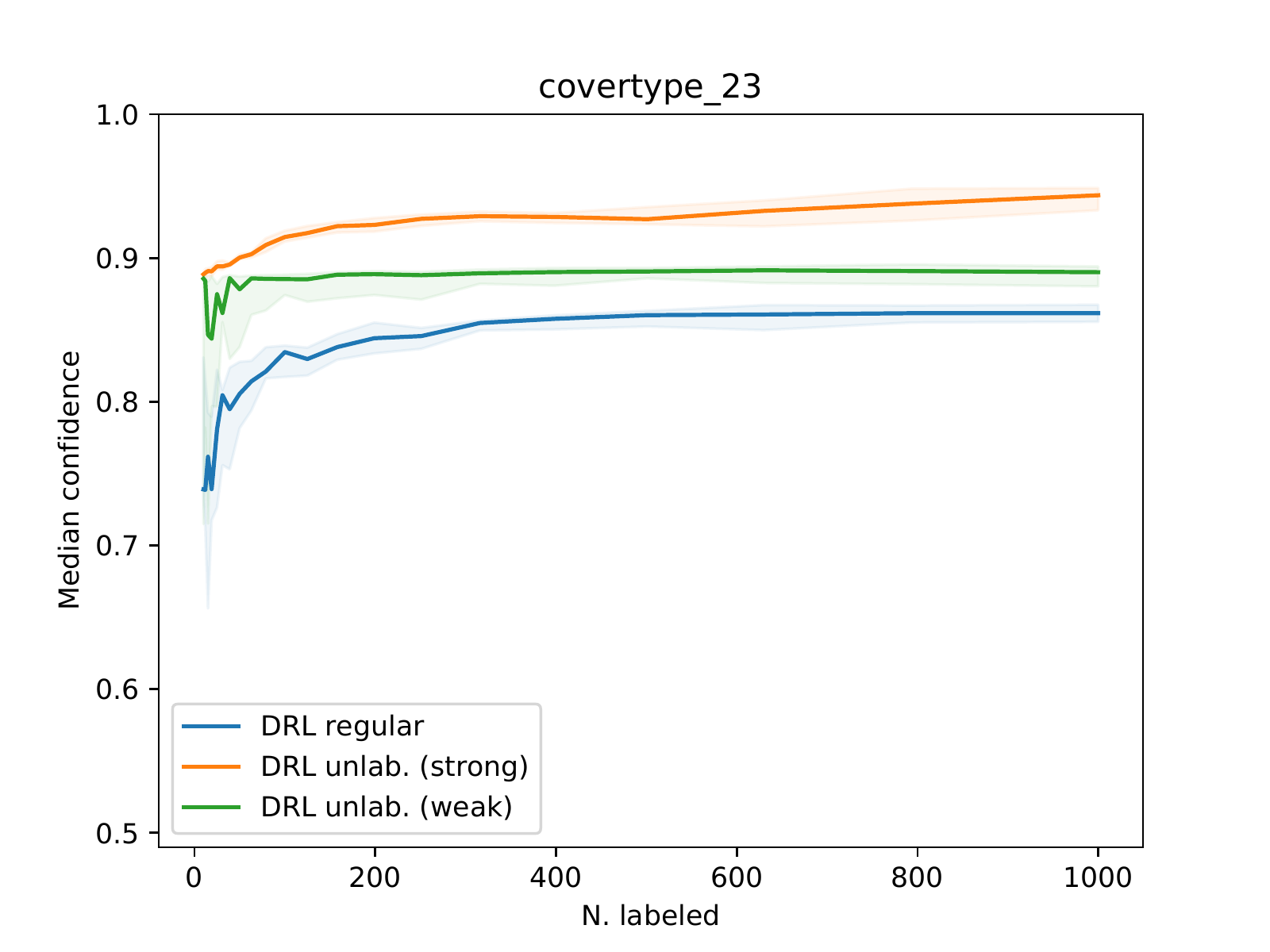}
\caption{Cover type (2-3)}
\end{subfigure}
\begin{subfigure}[b]{0.49\textwidth}
\includegraphics[width=\textwidth]{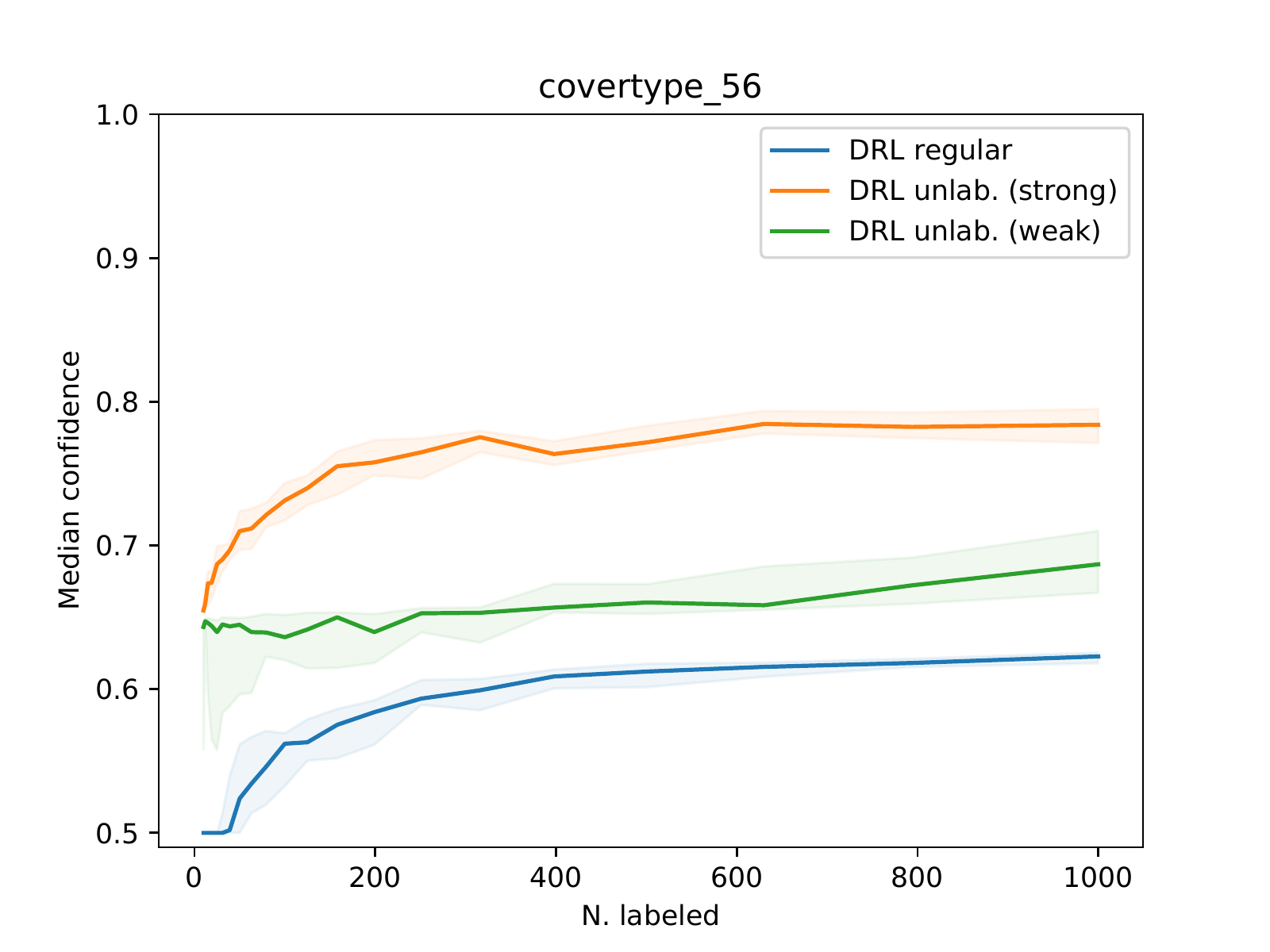}
\caption{Cover type (5-6)}
\end{subfigure}
\caption{Median confidence vs. number of labeled data, setting $\varepsilon$ to include the true test distribution. The regular DRL predictor often has confidence close to $0.5$ in settings where both DRL methods using unlabeled data yield non-trivial predictors.}
\label{fig:true-radius-conf}
\end{figure}

\subsection{Discussion}
The overwhelming size of an adversary's decision set is a weakness of Wasserstein DRL that prevents a reasonable tradeoff between robustness and confidence of the learned predictor. To circumvent this problem, we use unlabeled data to further constrain the decision set. Empirically, the proposed DRL problem produces non-trivial predictors having non-vacuous performance guarantees in cases where traditional Wasserstein DRL fails.

One topic we have not addressed is computational complexity. The proposed DRL is computed via stochastic gradient descent. Each gradient computation scales linearly in the number of labeled examples $N_l$ and this scaling might prohibit application to large labeled datasets. The key bottleneck is computing membership in the sets $V^{ik}$ in \eqref{eq:voronoi-sets}, which relies on a maximization over labeled examples. This computation might be a fruitful target for performance improvement, possibly via parallelization or by leveraging the fact that the cost function $c$ is a power of a metric.

\section{Application: Distributionally-robust active learning}
\label{sec:active}

Key to the learning algorithm of Section \ref{sec:optimization-by-sgd} is a mechanism for optimizing an objective over the intersection of a Wasserstein ball $\ball_{\varepsilon}(\hat{\Ppr}_l)$ with the set of distributions $\Uset(\Ppr_{\Xspace}, \overline{\pvec}_{\Yspace}, \underline{\pvec}_{\Yspace})$ that have prescribed marginals in $\Xspace$ and $\Yspace$. Learning a classifier is just one possible application of this mechanism, however. In this section, we demonstrate another application, to active learning.

\subsection{Model change heuristics}
\label{sec:active-a-class-of-heuristic}

Given a set $\hat{\Zspace}_l$ of labeled data and a set $\hat{\Xspace}_u$ of unlabeled data, an active learner attempts to choose the most beneficial example from $\hat{\Xspace}_u$ for which to acquire a label. The goal of the active learner is to reduce the out-of-sample error of the predictor trained on $\hat{\Zspace}_l$ as rapidly as possible. Many active learning methods assign a score to each unlabeled example, indicating its predicted impact on the learned classifier if we choose to acquire its label. This score might represent various properties, such as model uncertainty, expected error reduction, or expected model change. In the current work we focus on {\bf model change} criteria \citep{settles2008multiple,freytag2014selecting,cai2017active}, which are popular and often effective in practice \citep{yang2018benchmark}.

In model change criteria, we define an {\bf impact function} $f : \Xspace \times \Yspace \rightarrow \reals$, which is large when acquiring the label $\yvec$ for point $\xvec \in \hat{\Xspace}_u$ leads to a large change in the model parameters. Most often this is a norm of the parameter gradient \citep{yang2018benchmark},
$$ f(\xvec, \yvec) = \|\nabla_{\theta} \ell(h_{\theta}(\xvec), \yvec)\|, $$
for $\| \cdot \|$ a norm and $h_{\theta}$ the hypothesis trained on $\hat{\Zspace}_l$.
The active learning heuristic selects $\xvec_{\ast} \in \hat{\Xspace}_u$ that maximizes an estimate of the anticipated impact across possible labels at point $\xvec$. This might be the conditional expectation according to the model distribution at $\xvec$, the minimum over labels, or the maximum over labels:
\begin{itemize}
    \item {\bf Expected impact}: $\xvec_{\ast} = \argmax_{\xvec \in \hat{\Xspace}_u} \sum_{k=1}^{N_{\Yspace}} h_{\theta}(\xvec)_k f(\xvec, \yvec^k)$, with $h_{\theta}$ trained on $\hat{\Zspace}_l$.
    \item {\bf Optimistic}: $\xvec_{\ast} = \argmax_{\xvec \in \hat{\Xspace}_u} \max_{\yvec \in \Yspace} f(\xvec, \yvec)$.
    \item {\bf Conservative}: $\xvec_{\ast} = \argmax_{\xvec \in \hat{\Xspace}_u} \min_{\yvec \in \Yspace} f(\xvec, \yvec)$.
\end{itemize}

A potential problem with the expected model change criterion, which it shares with a number of other standard heuristics \citep{yang2018benchmark}, is that it relies on the current hypothesis $h_{\theta}$ when predicting the impact of choosing a new point $\xvec \in \hat{\Xspace}_u$ to label. Specifically, $h_{\theta}(\xvec)$ is used in place of the conditional distribution over labels at the point $\xvec$. This is prone to error when the hypothesis is far from the true conditional distribution, incorrectly weighting the impact of obtaining labels at the points where the hypothesis is in error.

The ``optimistic'' and ``conservative'' estimates above are simple attempts to eliminate the hypothesis $h_{\theta}$ from the estimated impact. Notably, these ignore the labeled data $\Zspace_{\ell}$ entirely.

\subsection{A distributionally-robust approach}
\label{sec:active-distributionally-robust}

The machinery presented in Section \ref{sec:drl-with-unlabeled-data} provides an alternative way to eliminate the hypothesis $h_{\theta}$ from our estimate of the impact of labeling point $\xvec \in \Xspace_u$. We can formulate a distributionally-robust estimate of the impact, which computes a lower bound on the expected impact with respect to an entire set of plausible data distributions, rather than just the model distribution. This lower bound can be closer to the true expected impact (under $\Ppr$) than the na\"ive conservative estimate, as our set of plausible distributions need not include those that are unreasonably far from the training set.

More precisely, we can formulate the problem of choosing the next sample to label as
\begin{equation}
\label{eq:distributionally-robust-active-learning}
\maximize_{\xvec_{\ast} \in \hat{\Xspace}_u} \inf_{\mu \in \Pset} \frac{\expect^{\mu} \delta_{\xvec_{\ast}}(\Xrv) f(\Xrv, \Yrv)}{\expect^{\mu} \delta_{\xvec_{\ast}}(\Xrv)},
\end{equation}
with $\Pset = \ball_{\varepsilon}(\hat{\Ppr}_l) \cap \Uset(\Ppr_{\Xspace}, \overline{\pvec}_{\Yspace}, \underline{\pvec}_{\Yspace})$ as in Section \ref{sec:drl-with-unlabeled-data}, 
the intersection of a Wasserstein ball centered at the labeled data and the set of distributions having the prescribed marginals. Note that $\expect^{\mu} \delta_{\xvec_{\ast}}(\Xrv) = \expect^{\Ppr_{\Xspace}} \delta_{\xvec_{\ast}}(\Xrv)$ for all $\mu \in \Uset(\Ppr_{\Xspace}, \overline{\pvec}_{\Yspace}, \underline{\pvec}_{\Yspace})$, meaning that the objective in \eqref{eq:distributionally-robust-active-learning} is in fact linear in $\mu$. In practice, given the unlabeled data $\hat{\Xspace}_u$, we can approximate this term by density estimation. We will use the notation $\hat{\varphi}(\xvec_{\ast}) \approx \expect^{\Ppr_{\Xspace}} \delta_{\xvec_{\ast}}(\Xrv)$ for this approximation.

The inner problem in \eqref{eq:distributionally-robust-active-learning} estimates the impact of labeling the point $\xvec_{\ast} \in \Xspace_u$. Just as in the DRL problem formulated in Section \ref{sec:drl-with-unlabeled-data}, this is the optimization of an objective with respect to a probability measure constrained to the feasible set $\Pset = \ball_{\varepsilon}(\hat{\Ppr}_l) \cap \Uset(\Ppr_{\Xspace}, \overline{\pvec}_{\Yspace}, \underline{\pvec}_{\Yspace})$.\footnote{Note that the objective here, $-\mathbf{1}_{\xvec_{\ast}}(\xvec) f(\xvec, \yvec)$, is lower semicontinuous in $\xvec$, whereas Theorem \ref{thm:strong-duality} requires upper semicontinuity of $\ell(h_{\theta}(\xvec), \yvec)$ in $\xvec$. We nevertheless use the duality proved in Theorem \ref{thm:strong-duality}, as one can approximate our lower semicontinuous objective here arbitrarily well by a combination of continuous bump functions centered at $(\xvec_{\ast}, \yvec^k)$, for all $k$, and the strong duality holds for any such approximation.} Just as in Section \ref{sec:drl-with-unlabeled-data}, we can solve this via its dual,
\begin{equation}
\label{eq:active-dual}
g(\xvec_{\ast}) =
\left\{
\begin{array}{rl}
-\inf_{\alpha, \beta, \underline{\lambda}, \overline{\lambda}} & \alpha \varepsilon + \frac{1}{N_l} 
\sum_{i=1}^{N_l} \beta^i + \sum_{k=1}^{N_{\Yspace}} \left(\overline{\lambda}^k \overline{\pvec}_{\Yspace}^k - \underline{\lambda}^k \underline{\pvec}_{\Yspace}^k \right) + \expect^{\Ppr_{\Xspace}} \Psi(\Xrv;\xvec_{\ast}, \alpha, \beta, \overline{\lambda}, \underline{\lambda}) \\
\text{s.t.}
& \alpha, \underline{\lambda}^k, \overline{\lambda}^k \geq 0, \quad \forall k \in \{1, \dots, N_{\Yspace}\}
\end{array}
\right.
\end{equation}
with
\begin{equation}
\Psi(\xvec;\xvec_{\ast}, \alpha, \beta, \overline{\lambda}, \underline{\lambda}) = \max_{\substack{k\in\{1,\dots,N_{\Yspace}\},\\ i\in\{1,\dots,N_l\}}} -\frac{\mathbf{1}_{\xvec_{\ast}}(\xvec) f(\xvec, \yvec^k)}{\hat{\varphi}(\xvec_{\ast})} - \alpha c((\xvec, \yvec^k), \zvec_{\ell}^i) - \beta^i - (\overline{\lambda}^k - \underline{\lambda}^k)
\end{equation}
and $f$ the impact function from Section \ref{sec:active-a-class-of-heuristic}.

The infimum in \eqref{eq:active-dual} corresponds exactly to \eqref{eq:worst-case-dual} in Section \ref{sec:drl-problem-formulation} and we can likewise solve \eqref{eq:active-dual} via SGD, as shown in Algorithm \ref{alg:active-sgd}, with $\Psi(\cdot;\xvec_{\ast}, \alpha, \beta, \overline{\lambda}, \underline{\lambda})$ here corresponding to $\Phi(\cdot;\theta,\alpha,\beta,\overline{\lambda},\underline{\lambda})$ from Section \ref{sec:optimization-by-sgd}, replacing $\ell(h_{\theta}, (\xvec, \yvec))$ with $\frac{-\mathbf{1}_{\xvec_{\ast}}(\xvec) f(\xvec, \yvec)}{\hat{\varphi}(\xvec)}$. The relevant gradient computations for $\Psi(\xvec;\xvec_{\ast}, \alpha, \beta, \overline{\lambda}, \underline{\lambda})$ are identical to those for $\Phi(\xvec;\theta,\alpha,\beta,\overline{\lambda},\underline{\lambda})$ and are included in Appendix \ref{sec:appendix-dual-gradients}.

\begin{algorithm}[h]
\caption{SGD for distributionally robust active learning}
\label{alg:active-sgd}
\begin{algorithmic}
\STATE {\bf Given}: $\theta \in \Theta$, $\varepsilon \geq 0$, $\overline{\pvec}_{\Yspace}, \underline{\pvec}_{\Yspace} \in [0, 1]^{N_{\Yspace}}$, $\theta_0 \in \Theta$, step size $\gamma > 0$, batch size $N_b$.
\FOR{$\xvec_{\ast} \in \hat{\Xspace}_u$}
\STATE $\alpha, \beta, \overline{\lambda}, \underline{\lambda} \gets \mathbf{0}$.
\WHILE{not converged}
\STATE $\xvec^1, \dots, \xvec^{N_b} \sim \Ppr_{\Xspace}$.
\STATE $\alpha \gets \max\left(0, \alpha - \gamma \left[\varepsilon + \frac{1}{N_b} \sum_{j=1}^{N_b} \nabla_{\alpha} \Psi(\xvec^j;\xvec_{\ast}, \alpha, \beta, \overline{\lambda}, \underline{\lambda})\right]\right)$.
\STATE $\beta \gets \beta - \gamma \left[\frac{1}{N_l} + \frac{1}{N_b} \sum_{j=1}^{N_b} \nabla_{\beta} \psi(\xvec^j;\xvec_{\ast}, \alpha, \beta, \overline{\lambda}, \underline{\lambda})\right]$.
\STATE $\overline{\lambda} \gets \max\left(\mathbf{0}, \overline{\lambda} - \gamma \left[\overline{\pvec}_{\Yspace} + \frac{1}{N_b} \sum_{j=1}^{N_b} \nabla_{\overline{\lambda}} \psi(\xvec^j;\xvec_{\ast}, \alpha, \beta, \overline{\lambda}, \underline{\lambda})\right]\right)$.
\STATE $\underline{\lambda} \gets \max\left(\mathbf{0}, \underline{\lambda} - \gamma \left[-\underline{\pvec}_{\Yspace} + \frac{1}{N_b} \sum_{j=1}^{N_b} \nabla_{\underline{\lambda}} \psi(\xvec^j;\xvec_{\ast}, \alpha, \beta, \overline{\lambda}, \underline{\lambda})\right]\right)$.
\ENDWHILE
\STATE $\hat{g}[\xvec_{\ast}] \gets  -\left(\alpha \varepsilon + \frac{1}{N_l} \sum_{i=1}^{N_l} \beta^i + \sum_{k=1}^{N_{\Yspace}} \bigl(\overline{\lambda}^k \overline{\pvec}_{\Yspace}^k - \underline{\lambda}^k \underline{\pvec}_{\Yspace}^k \bigr) + \frac{1}{N_u} \sum_{j=1}^{N_u} \Psi(\xvec_u^j;\xvec_{\ast}, \alpha, \beta, \overline{\lambda}, \underline{\lambda})\right)$.
\ENDFOR
\STATE Choose $\argmax_{\xvec_{\ast} \in \hat{\Xspace}_u} \hat{g}[\xvec_{\ast}]$.
\end{algorithmic}
\end{algorithm}

\subsection{Empirical Results}
\label{sec:active-learning-empirical}

We evaluate active learning performance on the set of $14$ binary classification datasets used in Section \ref{sec:drl-with-unlabeled-data}. Given a set of labeled examples, a linear classifier is trained by $\ell^2$-regularized logistic regression, with the weight on the regularizer fixed a priori. Given this classifier and a set of unlabeled examples, the active learning algorithm selects the next example for which to acquire a label. The process is iterated, beginning with $20$ examples chosen uniformly at random (the same initial examples for all active learning methods, but different initial examples between trials), and terminating after $100$ labeled examples have been acquired.

After training the classifier at each step, we evaluate the error (the likelihood) on the combination of labeled and unlabeled data, to provide a score that is comparable between steps. This score has previously been proposed as a proxy for out-of-sample error in both the semi-supervised \citep{grandvalet2005semi} and active \citep{guo2008discriminative} learning settings.

We compare the proposed distributionally-robust active learning method to both random sampling and the existing model-change heuristics (described in Section \ref{sec:active-a-class-of-heuristic}). Specifically, we test five methods:
\begin{enumerate}
\item {\bf Random}: We choose the next example uniformly at random.
\item {\bf EMC}: We choose the example that maximizes the expected norm of the parameter gradient, under the hypothesis distribution \citep{settles2008multiple}.
\item {\bf Min. MC}: We choose the example that maximizes the minimum (over possible labels) norm of the parameter gradient (``conservative'' in Section \ref{sec:active-a-class-of-heuristic}).
\item {\bf Max. MC}: We choose the example that maximizes the maximum (over possible labels) norm of the parameter gradient (``optimistic'' in Section \ref{sec:active-a-class-of-heuristic}).
\item {\bf DR (strong)}: We choose the example that maximizes the proposed distributionally-robust lower bound on the expected norm of the parameter gradient (Section \ref{sec:active-distributionally-robust}). We use a strong prior on the label probabilities, being the true label probabilities.
\item {\bf DR (weak)}: We choose the example that maximizes the proposed distributionally-robust lower bound on the expected norm of the parameter gradient (Section \ref{sec:active-distributionally-robust}). We use a weak prior on the label probabilities, being $95\%$ confidence bounds estimated from the labeled data.
\end{enumerate}

Table \ref{tbl:active-aulc} shows the area under the likelihood curve (from samples $20$ through $100$) for each method and dataset, using the median likelihood over trials (i.e. initial training samples). We make several observations:
\begin{enumerate}
\item The existing model change heuristics yield inconsistent performance, with EMC, Min. MC, and Max. MC underperforming random sampling on $10$ of $14$ datasets.
\item The proposed distributionally-robust heuristics only underperformed random sampling on $5$ of $14$ datasets.
\item The proposed distributionally-robust heuristics outperform the other model change-based heuristics (EMC, Min. MC, and Max. MC) on $8$ of $14$ datasets.
\item The distributionally-robust heuristics performed similarly to one another, using a strong prior and a weak prior.
\end{enumerate}

\begin{table}
  \caption{Active learning: $100 \times$ area under the likelihood curve (median over trials).}
  \label{tbl:active-aulc}
  \centering
  \begin{tabular}{rcccccc}
    \toprule
    Dataset & Random & EMC & Min MC & Max MC & DR (strong) & DR (weak) \\
    \midrule
    Abalone & $68.7$ & $66.6$ & $66.0$ & $64.2$ & {\color{blue}$\mathbf{69.8}$} & $69.8$ \\
    Bank & $92.6$ & $94.5$ & $94.5$ & $86.0$ & $95.3$ & {\color{blue}$\mathbf{95.4}$} \\
    Cover (2/3) & {\color{blue}$\mathbf{79.9}$} & $78.0$ & $79.4$ & $74.0$ & $78.4$ & $78.4$ \\
    Cover (5/6) & $72.7$ & $72.3$ & $71.4$ & $66.3$ & {\color{blue}$\mathbf{73.0}$} & $72.9$ \\
    Isolet & $60.8$ & $58.2$ & $60.0$ & $54.8$ & $63.3$ & {\color{blue}$\mathbf{63.8}$} \\
    Letter (C/E) & $71.6$ & $67.6$ & $67.6$ & $64.4$ & $73.2$ & {\color{blue}$\mathbf{73.6}$} \\
    Letter (U/V) & $77.8$ & $73.2$ & $73.1$ & $67.1$ & {\color{blue}$\mathbf{82.3}$} & $82.1$  \\
    Magic & {\color{blue}$\mathbf{52.2}$} & $46.0$ & $46.3$ & $46.0$ & $43.1$ & $43.1$ \\
    Mushroom & $86.3$ & $91.0$ & $90.6$ & $77.9$ & {\color{blue}$\mathbf{92.2}$} & {\color{blue}$\mathbf{92.2}$} \\
    Pulsar & $76.9$ & {\color{blue}$\mathbf{79.7}$} & $79.4$ & $75.0$ & $79.6$ & $79.6$ \\
    Spam & $68.4$ & $68.0$ & $67.3$ & $63.1$ & {\color{blue}$\mathbf{69.9}$} & {\color{blue}$\mathbf{69.9}$} \\
    Thyroid & $79.1$ & {\color{blue}$\mathbf{80.7}$} & $80.6$ & $78.8$ & $78.4$ & $78.6$ \\
    Wine & {\color{blue}$\mathbf{56.7}$} & $53.0$ & $51.9$ & $51.3$ & $50.4$ & $50.5$ \\
    Yeast & {\color{blue}$\mathbf{54.7}$} & $47.8$ & $47.3$ & $53.2$ & $46.6$ & $46.8$ \\
    \bottomrule
  \end{tabular}
\end{table}

\subsection{Discussion}

The mechanism used to solve the Wasserstein DRL problem with unlabeled data (Section \ref{sec:drl-with-unlabeled-data}), which relies on the duality result in \ref{sec:drl-problem-formulation}, can have broader applications. We have demonstrated an application to the problem of active learning that yields a distributionally-robust model change heuristic that empirically often outperforms the existing model change heuristics.

Computational complexity is a potential impediment to deploying the proposed active learning method. For each unlabeled example that might be selected for labeling, the method requires solving a distributionally robust problem (Equation \ref{eq:active-dual}) by an iterative method (such as in Algorithm \ref{alg:active-sgd}). This in our experiments required on the order of $50000$ iterations per example considered, with the complexity of each iteration scaling linearly with the number of training example $N_l$ (identically to the DRL method in Section \ref{sec:drl-with-unlabeled-data}). As in Section \ref{sec:drl-with-unlabeled-data}, this complexity might be a productive target for future work.

\section{Conclusion}

We have explored an alternative to Wasserstein distributionally robust learning that incorporates unlabeled data to restrict the adversary's decision set. In particular, we proposed to intersect the standard Wasserstein ball constraint with the set of probability measures having specified marginals in both feature space and label space. This latter constraint adds some complexity to the derivation of a tractable algorithm (Section \ref{sec:drl-problem-formulation}), which follows the standard DRL framework of dualizing the problem but requires some care as the dual is now infinite-dimensional, due to specifying the feature-space marginal in the primal. We prove a strong duality theorem (Theorem \ref{thm:strong-duality}) that guarantees we can solve our proposed Wasserstein DRL problem via a dual formulation. This dual problem we can treat as the minimization of an expectation with respect to the pre-specified feature-space marginal, which is amenable to stochastic gradient methods (Algorithm \ref{alg:sgd}). Critically, such methods rely only on sampling unlabeled data from the feature-space marginal of the data distribution. Therefore the resulting SGD algorithm is tractable whenever we have access to plentiful unlabeled data, which is frequently the case in machine learning settings.

The motivation for exploring this alternative approach is an empirical observation that in standard Wasserstein DRL the adversary's decision set grows overly large very quickly as the radius of robustness $\varepsilon$ is increased. As a result, choosing any radius sufficiently large for the adversary's decision set to contain the true data distribution will yield a trivial classifier (predicting equal probability for every class). Moreover, the generalization performance guarantee implied by distributionally robust methods only holds when the decision set contains the true data distribution. This performance guarantee is one major motivation for using DRL methods and here we have shown that there is a gap between theory and practice.

Restricting the adversary's decision set even more than we have done here might be a profitable avenue for further research. There are likely other ways to incorporate side information into the problem that can be applicable in a variety of practical settings. Moreover, the practical application of the method proposed here is currently somewhat constrained by the computational complexity of the SGD iterations, which scale linearly in the number of labeled examples used. It is possible there are significant speedups to be obtained via parallelization and further assumptions about the structure of the transport cost $c$.

Finally, we have proved the strong duality theorem (Theorem \ref{thm:strong-duality}) only for feature spaces $\Xspace$ that are compact. This is a reasonable assumption from a pragmatic standpoint, as no data distribution in practice will have unbounded support, but we also expect an equivalent theorem to hold for non-compact $\Xspace$. Proof techniques from duality theorems for optimal transport, as in \citep[Theorem 1.3]{villani2003topics} and \citep[Theorem 4.6.14]{rachev1998mass}, and distributional model risk assessment \citep[Theorem 1]{blanchet2019quantifying} might be applicable here.

\subsubsection*{Acknowledgments}

The authors acknowledge the generous support of Army Research Office grant W911NF1710068, Air Force Office of Scientific Research award FA9550-19-1-031, of National Science Foundation grant IIS-1838071, from the MIT-IBM Watson AI Laboratory, from the Toyota-CSAIL Joint Research Center, from the QCRI–CSAIL Computer Science Research Program, from the MIT CSAIL Systems that Learn initiative, and from a gift from Adobe Systems. Any opinions, findings, and conclusions or recommendations expressed in this material are those of the authors and do not necessarily reflect the views of these organizations. The authors also thank Nestor Guillen for helping them to understand duality over spaces of measures.

\newpage
\appendix

\section{Proof of strong duality}
\label{sec:appendix-strong-duality}

To show strong duality, we will make use of a fundamental convex analysis result: the Fenchel duality theorem \citep[Theorem 4.4.3]{borwein2005techniques}.

\begin{theorem}[Fenchel duality]
\label{thm:fenchel_duality}
Let $\Xi, \Gamma$ be Banach spaces, with convex functions $\gamma: \Xi \to \mathbb{R} \cup \{+ \infty \}$ and $\chi: \Gamma \to \mathbb{R} \cup \{+ \infty \}$, and a continuous linear map $A: \Xi \to \Gamma$. Then
\begin{align*}
    \inf_{\xi \in \Xi} \gamma(\xi) + \chi(A\xi) \geq \sup_{u \in \Gamma^*} -\gamma^*(A^*u) - \chi^*(-u).
\end{align*}
If $\gamma$ and $\chi$ are lower semicontinuous and $A \text{ dom } \gamma \cap \text{ cont } \chi \neq \emptyset$, then equality holds above and the supremum on the right-hand side is attained.
\end{theorem}

{
\renewcommand{\thetheorem}{\ref{thm:strong-duality}}
\begin{theorem}[Strong duality]
Let $\Xspace$ be a compact Polish space and $\Yspace = \{\yvec^k\}_{k=1}^{N_{\Yspace}}$ any finite set. Let $\Ppr_{\Xspace}$ be a probability measure over $\Xspace$ and $\hat{\Ppr}_l = \frac{1}{N_l} \sum_{i=1}^{N_l} \delta_{\zvec_{\ell}^i}$ an empirical probability measure over $\Zspace = \Xspace \times \Yspace$, and define intervals $[\underline{\pvec}_{\Yspace}^k, \overline{\pvec}_{\Yspace}^k] \subseteq [0, 1]$, $k \in \{1, \dots, N_{\Yspace}\}$. Let the transportation cost $c : \Zspace \times \Zspace \rightarrow [0, +\infty)$ be nonnegative and upper semicontinuous with $c(\zvec, \zvec^{\prime}) = 0 \Leftrightarrow \zvec = \zvec^{\prime}$. Assume $\ell(h_{\theta}(\cdot), \cdot) : \Zspace \rightarrow \reals$ is upper semicontinuous. Define $f$ as in \eqref{eq:worst-case-primal} and $g$ as in \eqref{eq:worst-case-dual}. If $\Uset(\Ppr_{\Xspace}, \underline{\pvec}_{\Yspace}, \overline{\pvec}_{\Yspace}) \cap \ball_{\varepsilon}(\hat{\Ppr}_l) \neq \emptyset$, then
\begin{equation}
f(\theta) = g(\theta),\quad \forall \theta \in \Theta.
\end{equation}
If $\operatorname{relint} (\Uset(\Ppr_{\Xspace}, \underline{\pvec}_{\Yspace}, \overline{\pvec}_{\Yspace}) \cap \ball_{\varepsilon}(\hat{\Ppr}_l)) \neq \emptyset$, then there exists a minimizer $(\alpha_{\ast}, \beta_{\ast}, \overline{\lambda}_{\ast}, \underline{\lambda}_{\ast}) \in \reals_+ \times \reals^{N_l} \times \reals_+^{N_{\Yspace}} \times \reals_+^{N_{\Yspace}}$ attaining the infimum in \eqref{eq:worst-case-dual}.
\end{theorem}

\begin{proof}
For convenience, in what follows we will write $\ell(h_{\theta}, \zvec) \triangleq \ell(h_{\theta}(\xvec), \yvec)$ for $\zvec = (\xvec, \yvec) \in \Zspace$. We first recall the primal problem \eqref{eq:worst-case-primal-transport} and the dual problem \eqref{eq:worst-case-dual} from the main text:

\begin{equation}
\label{eq:primal-constraints}
f(\theta)=
\left\{
\begin{array}{rll}
\sup_{\pi} & \int_{\Zspace \times \Zspace} \ell(h_{\theta}, \zvec) \,\dx\pi(\zvec, \zvec^{\prime}) \\
\text{s.t.}
 & \int_{\Zspace \times \Zspace} c(\zvec, \zvec^{\prime}) \,\dx\pi(\zvec, \zvec^{\prime}) \leq \varepsilon \\
& \int_{\Zspace \times \Zspace} \delta_{\zvec_{\ell}^i}(\zvec^{\prime}) d\pi(\zvec, \zvec^{\prime}) = \frac{1}{N_l}
&\forall i \in \{1, \dots, N_l\}, \\
& \pi((A \times \Yspace) \times \Zspace) = \Ppr_{\Xspace}(A) & \forall A \in \Bset(\Xspace), \\
& \int_{(\Xspace \times \Yspace) \times \Zspace} \delta_{\yvec^k}(\yvec) \,\dx\pi((\xvec, \yvec), \zvec^{\prime}) \leq \overline{\pvec}_{\Yspace}^k
& \forall k \in \{1, \dots, N_{\Yspace}\}, \\
& \int_{(\Xspace \times \Yspace) \times \Zspace} \delta_{\yvec^k}(\yvec) \,\dx\pi((\xvec, \yvec), \zvec^{\prime}) \geq \underline{\pvec}_{\Yspace}^k 
& \forall k \in \{1, \dots, N_{\Yspace}\}, \\
& \pi(A) \geq 0
& \forall A \in \Bset(\Zspace \times \Zspace).
\end{array}
\right.
\end{equation}
More succinctly, we may write the primal problem as:
\begin{equation}
\label{eq:primal-lp}
f(\theta) = \sup_{\pi \in \Phi} F_{\theta}(\pi).
\end{equation}
with $F_{\theta}$ the objective above and $\Phi$ the feasible set.

The Lagrangian dual to \eqref{eq:primal-lp} is
\begin{equation}
\label{eq:dual-lp}
\inf_{(\alpha, \beta, \phi, \overline{\lambda}, \underline{\lambda}) \in \Lambda_{\theta}} G(\alpha, \beta, \phi, \overline{\lambda}, \underline{\lambda}),
\end{equation}
with
\begin{equation}
\label{eq:dual-lp-components}
\begin{aligned}
G(\alpha, \beta, \phi, \overline{\lambda}, \underline{\lambda}) &= \alpha \varepsilon + \frac{1}{N_l} \sum_{i=1}^{N_l} \beta^i + \int_{\Xspace} \phi(\xvec) \,\dx\Ppr_{\Xspace}(\xvec) + \sum_{k=1}^{N_{\Yspace}} \left(\overline{\lambda}^k \overline{\pvec}_{\Yspace}^k - \underline{\lambda}^k \underline{\pvec}_{\Yspace}^k\right), \\
\Lambda_{\theta} &= \Bigl\{(\alpha, \beta, \phi, \overline{\lambda}, \underline{\lambda}) \in \reals \times \reals^{N_l} \times C_b(\Xspace) \times \reals^{N_{\Yspace}} \times \reals^{N_{\Yspace}} : \\
&\quad\quad\quad \phi(\xvec) \geq \ell(h_{\theta}, (\xvec, \yvec^k)) - \alpha c((\xvec, \yvec^k), \zvec_{\ell}^i) - \beta^i - (\overline{\lambda}^k - \underline{\lambda}^k), \\
&\quad\quad\quad \alpha, \overline{\lambda}^k, \underline{\lambda}^k \geq 0, \\
&\quad\quad\quad \forall \xvec \in \Xspace, k \in \{1, \dots, N_{\Yspace}\}, i \in \{1, \dots, N_l\}\Bigr\}.
\end{aligned}
\end{equation}
We make two claims:
\begin{enumerate}
\item $\sup_{\pi \in \Phi} F_{\theta}(\pi) = \inf_{(\alpha, \beta, \phi, \overline{\lambda}, \underline{\lambda}) \in \Lambda_{\theta}} G(\alpha, \beta, \phi, \overline{\lambda}, \underline{\lambda})$.
\item There exists a dual optimizer $(\alpha, \beta, \phi_{\theta, \alpha, \beta, \overline{\lambda}, \underline{\lambda}}, \overline{\lambda}, \underline{\lambda}) \in \reals \times \reals^{N_l} \times L^0(\Xspace, \Ppr_{\Xspace}) \times \reals^{N_{\Yspace}} \times \reals^{N_{\Yspace}}$ with
\begin{equation}
\label{eq:optimal-phi}
\phi_{\theta, \alpha, \beta, \overline{\lambda}, \underline{\lambda}}(\xvec) = \max_{k \in \{1, \dots, N_{\Yspace}\}} \ell(h_{\theta}, (\xvec, \yvec^k)) - \min_{i \in \{1, \dots, N_l\}} \left(\alpha c((\xvec, \yvec^k), \zvec_{\ell}^i) + \beta^i\right) - (\overline{\lambda}^k - \underline{\lambda}^k).
\end{equation}
If $\ell(h_{\theta}, \cdot)$ and $c(\cdot, \zvec)$ are continuous, then $(\alpha, \beta, \phi_{\theta, \alpha, \beta, \overline{\lambda}, \underline{\lambda}}, \overline{\lambda}, \underline{\lambda}) \in \Lambda_{\theta}$.
\end{enumerate}

We start with the first claim, and apply Fenchel duality to the proposed dual problem \eqref{eq:dual-lp-components} to show the desired equality. This strategy of dualizing the proposed dual mirrors the arguments of \cite{villani2003topics} in proving strong duality for regular optimal transport. On a compact Polish space, the dual of the space of finite, signed measures is larger than the space of continuous, bounded functions, necessitating this approach. On the other hand, the Riesz representation theorem allows us to move from the proposed dual to the primal in a rigorous manner. It tells us that the space of finite, signed measures is isomorphic to the dual of the space of continuous, bounded functions (on a compact Polish space). 

We will actually show the first claim holds for a slight generalization of the dual problem \eqref{eq:dual-lp}. Consider the spaces $\Xi = \mathbb{R}\times C_b(\Zspace) \times C_b(\Xspace) \times C_b(\Yspace) \times C_b(\Yspace)$ and $\Gamma = C_b(\Zspace\times \Zspace)$, which are Banach spaces. The norm on $\Xi$ is given  by $\|(\alpha, \beta, \phi, \overline{\lambda}, \underline{\lambda})\| \triangleq |\alpha| + \|\beta\|_{\infty} + \|\phi\|_{\infty} + \|\overline{\lambda}\|_{\infty} + \|\underline{\lambda}\|_{\infty}$, while the norm on $\Gamma$ is $\|\cdot\|_{\infty}$.

Let $\nu_{\Zspace} \in \Mset(\Zspace), \nu_{\Xspace} \in \Mset(\Xspace), \overline{\nu}_{\Yspace} \in \Mset(\Yspace), \underline{\nu}_{\Yspace} \in \Mset(\Yspace)$. The dual problem we will rewrite as
\begin{equation} 
\label{eq:dual-fenchelized}
\inf_{\xi \in \Xi} \tilde{G}(\xi) + \chi(A \xi),
\end{equation}
with
\begin{equation}
\label{eq:dual-fenchelized-components}
\begin{aligned}
\tilde{G} &: \xi = (\alpha, \beta, \phi, \overline{\lambda}, \underline{\lambda}) \in \Xi \mapsto \bigbrace{\alpha \varepsilon + \langle \nu_{\Zspace}, \beta \rangle + \langle \nu_{\Xspace}, \phi \rangle + \langle \overline{\nu}_{\Yspace}, \overline{\lambda} \rangle - \langle \underline{\nu}_{\Yspace}, \underline{\lambda} \rangle & \alpha, \overline{\lambda}, \underline{\lambda} \geq 0 \\ +\infty & \otherwise}, \\
\chi &: u \in \Gamma \mapsto \bigbrace{0 & u(\zvec, \zvec^{\prime}) \geq \ell(h_{\theta}, \zvec)\ \forall \zvec, \zvec^{\prime} \in \Zspace \\
+\infty & \otherwise},
\end{aligned}
\end{equation}
and $A : \Xi \rightarrow \Gamma$ a linear operator defined by
\begin{equation}
\label{eq:dual-fenchelized-linear}
(A \xi)(\zvec, \zvec^{\prime}) = \alpha c(\zvec, \zvec^{\prime}) + \beta(\zvec^{\prime}) + \phi(\xvec) + \overline{\lambda}(\yvec) - \underline{\lambda}(\yvec),
\end{equation}
where $\zvec = (\xvec, \yvec)$. Optimization problem \eqref{eq:dual-fenchelized} is identical to the dual problem \eqref{eq:dual-lp} when $\nu_{\Zspace} = \hat{\Ppr}_l$, $\nu_{\Xspace} = \Ppr_{\Xspace}$, $\overline{\nu}_{\Yspace} = \overline{\Ppr}_{\Yspace}$, and $\underline{\nu}_{\Yspace} = \underline{\Ppr}_{\Yspace}$.

$\tilde{G}$ is convex and lower semi-continuous as a function of $(\alpha, \beta, \phi, \overline{\lambda}, \underline{\lambda})$, because it is linear on a closed, convex domain. $A$ is clearly continuous and $\chi$ is convex and lower semi-continuous as the indicator of a closed, convex domain. Note also that $A \dom \tilde{G} \cap \text{cont} \chi$ is nonempty as $\ell(h_{\theta}, \cdot)$ is an upper semi-continuous function on a compact domain and is bounded. In particular, $\ell(h_{\theta}, \cdot) < M$ for some $M \in \reals$, so by choosing $\beta = M$ with $\alpha = \phi = \overline{\lambda} = \underline{\lambda} = 0$ we have that $A(\alpha, \beta, \phi, \overline{\lambda}, \underline{\lambda}) \in \text{cont} \chi$.

As the underlying domain is compact, the topological duals of $\Xi$ and $\Gamma$ are $\Xi' = \mathbb{R}\times \mathcal{M}(\Zspace) \times \mathcal{M}(\Xspace) \times \mathcal{M}(\Yspace) \times \mathcal{M}(\Yspace)$ and $\Gamma' = \mathcal{M}(\Zspace \times \Zspace)$. This duality allows us to define an adjoint for the operator $A$, given by $A^{\ast} : \Gamma^{\prime} \rightarrow \Xi^{\prime}$, with
\begin{align*}
    A^*(\pi) &= \left(\int c \, \dx\pi, \pi_{Z'}, \pi_X, \pi_Y, -\pi_Y \right),
\end{align*}
such that
\begin{equation}
\begin{aligned}
A^*(\pi)(\alpha, \beta, \phi, \overline{\lambda}, \underline{\lambda}) &= \int_{\Zspace \times \Zspace} (\alpha c((\xvec, \yvec),\zvec') + \beta(\zvec') + \phi(\xvec) + \overline{\lambda}(\yvec) - \underline{\lambda}(\yvec))d\pi((\xvec, \yvec), \zvec^{\prime}) \\
&= \int_{\Zspace \times \Zspace} A(\alpha, \beta, \phi, \overline{\lambda}, \underline{\lambda})\,\dx\pi.
\end{aligned}
\end{equation}
Here, $\pi_X$, $\pi_Y$ denote $\Xspace$- and $\Yspace$-marginals of the first marginal of $\pi$, while $\pi_{Z^{\prime}}$ is the second marginal of $\pi$.

We can compute the convex conjugates of $\tilde{G}$ and $\chi$.
\begin{align*}
    \tilde{G}^*(a,\sigma,\tau, \zeta, \omega) &= \sup_{(\alpha, \beta, \phi, \overline{\lambda}, \underline{\lambda}) \in \Xi} \langle (a,\sigma,\tau, \zeta, \omega) , (\alpha, \beta, \phi, \overline{\lambda}, \underline{\lambda}) \rangle - \tilde{G}(\alpha, \beta, \phi, \overline{\lambda}, \underline{\lambda}) \\
    &= \sup_{(\alpha, \beta, \phi, \overline{\lambda}, \underline{\lambda}) \in \Xi} \bigbrace{(a-\varepsilon)\alpha + \int_{\Zspace} \beta \, (\dx\sigma - \dx\nu_{\Zspace}) + \int_{\Xspace} \phi \, (\dx\tau - \dx\nu_{\Xspace}) \\
    \qquad \qquad + \int_{\Yspace} \overline{\lambda} \, (\dx\zeta - \dx\overline{\nu}_{\Yspace}) + \int_{\Yspace} \underline{\lambda} \, (\dx\omega + \dx\underline{\nu}_{\Yspace}) \\
    \qquad \qquad \text{if}\ \alpha, \overline{\lambda}, \underline{\lambda} \geq 0 \\
    -\infty\ \ \text{otherwise}} \\
    &= \begin{cases}
    0 &\text{if } a \leq \varepsilon, \dx\sigma = \dx\nu_{\Zspace}, \dx\tau = \dx\nu_{\Xspace}, \dx\zeta \leq \dx\overline{\nu}_{\Yspace}, \dx\omega \leq -\dx\underline{\nu}_{\Yspace} \\
    +\infty&\text{otherwise}.
    \end{cases}\\
    \chi^*(\pi) &= \sup_{u \in \Gamma} \langle \pi, u \rangle - \chi(u) \\
    & = \sup_{u \geq l(h_\theta(\cdot),\cdot)} \int_{\Zspace \times \Zspace} u \, \dx\pi \\
    &= \begin{cases}
    \int_{\Zspace} l(h_\theta, \zvec) \, d\pi(\zvec) &\text{if } \,\dx\pi \leq 0 \\
    +\infty&\text{otherwise}.
    \end{cases}
\end{align*}
Above we have again used the fact that $l(h_\theta,\cdot)$ is bounded as an upper-semicontinuous function on a compact set. The resulting optimization problem, given by $\sup_{\pi \in \Gamma^{\ast}} -\tilde{G}^{\ast}(A^*\pi) - \xi^{\ast}(-\pi)$, is the primal problem \eqref{eq:primal-lp} with $-\tilde{G}^*(A^*\pi)$ expressing the Wasserstein and marginal constraints, and $-\chi^*(-\pi)$ expressing the objective and the positivity constraints on $\pi$. Strong duality follows from direct application of Theorem \ref{thm:fenchel_duality}, which also gives us existence of a primal maximizer $\pi^*$ to \eqref{eq:primal-lp}.

The second claim states that there exists a dual minimizer $(\alpha, \beta, \phi_{\theta, \alpha, \beta, \overline{\lambda}, \underline{\lambda}}, \overline{\lambda}, \underline{\lambda}) \in \reals \times \reals^{N_l} \times L^0(\Xspace, \Ppr_{\Xspace}) \times \reals^{N_{\Yspace}} \times \reals^{N_{\Yspace}}$ to \eqref{eq:dual-lp-components} with
\begin{equation*}
\phi_{\theta, \alpha, \beta, \overline{\lambda}, \underline{\lambda}}(\xvec) = \max_{k \in \{1, \dots, N_{\Yspace}\}, i \in \{1, \dots, N_l\}} \ell(h_{\theta}, (\xvec, \yvec^k)) - \alpha c((\xvec, \yvec^k), \zvec_{\ell}^i) - \beta^i - (\overline{\lambda}^k - \underline{\lambda}^k).
\end{equation*}
We start by noting that the function $\phi_{\theta, \alpha, \beta, \overline{\lambda}, \underline{\lambda}}$ is a pointwise maximum over a finite collection of functions $\ell(h_{\theta}, \cdot) - \alpha c(\cdot, \zvec_{\ell}^i)$, plus a constant term. If both $\ell(h_{\theta}, (\cdot, \yvec^k))$ and $c((\cdot, \yvec^k), \zvec_{\ell}^i)$, for all $k, i$, are measurable with respect to $\Ppr_{\Xspace}$, then $\phi_{\theta, \alpha, \beta, \overline{\lambda}, \underline{\lambda}} \in L^0(\Xspace, \Ppr_{\Xspace})$ as well.
Moreover, if $\ell(h_{\theta}, \cdot)$ and $c(\cdot, \zvec_{\ell}^i)$ are continuous, then they are bounded (under the assumption that $\Xspace$ is compact) and so $\phi_{\theta, \alpha, \beta, \overline{\lambda}, \underline{\lambda}} \in C_b(\Xspace)$. In that case, for any finite $(\alpha, \beta, \overline{\lambda}, \underline{\lambda})$ satisfying $\alpha, \overline{\lambda}, \underline{\lambda} \geq 0$, the element $(\alpha, \beta, \phi_{\theta, \alpha, \beta, \overline{\lambda}, \underline{\lambda}}, \overline{\lambda}, \underline{\lambda})$ is in $\Lambda_{\theta}$.

For any probability measure $\pi \in \Qset(\Zspace \times \Zspace)$ whose first marginal satisfies $\pi((A \times \Yspace) \times \Zspace) = \Ppr_{\Xspace}(A)$, $\forall A \in \Bset(\Xspace)$, and whose second marginal satisfies $\supp \pi \subseteq \Zspace \times \hat{\Zspace}_l$, the following holds:
\begin{equation}
\label{eq:expectation-phi-lower-bound-arbitrary-measure}
\expect^{\Ppr_{\Xspace}} \phi_{\theta, \alpha, \beta, \overline{\lambda}, \underline{\lambda}}(\Xrv) \geq \expect^{\pi} [\ell(h_{\theta}, \Zrv) - \alpha c(\Zrv, \Zrv^{\prime}) - \beta(\Zrv^{\prime}) - (\overline{\lambda}(\Yrv) - \underline{\lambda}(\Yrv))],
\end{equation}
with $(\Zrv, \Zrv^{\prime}) \sim \pi$ and $\Zrv = (\Xrv, \Yrv)$. Here we abuse notation slightly and write $\beta$ as a function on $\Zspace$, with $\beta(\zvec_{\ell}^i) \triangleq \beta^i$, and $\overline{\lambda}, \underline{\lambda}$ as functions on $\Yspace$, with $\overline{\lambda}(\yvec^k) \triangleq \overline{\lambda}^k$ and $\underline{\lambda}(\yvec^k) \triangleq \underline{\lambda}^k$. The inequality holds necessarily because $\expect^{\Ppr_{\Xspace}} \phi_{\theta, \alpha, \beta, \overline{\lambda}, \underline{\lambda}}(\xvec)$ is exactly the maximal value of the righthand side of the inequality, over $\pi$ satisfying the above constraints.

Define $\Lambda_{\theta, \ast} = \{(\alpha, \beta, \overline{\lambda}, \underline{\lambda}) \in \reals \times \reals^{N_l} \times \reals^{N_{\Yspace}} \times \reals^{N_{\Yspace}} : \alpha, \overline{\lambda}, \underline{\lambda} \geq 0\}$ and 
\begin{equation}
G_{\theta} : (\alpha, \beta, \overline{\lambda}, \underline{\lambda}) \in \Lambda_{\theta, \ast} \mapsto G(\alpha, \beta, \phi_{\theta, \alpha, \beta, \overline{\lambda}, \underline{\lambda}}, \overline{\lambda}, \underline{\lambda}).
\end{equation}
It is clear from statement of the dual problem in \eqref{eq:dual-lp-components} that for any $(\alpha, \beta, \phi, \overline{\lambda}, \underline{\lambda}) \in \Lambda_{\theta}$, $G_{\theta}(\alpha, \beta, \overline{\lambda}, \underline{\lambda}) \leq G(\alpha, \beta, \phi, \overline{\lambda}, \underline{\lambda})$. This is because $\phi_{\theta, \alpha, \beta, \overline{\lambda}, \underline{\lambda}}$ is the smallest function that satisfies the constraints in \eqref{eq:dual-lp-components}. Moreover, $G_{\theta}(\alpha, \beta, \overline{\lambda}, \underline{\lambda}) \geq \sup_{\pi \in \Phi} F_{\theta}(\pi)$ (i.e. the optimal value of the primal), as Fenchel duality guarantees existence of a primal optimizer $\pi_{\ast}$ and the above fact bounding $\expect^{\Ppr_{\Xspace}} \phi_{\theta, \alpha, \beta, \overline{\lambda}, \underline{\lambda}}$ guarantees
\begin{equation}
\label{eq:dual-expectation-lower-bound}
\begin{aligned}
&G_{\theta}(\alpha, \beta, \overline{\lambda}, \underline{\lambda}) \\
&\quad\quad\quad\quad \geq \expect^{\pi_{\ast}} \Biggl[\alpha (\varepsilon - c(\Zrv, \Zrv^{\prime})) + \sum_{i=1}^{N_l} \left(\frac{1}{N_l} - \delta_{\zvec_{\ell}^i}(\Zrv^{\prime})\right) \beta(\zvec_{\ell}^i) \\
&\quad\quad\quad\quad\quad\quad\quad + \ell(h_{\theta}, \Zrv) + \sum_{k=1}^{N_{\Yspace}} (\overline{\pvec}_{\Yspace}^k - \delta_{\yvec^k}(\Yrv)) \overline{\lambda}(\yvec^k) + (\delta_{\yvec^k}(\Yrv) - \underline{\pvec}_{\Yspace}^k) \underline{\lambda}(\yvec^k)\Biggr] \\
&\quad\quad\quad\quad \geq \expect^{\pi_{\ast}} \ell(h_{\theta}, Z) \\
&\quad\quad\quad\quad = F_{\theta}(\pi_{\ast}).
\end{aligned}
\end{equation}
The second inequality follows from feasibility of $\pi_{\ast}$, meaning that it satisfies the constraints \eqref{eq:primal-constraints}. As we have shown, though, $\sup_{\pi \in \Phi} F_{\theta}(\pi) = \inf_{(\alpha, \beta, \phi, \overline{\lambda}, \underline{\lambda}) \in \Lambda_{\theta}} G(\alpha, \beta, \phi, \overline{\lambda}, \underline{\lambda})$, so it must be that $\sup_{\pi \in \Phi} F_{\theta}(\pi) = \inf_{(\alpha, \beta, \overline{\lambda}, \underline{\lambda}) \in \Lambda_{\theta, \ast}} G_{\theta}(\alpha, \beta, \overline{\lambda}, \underline{\lambda})$.

The claim therefore reduces to existence of a finite minimizer of $G_{\theta}$. Suppose $\overline{\pvec}_{\Yspace}^k > \underline{\pvec}_{\Yspace}^k$ for all $k$ and choose any $\pi \in \operatorname{relint} \Phi$, meaning that
\vspace{-0.1cm}
\begin{equation}
\begin{aligned}
\expect^{\pi} c(\Zrv, \Zrv^{\prime}) &< \varepsilon, \\ \expect^{\pi} \delta_{\zvec_{\ell}^i}(\Zrv^{\prime}) &= \frac{1}{N_l},\ \ \forall i \in \{1, \dots, N_l\}, \\
\expect^{\pi} \delta_{\yvec^k}(\Yrv) &\in (\underline{\pvec}_{\Yspace}^k, \overline{\pvec}_{\Yspace}^k),\ \ \forall k \in \{1, \dots, N_{\Yspace}.
\end{aligned}
\end{equation}
\vspace{-0.1cm}
Substituting $\pi$ for $\pi_{\ast}$, the bound in \eqref{eq:dual-expectation-lower-bound} still holds. Strict feasibility of $\pi$ therefore implies that $G_{\theta}$ is lower bounded by a function linear in $(\alpha, \beta, \overline{\lambda}, \underline{\lambda})$ that is increasing in $\alpha, \overline{\lambda}, \underline{\lambda}$. Since these variables are constrained to be nonnegative, $G_{\theta}$ is therefore coercive in $\alpha$, $\overline{\lambda}$, and $\underline{\lambda}$.

Unfortunately, any feasible $\pi$ yields a lower bound that is independent of $\beta$. We can remedy this, however, as follows. Let $\{\alpha, \beta, \overline{\lambda}, \underline{\lambda}\} \in \Lambda_{\theta, \ast}$. Note that $G_{\theta}$ is invariant under shifts of $\beta$ by a constant, so we can assume $\sum_{i=1}^{N_l} \beta^i = 0$. Therefore for any $\beta \neq 0$ there exists at least one pair of indices $i, i^{\prime}$ such that $\operatorname{sgn} \beta^i \neq \operatorname{sgn} \beta^{i^{\prime}}$.

We will define a set of probability measures $\pi_+$ that have first marginal identical to that of $\pi$ but that have a very small amount of mass shifted so as to alter the second marginal. Call this set $\Psi[\pi]$,
\begin{equation}
\Psi[\pi] = \ball_a(\pi) \cap \{ \pi_+ \in \Qset(\Zspace \times \Zspace) :  \pi_+(A \times \Zspace) = \pi(A \times \Zspace),\ \forall A \in \Bset(\Xspace) \},
\end{equation}
with $\ball_a(\pi)$ the total variation norm ball of radius $a < \frac{\varepsilon - \expect^{\pi} c(\Zrv, \Zrv^{\prime})}{C}$, where $C = \max_{\zvec, \zvec^{\prime} \in \Zspace} c(\zvec, \zvec^{\prime})$. Note that $C$ is finite due to compactness of $\Xspace$, as we assumed $c$ is semicontinuous. And $a$ is a radius sufficiently small to guarantee $\expect^{\pi_+} c(\Zrv, \Zrv^{\prime}) < \varepsilon$, for all $\pi_+ \in \ball_a(\pi)$, despite the shifted mass. $a$ is positive due to strict feasibility of $\pi$.

For any $\pi_+ \in \Psi[\pi]$, the lower bound stated in \eqref{eq:dual-expectation-lower-bound} still holds for $\pi_+$ in place of $\pi^{\ast}$. This means that
\begin{equation}
\label{eq:dual-lower-bound-supremum}
\begin{aligned}
&G_{\theta}(\alpha, \beta, \overline{\lambda}, \underline{\lambda}) \\
&\quad\quad\quad\quad \geq \sup_{\pi_+ \in \Psi[\pi]} \expect^{\pi_+} \Biggl[\alpha (\varepsilon - c(\Zrv, \Zrv^{\prime})) + \sum_{i=1}^{N_l} \left(\frac{1}{N_l} - \delta_{\zvec_{\ell}^i}(\Zrv^{\prime})\right) \beta(\zvec_{\ell}^i) \\
&\quad\quad\quad\quad\quad\quad\quad\quad + \ell(h_{\theta}, \Zrv) + \sum_{k=1}^{N_{\Yspace}} (\overline{\pvec}_{\Yspace}^k - \delta_{\yvec^k}(\Yrv)) \overline{\lambda}(\yvec^k) + (\delta_{\yvec^k}(\Yrv) - \underline{\pvec}_{\Yspace}^k) \underline{\lambda}(\yvec^k)\Biggr],
\end{aligned}
\end{equation}
for all $(\alpha, \beta, \overline{\lambda}, \underline{\lambda}) \in \Lambda_{\theta, \ast}$. This lower bound is a pointwise supremum over linear functions in $(\alpha, \beta, \overline{\lambda}, \underline{\lambda})$, induced by measures $\pi_+ \in \Psi[\pi]$. Importantly, the bound is still increasing in $\alpha, \overline{\lambda}, \underline{\lambda}$, by definition of $\Psi[\pi]$. And for any $\beta$, there exists $\pi_+ \in \Psi[\pi]$ such that $\operatorname{sgn} (\frac{1}{N_l} - \expect^{\pi_+} \delta_{\zvec_{\ell}^i}(\Zrv^{\prime})) = \operatorname{sgn} \beta^i$, for all $i$. This pointwise supremum is increasing in $\alpha, \overline{\lambda}, \underline{\lambda}$, and increasing in $|\beta^i|$ for all i. So the lower bound is coercive and therefore $G_{\theta}$ is coercive. This suffices for existence of a finite $(\alpha, \beta, \overline{\lambda}, \underline{\lambda}) \in \Lambda_{\theta, \ast}$ that optimizes $G_{\theta}$.

The above holds when $\overline{\pvec}_{\Yspace}^k > \underline{\pvec}_{\Yspace}^k$ for all $k$. Suppose now that there exists $k$ such that $\overline{\pvec}_{\Yspace}^k = \underline{\pvec}_{\Yspace}^k$. Then we face the same problem as we did with $\beta$, with the linear lower bound defined by substituting strictly feasible $\pi$ for $\pi_{\ast}$ in \eqref{eq:dual-expectation-lower-bound} now independent of the $\overline{\lambda}^k$ and $\underline{\lambda}^k$ terms. We will deal with this problem analogously to the approach above.

We again start with $\pi \in \operatorname{relint} \Phi$. Now, however, we relax the constraints defining $\Psi[\pi]$, to allow small shifts of mass that preserve only the first $\Xspace$-marginal. We define
\begin{equation}
\Psi_{\Xspace}[\pi] = \ball_a(\pi) \cap \{\pi_+ \in \Qset(\Zspace \times \Zspace) : \pi_+((A \times \Yspace) \times \Zspace) = \pi(A)\ \forall A \in \Bset(\Xspace)\},
\end{equation}
with ball $\ball_a(\pi)$ having radius $a < \min\left\{\frac{\varepsilon - \expect^{\pi} c(\Zrv, \Zrv^{\prime})}{C}, (\overline{\pvec}_{\Yspace}^k - \expect^{\pi} \delta_{\yvec^k}(\Yrv)), (\expect^{\pi} \delta_{\yvec^k}(\Yrv) - \underline{\pvec}_{\Yspace}^k)\right\}$ with $k$ ranging over the indices such that $\overline{\pvec}_{\Yspace}^k \neq \underline{\pvec}_{\Yspace}^k$ and $C$ defined as above.

For any $\pi_+ \in \Psi_{\Xspace}[\pi]$, the lower bound in \eqref{eq:dual-expectation-lower-bound} still holds, when we  substitute $\pi_+$ for $\pi_{\ast}$. So $G_{\theta}$ is lower bounded by a pointwise supremum over linear functions in $(\alpha, \beta, \overline{\lambda}, \underline{\lambda})$, identically to \eqref{eq:dual-lower-bound-supremum}, substituting $\Psi_{\Xspace}[\pi]$ for $\Psi[\pi]$. For each $k$ with $\overline{\pvec}_{\Yspace}^k = \underline{\pvec}_{\Yspace}^k \triangleq \pvec_{\Yspace}^k$, however, the linear terms depending on $\overline{\lambda}^k, \underline{\lambda}^k$ simplify slightly and each one becomes $(\pvec_{\Yspace}^k - \expect^{\pi_+} \delta_{\yvec^k}(\Yrv)) (\overline{\lambda}^k - \underline{\lambda}^k)$. Although $\overline{\lambda}^k$ and $\underline{\lambda}^k$ are nonnegative, their difference is unconstrained. $G_{\theta}$ is independent of shifts of $\overline{\lambda}^k, \underline{\lambda}^k$ by a constant, so we will assume $\min\{\overline{\lambda}^k, \underline{\lambda}^k\} = 0$ and discuss the single unconstrained variable $\lambda^k = \overline{\lambda}^k - \underline{\lambda}^k$.

If there exists at least one $k^{\prime}$ such that $\overline{\pvec}_{\Yspace}^{k^{\prime}} \neq \underline{\pvec}_{\Yspace}^{k^{\prime}}$ then for any $(\alpha, \beta, \overline{\lambda}, \underline{\lambda}) \in \Lambda_{\theta, \ast}$ it is possible to choose $\pi_+ \in \Psi_{\Xspace}[\pi]$ such that $\operatorname{sgn} (\pvec^k - \expect^{\pi_+} \delta_{\yvec^k}(\Yrv)) = \operatorname{sgn} \lambda^k$, for all $k$ such that $\overline{\pvec}_{\Yspace}^k = \underline{\pvec}_{\Yspace}^k$, by shifting mass between the set $\Xspace \times \{\yvec^{k^{\prime}} : \overline{\pvec}_{\Yspace}^{k^{\prime}} \neq \underline{\pvec}_{\Yspace}^{k^{\prime}}\} \times \Zspace$ and the set $\Xspace \times \{\yvec^k : \overline{\pvec}_{\Yspace}^k = \underline{\pvec}_{\Yspace}^k\} \times \Zspace$, such that the corresponding terms have the correct sign. The radius $a$ for $\Psi_{\Xspace}[\pi]$ was chosen explicitly so that this movement of mass (combined with that described above for $\beta$) will yield $\pi_+ \in \Psi_{\Xspace}[\pi]$ that is increasing in all of $\alpha$, $\overline{\lambda}^{k^{\prime}}$, and $\underline{\lambda}^{k^{\prime}}$, $\overline{\pvec}_{\Yspace}^{k^{\prime}} \neq \underline{\pvec}_{\Yspace}^{k^{\prime}}$, while possessing the correct sign for $\beta^i$ and $\lambda^k$, $\overline{\pvec}_{\Yspace}^k = \underline{\pvec}_{\Yspace}^k$. Therefore the supremum over $\Psi_{\Xspace}[\pi]$ lower bounds $G_{\theta}$ and is coercive in $\alpha, \beta, \overline{\lambda}, \underline{\lambda}$.

If $\overline{\pvec}_{\Yspace}^k = \underline{\pvec}_{\Yspace}^k$ for all $k$, then there is an additional symmetry: We can shift $\lambda^k$ by a constant without impacting $G_{\theta}$. Therefore we can assume $\sum_{k=1}^{N_{\Yspace}} \lambda^k = 0$. The rest of the proof proceeds identically to that for $\beta$.

Note that this proof of the second claim relies on the existence of $\pi \in \operatorname{relint} \Phi$. Suppose no such $\pi$ exists, meaning that for all $\pi \in \Phi$ either $\expect^{\pi} c(\Zrv, \Zrv^{\prime}) = \varepsilon$ or $\expect^{\pi} \delta_{\yvec^k}(\Yrv) \in \{\overline{\pvec}_{\Yspace}^k, \underline{\pvec}_{\Yspace}^k\}$, for some $k$ (or both). Suppose the latter is the case. Assuming the pairs $\{(\overline{\pvec}_{\Yspace}^k, \underline{\pvec}_{\Yspace}^k)\}_{k=1}^{N_{\Yspace}}$ are not degenerate in the sense that there exists only one probability vector $\pvec_{\Yspace} \in \simplex^{N_{\Yspace}}$ that satisfies the constraints $\pvec_{\Yspace}^k \in [\underline{\pvec}_{\Yspace}^k, \overline{\pvec}_{\Yspace}^k]$, for all $k$, we can shift any small amount of mass $\delta$ in $\pi$ between the set $\Xspace \times \{\yvec^k : \expect^{\pi} \delta_{\yvec^k}(\Yrv) = \overline{\pvec}_{\Yspace}^k \neq \underline{\pvec}_{\Yspace}^k\} \times \Zspace$ and the set $\Xspace \times \{\yvec^k : \expect^{\pi} \delta_{\yvec^k}(\Yrv) < \overline{\pvec}_{\Yspace}^k \neq \underline{\pvec}_{\Yspace}^k\} \times \Zspace$ and likewise between the set $\Xspace \times \{\yvec^k : \expect^{\pi} \delta_{\yvec^k}(\Yrv) = \underline{\pvec}_{\Yspace}^k \neq \overline{\pvec}_{\Yspace}^k\} \times \Zspace$ and the set $\Xspace \times \{\yvec^k : \expect^{\pi} \delta_{\yvec^k}(\Yrv) > \underline{\pvec}_{\Yspace}^k \neq \underline{\pvec}_{\Yspace}^k\} \times \Zspace$, with the resulting altered probability measure now strictly feasible for radius $\varepsilon$ increased by $C \delta$ ($C$ defined as above, the maximal value of $c$). Moreover, if there is no such $k$ with $\expect^{\pi} \delta_{\yvec^k}(\Yrv) \in \{\overline{\pvec}_{\Yspace}^k, \underline{\pvec}_{\Yspace}^k\}$, then $\expect^{\pi} c(\Zrv, \Zrv^{\prime}) = \varepsilon$ and $\pi$ will be strictly feasible for any radius $\varepsilon$ that is increased by $\delta > 0$. So, assuming nondegenerate $\{(\underline{\pvec}_{\Yspace}^k, \overline{\pvec}_{\Yspace}^k)\}_{k=1}^{N_{\Yspace}}$, there exists at most one value of $\varepsilon$ for which $\Phi \neq \emptyset$ but there is no $\pi \in \operatorname{relint} \Phi$.
\end{proof}

\section{Subgradients of dual program}
\label{sec:appendix-dual-gradients}
Recall that $\Phi(\xvec;\theta,\alpha,\beta,\overline{\lambda},\underline{\lambda})$ is defined in \eqref{eq:big-phi} as 
\[
    \Phi(\xvec;\theta,\alpha,\beta,\overline{\lambda},\underline{\lambda}) = \max_{\substack{k \in \{1, \dots, N_{\Yspace}\}\\i \in \{1, \dots, N_l\}}} \ell(h_{\theta}, (\Xrv, \yvec^k)) -  \left(\alpha c\left((\Xrv, \yvec^k), \zvec_{\ell}^i\right) + \beta^i\right) - (\overline{\lambda}^k - \underline{\lambda}^k).
\]
For fixed $\xvec \in \Xspace$, we can view $\Phi(\xvec;\theta,\alpha,\beta,\overline{\lambda},\underline{\lambda})$ as a function of the dual variables. Algorithm \ref{alg:sgd} relies on computing a subderivative of $\Phi(\xvec;\theta,\alpha,\beta,\overline{\lambda},\underline{\lambda})$ with respect to the variables $\theta, \alpha, \beta, \underline{\lambda}, \overline{\lambda}$. Assuming that $\ell(h_\theta, \cdot)$  admits a subderivative for any $\theta$, we can use the following expressions:
\begin{equation}
\begin{aligned}
\frac{\partial}{\partial \theta^j} \Phi(\xvec;\theta,\alpha,\beta,\overline{\lambda},\underline{\lambda}) &\in \sum_{i=1}^{N_l} \sum_{k=1}^{N_{\Yspace}} \mathbbm{1}_{V^{ik}}(\xvec) \frac{\partial}{\partial \theta^j} \ell(h_{\theta}, (\xvec, \yvec^k)), \\
\frac{\partial}{\partial \alpha} \Phi(\xvec;\theta,\alpha,\beta,\overline{\lambda},\underline{\lambda}) &= -\sum_{i=1}^{N_l} \sum_{k=1}^{N_{\Yspace}} \mathbbm{1}_{V^{ik}}(\xvec) c((\xvec, \yvec^k), \zvec_{\ell}^i), \\
\frac{\partial}{\partial \beta^i}\Phi(\xvec;\theta,\alpha,\beta,\overline{\lambda},\underline{\lambda})  &= -\sum_{k=1}^{N_{\Yspace}} \mathbbm{1}_{V^{ik}}(\xvec), \\
\frac{\partial}{\partial \overline{\lambda}^k} \Phi(\xvec;\theta,\alpha,\beta,\overline{\lambda},\underline{\lambda}) &= -\sum_{i=1}^{N_l} \mathbbm{1}_{V^{ik}}(\xvec), \\
\frac{\partial}{\partial \underline{\lambda}^k} \Phi(\xvec;\theta,\alpha,\beta,\overline{\lambda},\underline{\lambda}) &= \sum_{i=1}^{N_l} \mathbbm{1}_{V^{ik}}(\xvec).
\end{aligned}
\end{equation}
For $\xvec$ lying on the boundary between two of the sets $V^{ik}$, we can obtain a subgradient by arbitrarily selecting only one of these $V^{ik}$ to contain $\xvec$ when evaluating $\mathbbm{1}_{V^{ik}}(\xvec)$. In most practical settings, the boundaries should have lower dimension than $\Xspace$ and therefore $\Ppr_{\Xspace}$-measure zero.

\section{Datasets}

In the experiments described in Sections \ref{sec:empirical-performance-of-drl-with-unlabeled}, \ref{sec:how-important-is-radius}, and \ref{sec:active-learning-empirical}, we use $14$ real datasets, taken from the UCI repository \citep{dua2019:uci}. Datasets were chosen from amongst those of ``multivariate classification'' type having more than $1000$ examples. We attempted to select frequently-downloaded datasets from a variety of domains. We excluded datasets on which a $\ell^2$-regularized linear logistic regression using $100$ randomly selected training samples could not achieve likelihood greater than $0.55$.
Table \ref{tbl:datasets} shows the datasets along with the number of examples and class balance in each dataset.

In every experiment, the full dataset was first standardized by subtracting out the mean and dividing by the standard deviation, per-feature, then scaling the resulting data matrix by dividing out the maximum absolute value over all entries.

Some caveats apply:
\begin{itemize}
    \item {\bf Abalone}: We excluded all examples from classes $9$ and $10$, to ensure non-trivial classification performance was possible with a small number of labeled examples.
    \item {\bf Bank} and {\bf Cover}: We chose $10000$ examples uniformly without replacement.
    \item {\bf Cover}: We classified types $2$ vs $3$ and $5$ vs. $6$.
    \item {\bf Isolet}: We classified vowels vs.\ the rest.
    \item{\bf Letter recognition}: We classified ``C'' vs. ``E'' and ``U'' vs. ``V.''
    \item {\bf Bank}, {\bf Magic}, and {\bf Pulsar}: We chose examples uniformly without replacement from the larger class to achieve exact balance.
    \item {\bf Thyroid}: We classified ``3'' vs.\ the rest.
    \item {\bf Wine}: We excluded all examples with rating $6$, to ensure non-trivial classification performance was possible with a small number of labeled examples.
    \item {\bf Yeast}: We classified ``nuclear'' vs.\ the rest.
\end{itemize}

\begin{table}
\caption{Datasets used in this paper. All are from the UCI repository \citep{dua2019:uci}.}
\label{tbl:datasets}
    \centering
    \begin{tabular}{rcccl}
        \toprule
        Dataset & Full name & N. features & N. examples & \% positive \\
        \midrule
        Abalone & Abalone & 10 & 2854 & 50.7 \\
        Bank & Bank Marketing & 53 & 10000 & 50.0 \\
        Cover (2/3) & Cover Type & 54 & 10000 & 11.2 \\
        Cover (5/6) & Cover Type & 54 & 10000 & 64.7 \\
        Isolet & Isolet & 617 & 7797 & 19.2 \\
        Letter (C/E) & Letter Recognition & 256 & 1504 & 48.9 \\
        Letter (U/V) & Letter Recognition & 256 & 1577 & 51.6 \\
        Magic & MAGIC Gamma Telescope & 10 & 13376 & 50.0 \\
        Mushroom & Mushroom & 117 & 8124 & 48.2 \\
        Pulsar & HTRU2 & 8 & 3278 & 50.0 \\
        Spam & Spambase & 57 & 4601 & 39.4 \\
        Thyroid & Thyroid Disease & 21 & 7200 & 92.6 \\
        Wine & Wine & 11 & 2700 & 39.3 \\
        Yeast & Yeast & 8 & 1484 & 28.9 \\
        \bottomrule
    \end{tabular}
\end{table}

\section{Description of experiments and additional empirical results}

In all experiments we use the transport cost $c((\xvec, \yvec), (\xvec^{\prime}, \yvec^{\prime})) = \|\xvec - \xvec^{\prime}\|_2 + \frac{\kappa}{2} |\yvec - \yvec^{\prime}|$ with $\kappa = 1$ and $\yvec \in \{+1, -1\} \subset \reals$.

\subsection{Performance when the decision set contains the true data distribution}
\label{sec:appendix-true-radius}

The following describes Figures \ref{fig:true-radius-fval} and \ref{fig:true-radius-conf} in Section \ref{sec:empirical-performance-of-drl-with-unlabeled} as well as the supplementary Figures \ref{fig:appendix-true-radius-fval} and \ref{fig:appendix-true-radius-conf}. Figure \ref{fig:true-radius-fval} shows the worst-case likelihood bound output by each algorithm, varying the number of training examples. Specifically, we choose $\varepsilon$ to be the smallest radius such that the Wasserstein ball $\ball_{\varepsilon}(\hat{\Ppr}_l)$ contains the true data distribution $\Ppr$ (i.e. the empirical distribution defined by the full dataset), and compute a linear logistic regression under the traditional Wasserstein DRL model (Equation \eqref{eq:distributionally-robust-learning}, setting $\Pset = \ball_{\varepsilon}(\hat{\Ppr}_l)$) and under the proposed model (with $\Pset = \ball_{\varepsilon}(\hat{\Ppr}_l) \cap \Uset(\Ppr_{\Xspace}, \overline{\pvec}_{\Yspace}, \underline{\pvec}_{\Yspace})$). In both cases the problem can be written
\begin{equation}
\label{eq:appendix-drl-linear-logistic-regression}
\minimize_{\theta \in \reals^{q+1}} \sup_{\mu \in \Pset} \expect^{\mu} \Yrv \log(1 + \exp\{-\langle \Xrv, \theta \rangle\}) + (1 - \Yrv) \log(1 + \exp\{\langle \Xrv, \theta \rangle\}),
\end{equation}
with $\Xspace = \reals^q \times \{1\}$ the feature space and $\Yspace = \{0,1\}$ the label space. There are two settings for the intervals $[\overline{\pvec}_{\Yspace}^k, \underline{\pvec}_{\Yspace}^k]$ constraining the label probabilities. The first is the ``strong'' prior, in which we know the exact label marginal $\Ppr_{\Yspace} = \frac{1}{N_{\Yspace}} \sum_{k=1}^{N_{\Yspace}} \pvec_{\Yspace}^k \delta_{\yvec^k}$ and we set $\overline{\pvec}_{\Yspace}^k = \underline{\pvec}_{\Yspace}^k = \pvec_{\Yspace}^k$ for all $k$. The second setting is the ``weak'' prior, in which we estimate a $95\%$ confidence interval for the label probability, directly from the labeled sample $\hat{\Zspace}_l$, using the method of \citet{clopper1934use}.

We solve \eqref{eq:appendix-drl-linear-logistic-regression} via its dual (Section \ref{sec:drl-problem-formulation}),
using the Adam optimizer \citep{kingma2014adam} with $\beta_1 = 0.9$, $\beta_2 = 0.999$, $\epsilon = 10^{-8}$, and a batch size of $100$ and decreasing the learning rate by a factor of $8$ every $10000$ steps. The resulting dual objective value is used as the negative log of the worst-case likelihood bound shown in Figure \ref{fig:true-radius-fval}.

Figure \ref{fig:true-radius-conf} shows the median over unlabeled input examples of the confidence of the learned predictor evaluated on those examples. For example $\xvec \in \Xspace$ the confidence is $\max\{h_{\theta}(\xvec), 1 - h_{\theta}(\xvec)\}$. 

In both figures, the solid line is the median over $40$ independent trials and the shaded region is the $95\%$ interval of the median. Each trial represents a single independent sample of $N_l$ labeled examples from the given dataset, taken uniformly without replacement.

Figure \ref{fig:appendix-true-radius-fval} shows the worst-case likelihood bound for additional datasets and Figure \ref{fig:appendix-true-radius-conf} the median confidence.

\begin{figure}
\centering
\begin{subfigure}[b]{0.36\textwidth}
\includegraphics[width=\textwidth]{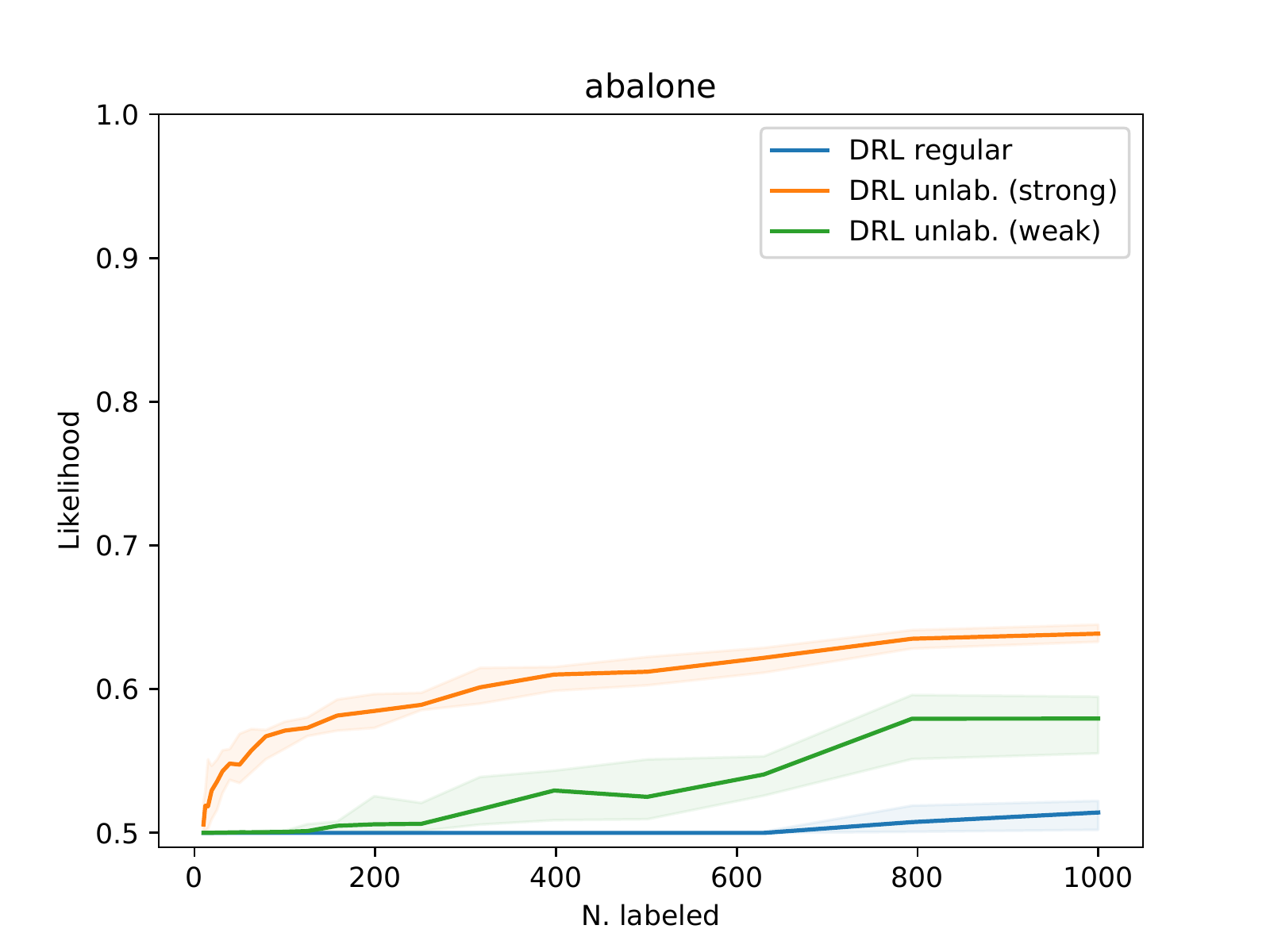}
\caption{Abalone}
\end{subfigure}
\begin{subfigure}[b]{0.36\textwidth}
\includegraphics[width=\textwidth]{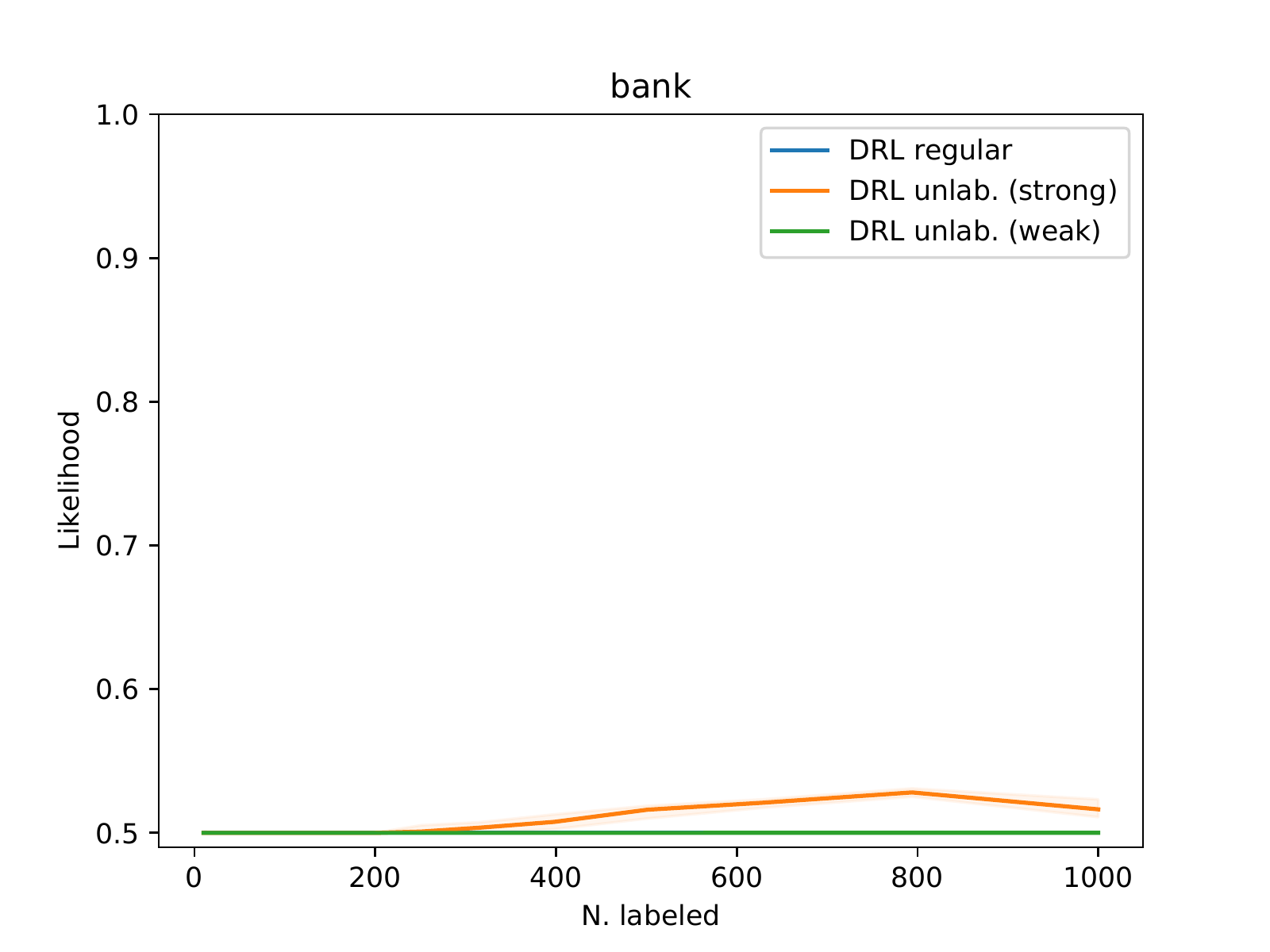}
\caption{Bank}
\end{subfigure}
\begin{subfigure}[b]{0.36\textwidth}
\includegraphics[width=\textwidth]{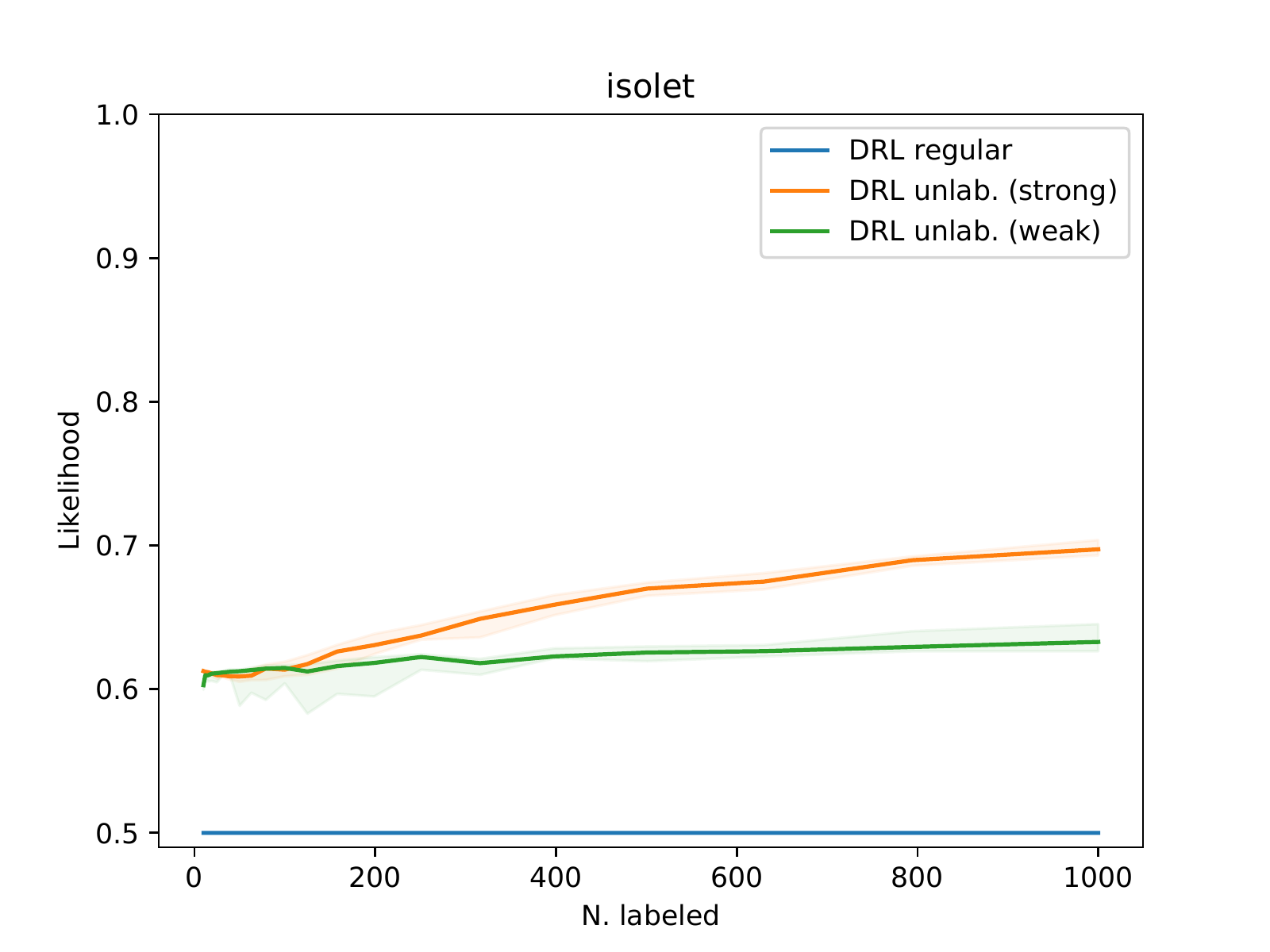}
\caption{Isolet}
\end{subfigure}
\begin{subfigure}[b]{0.36\textwidth}
\includegraphics[width=\textwidth]{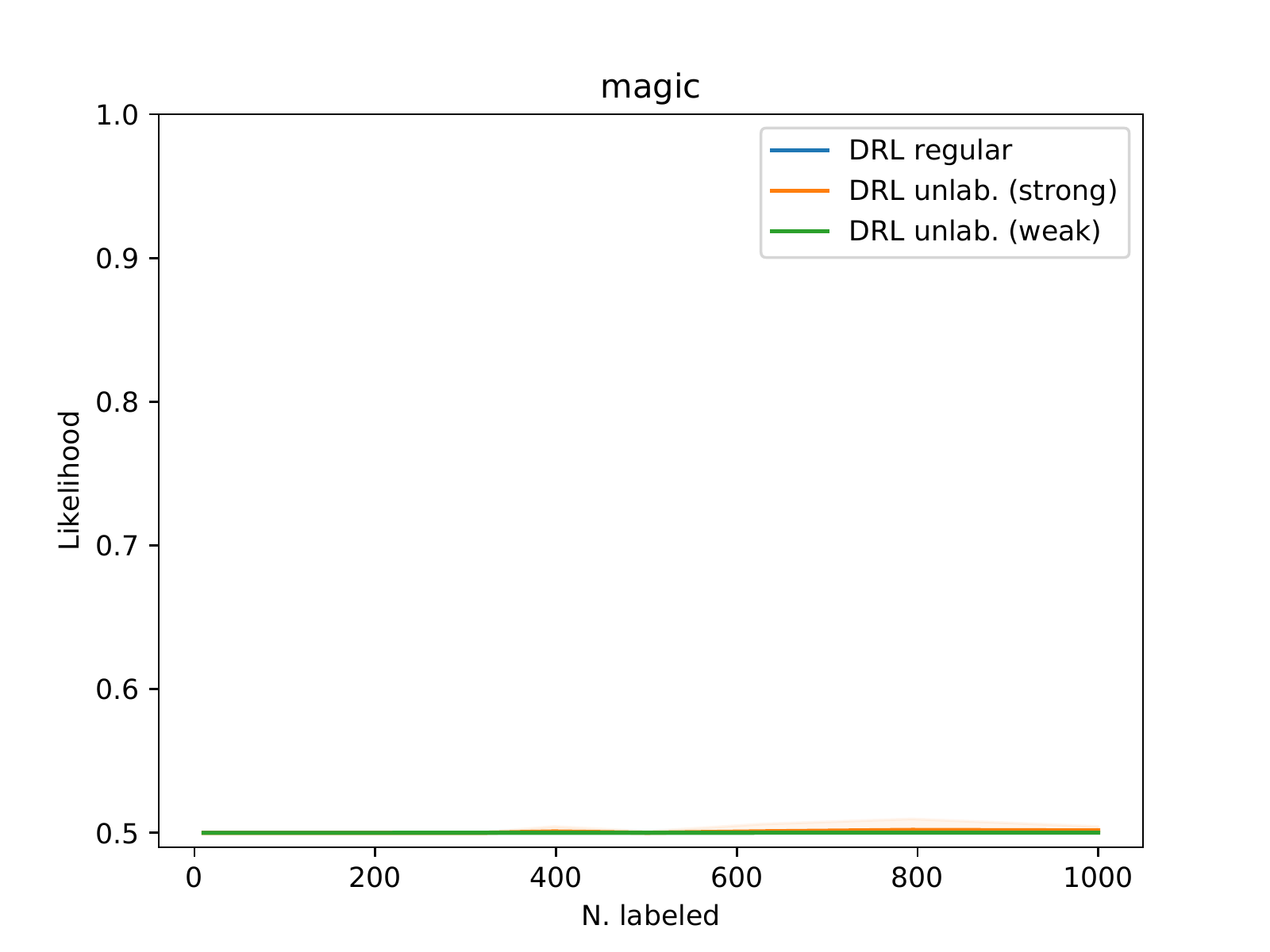}
\caption{Magic}
\end{subfigure}
\begin{subfigure}[b]{0.36\textwidth}
\includegraphics[width=\textwidth]{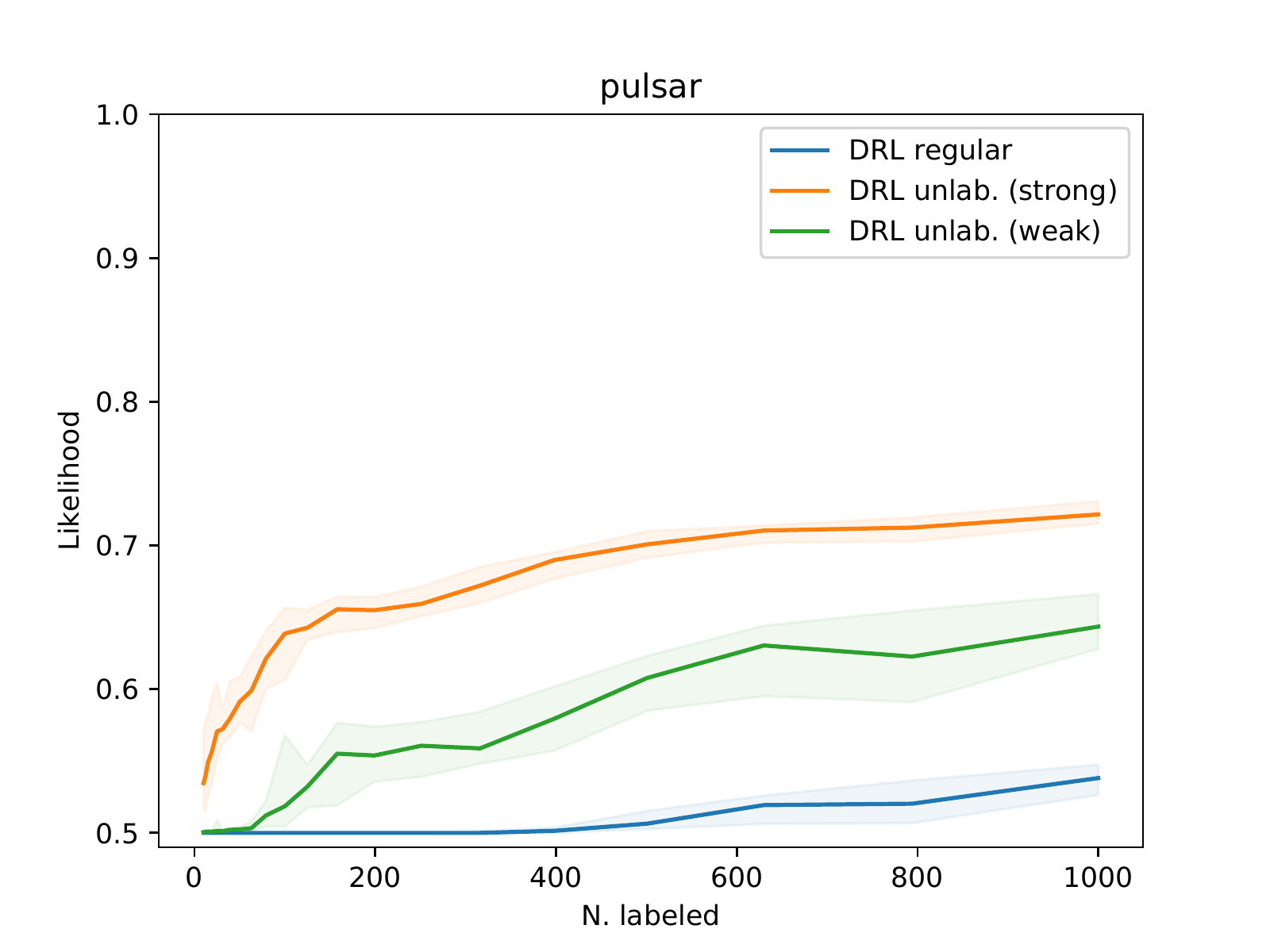}
\caption{Pulsar}
\end{subfigure}
\begin{subfigure}[b]{0.36\textwidth}
\includegraphics[width=\textwidth]{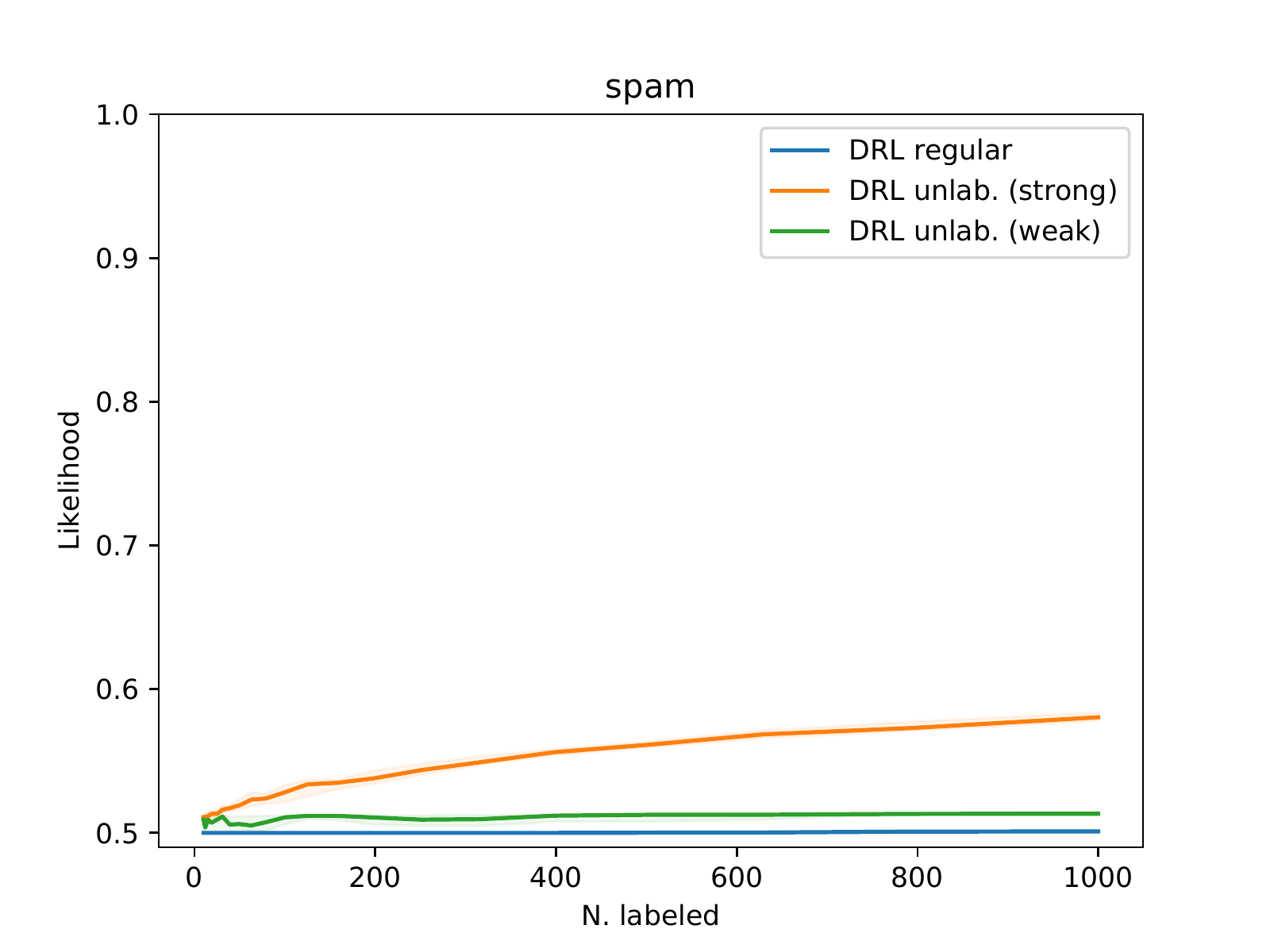}
\caption{Spam}
\end{subfigure}
\begin{subfigure}[b]{0.36\textwidth}
\includegraphics[width=\textwidth]{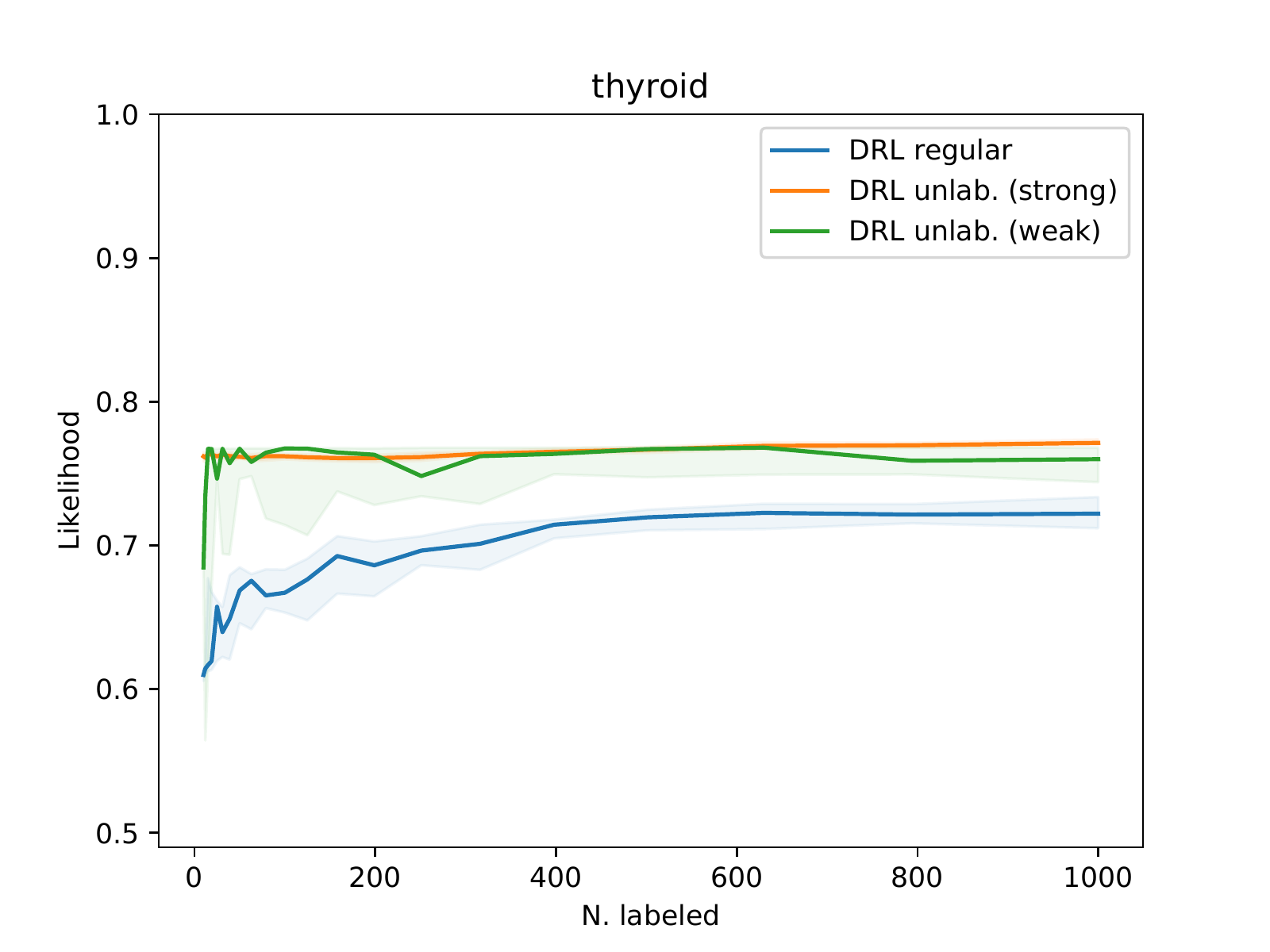}
\caption{Thyroid}
\end{subfigure}
\begin{subfigure}[b]{0.36\textwidth}
\includegraphics[width=\textwidth]{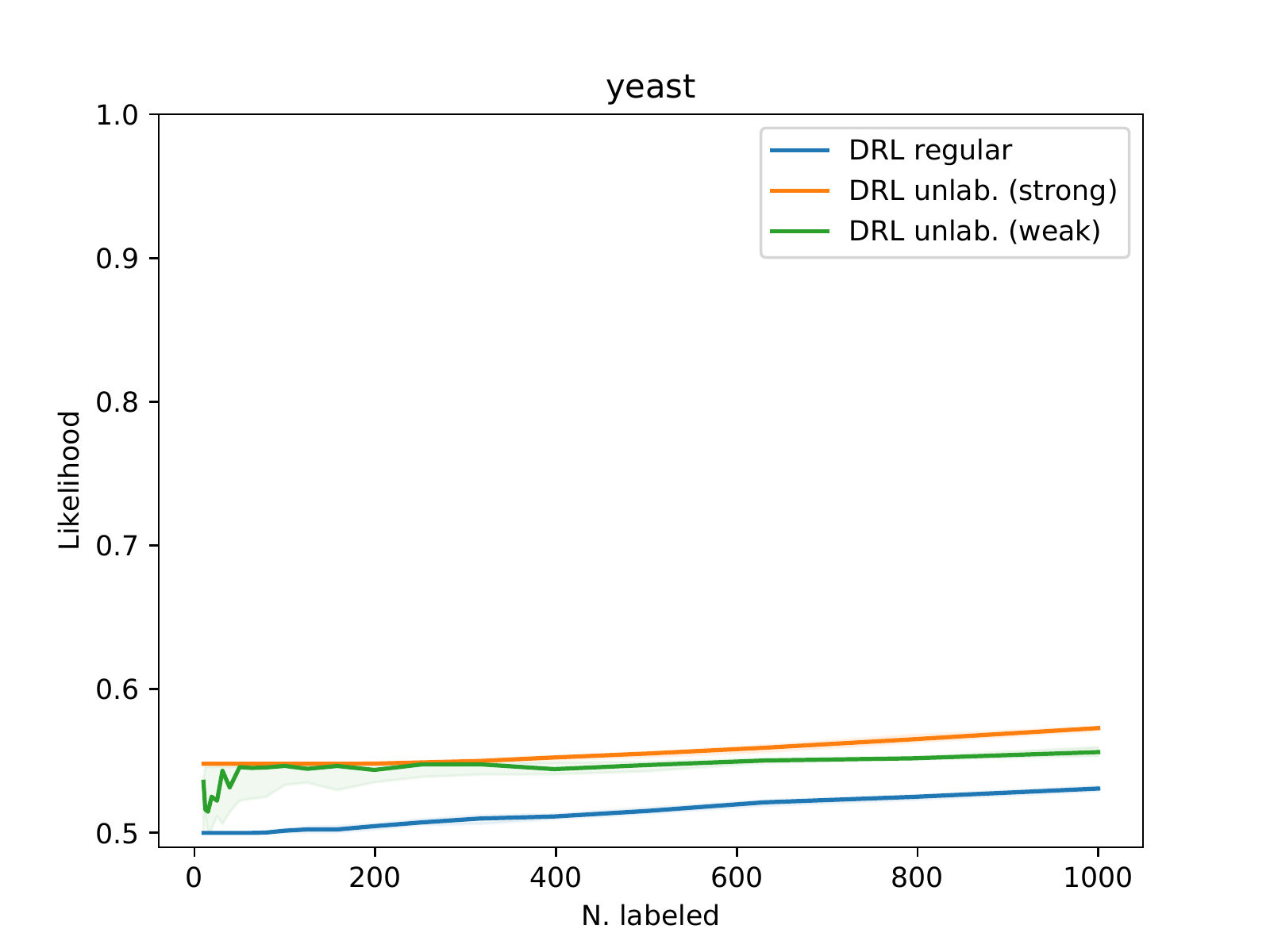}
\caption{Yeast}
\end{subfigure}
\caption{Worst-case performance bound (likelihood) vs. number of labeled data, setting $\varepsilon$ to include the true test distribution. The regular DRL bound is often vacuous through $N_l = 1000$ while both DRL methods with unlabeled data often yield non-vacuous bounds.}
\label{fig:appendix-true-radius-fval}
\end{figure}

\begin{figure}
\centering
\begin{subfigure}[b]{0.36\textwidth}
\includegraphics[width=\textwidth]{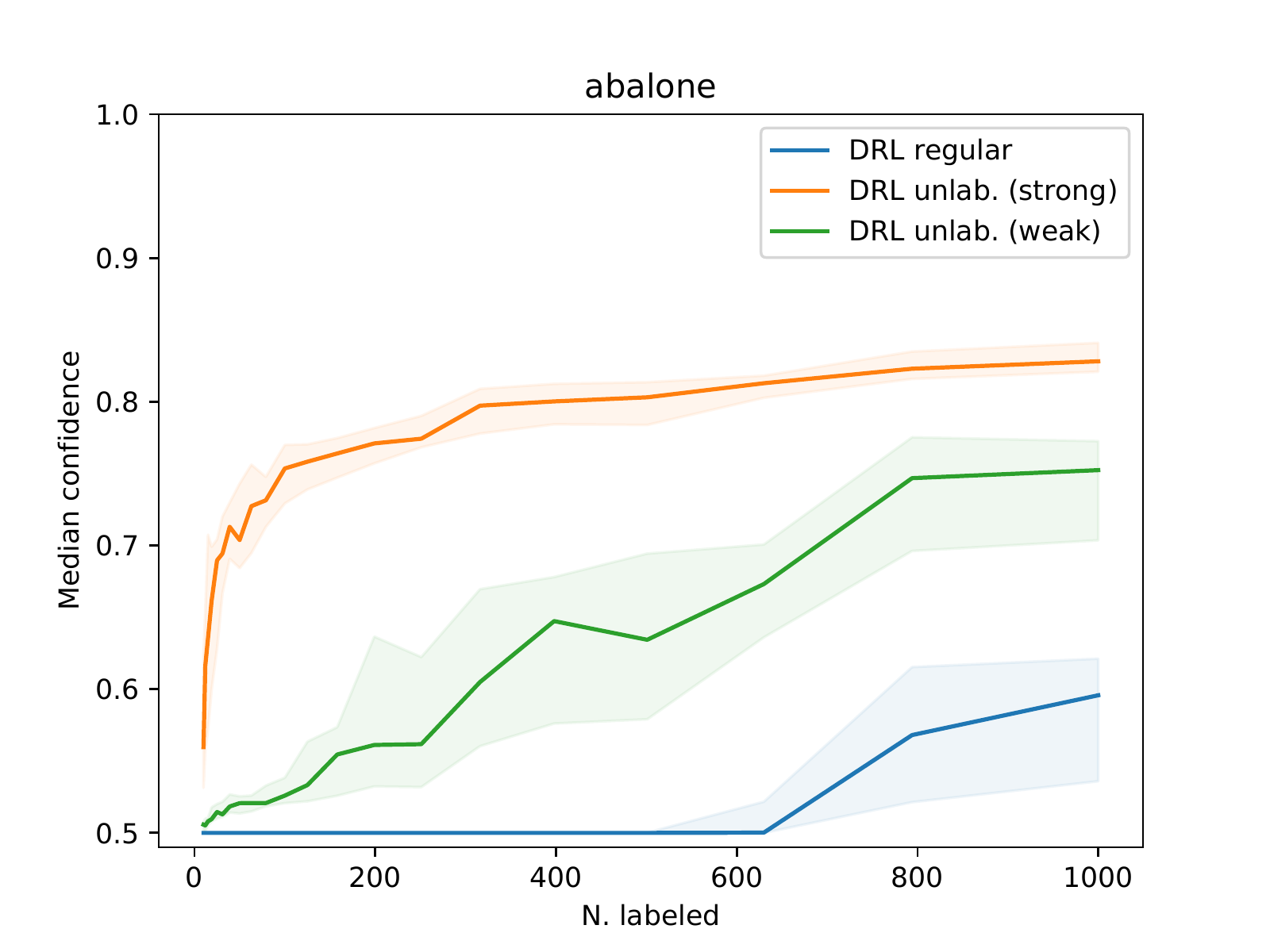}
\caption{Abalone}
\end{subfigure}
\begin{subfigure}[b]{0.36\textwidth}
\includegraphics[width=\textwidth]{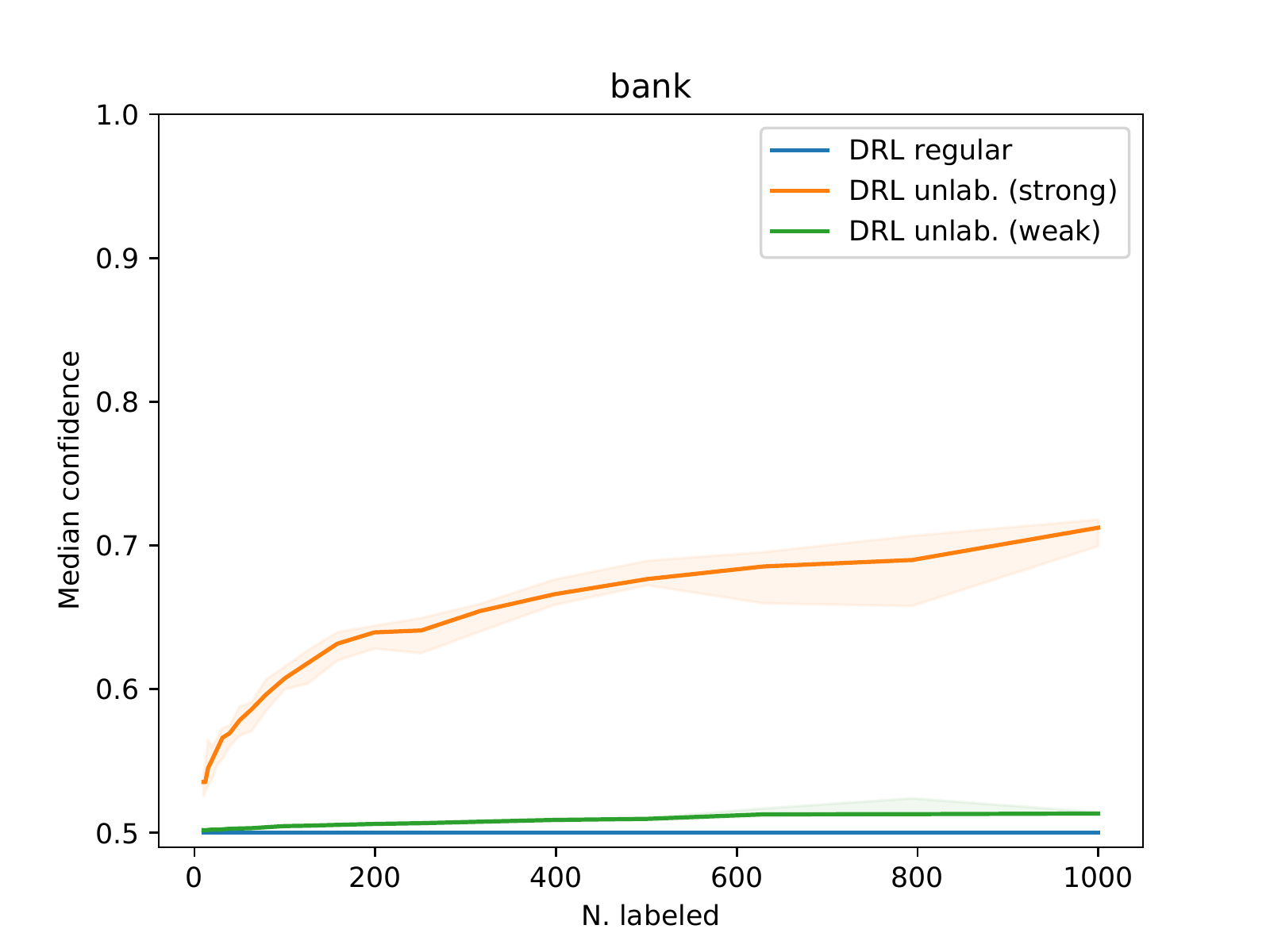}
\caption{Bank}
\end{subfigure}
\begin{subfigure}[b]{0.36\textwidth}
\includegraphics[width=\textwidth]{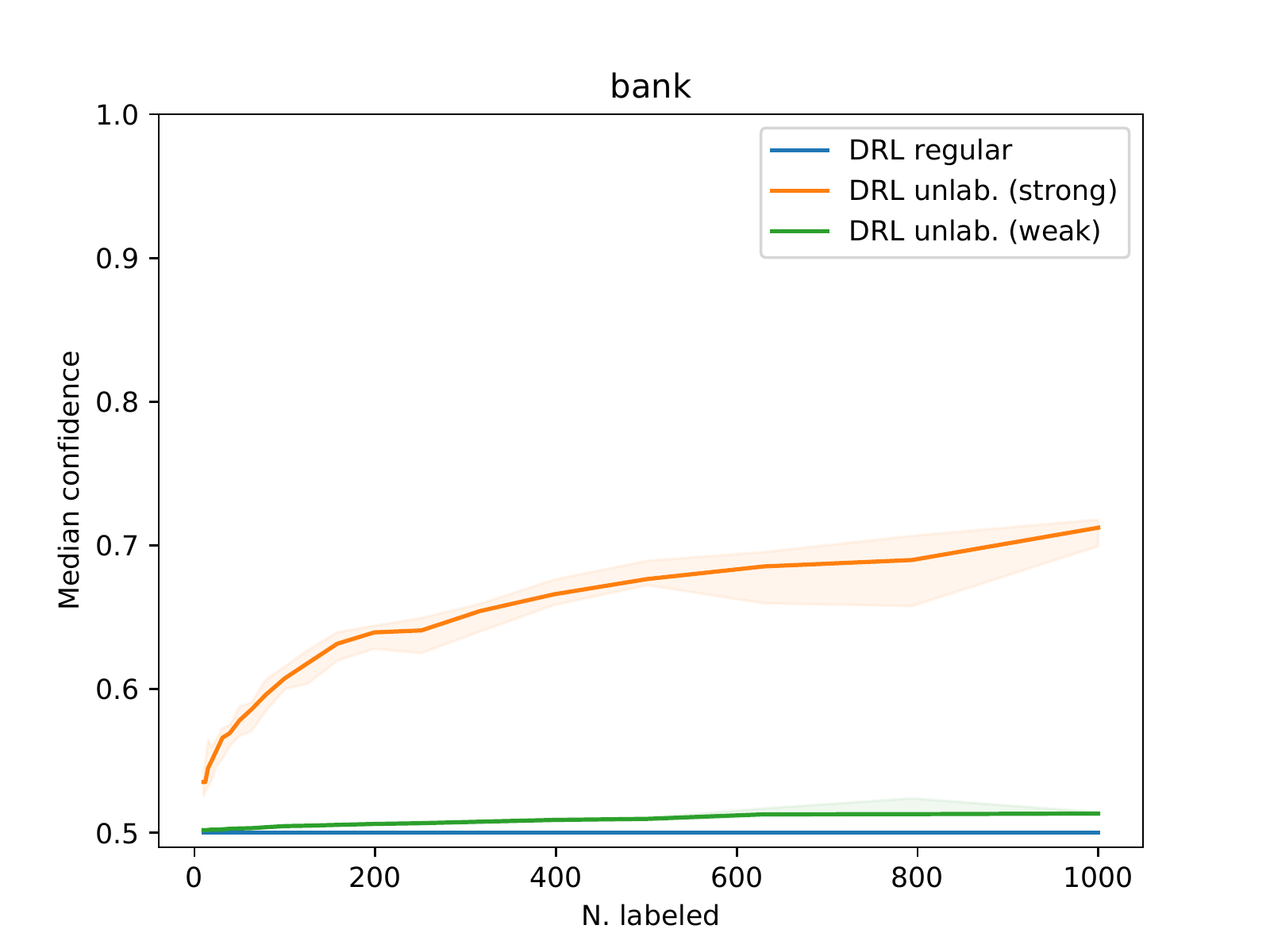}
\caption{Isolet}
\end{subfigure}
\begin{subfigure}[b]{0.36\textwidth}
\includegraphics[width=\textwidth]{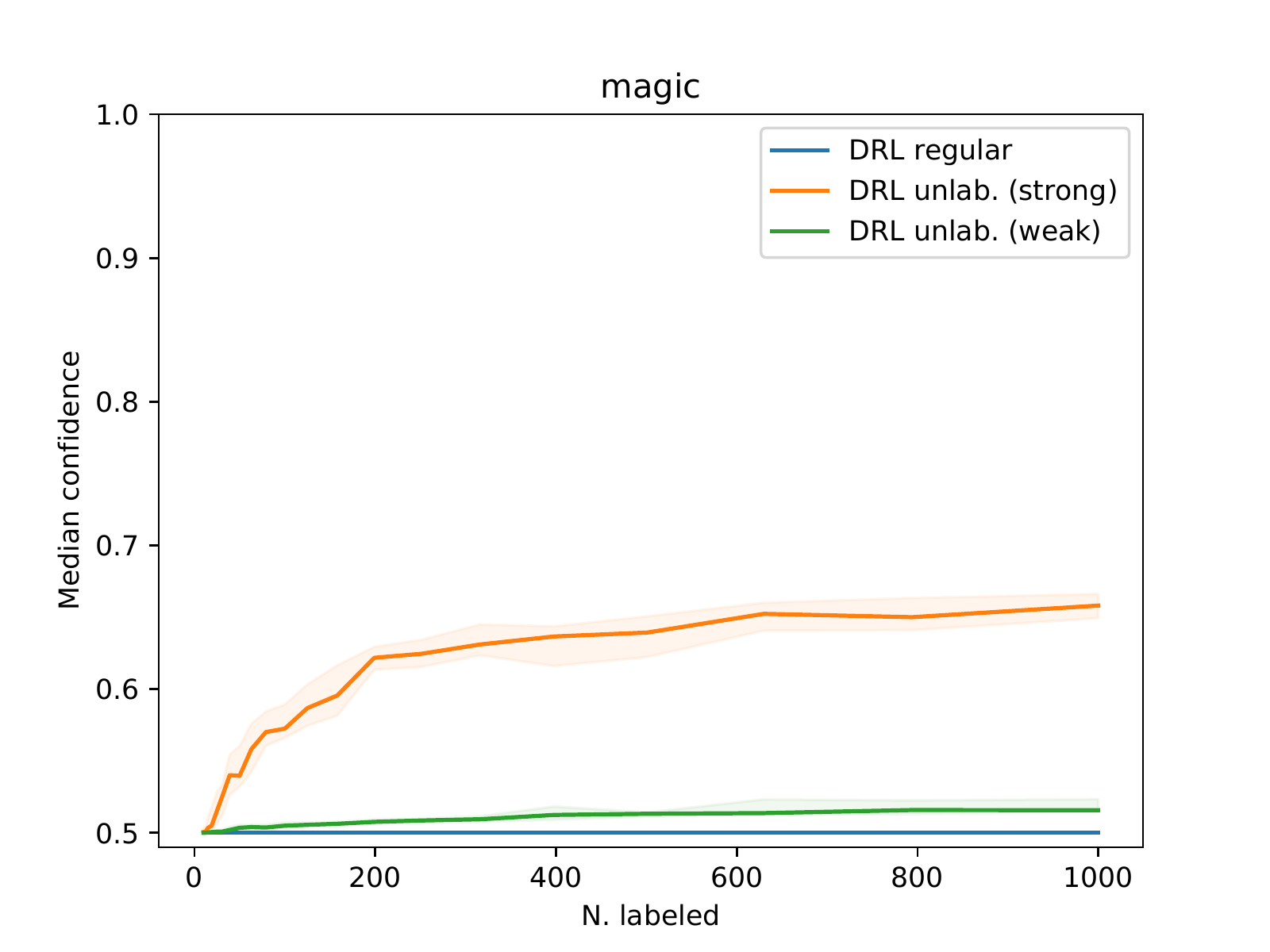}
\caption{Magic}
\end{subfigure}
\begin{subfigure}[b]{0.36\textwidth}
\includegraphics[width=\textwidth]{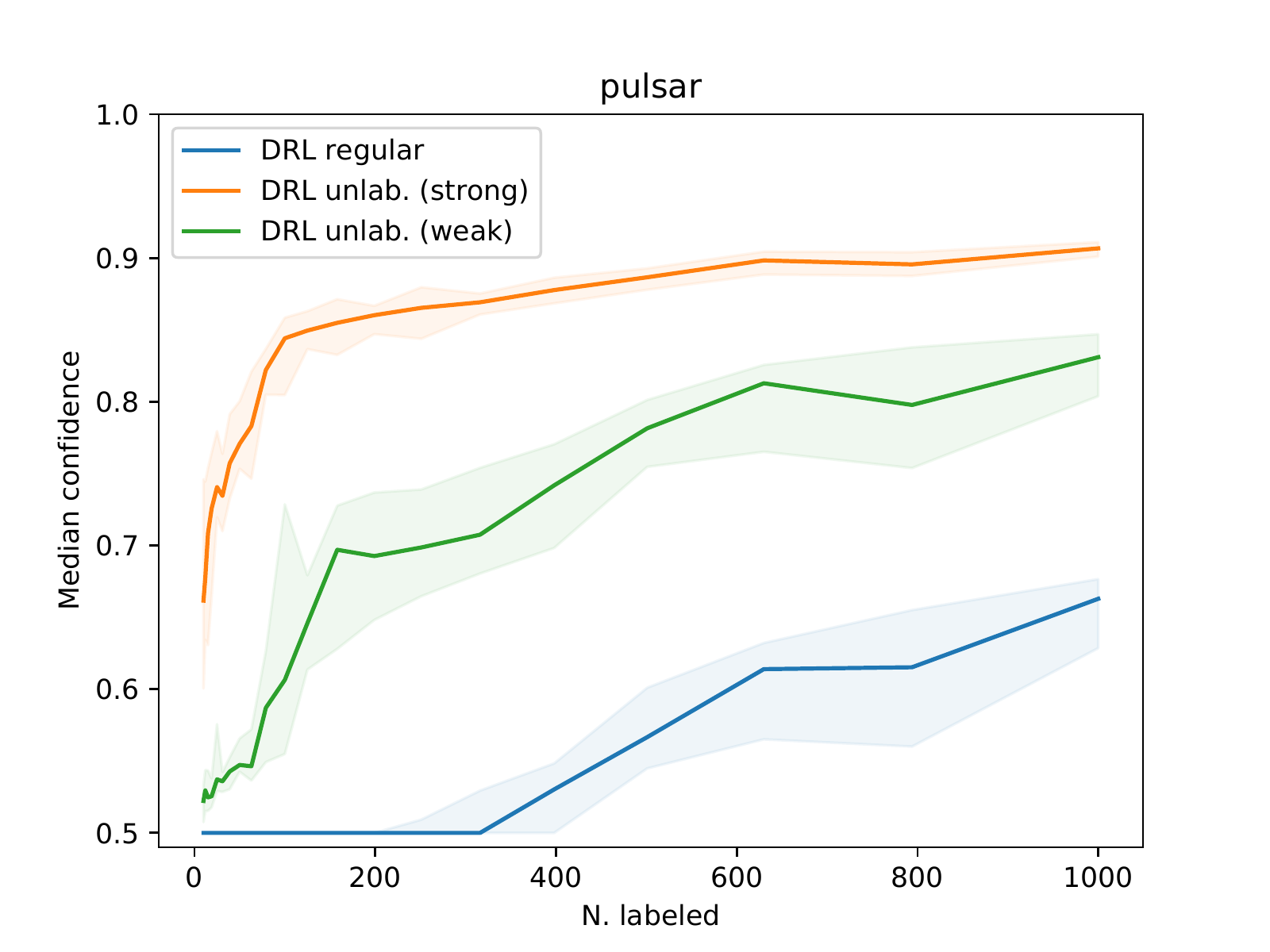}
\caption{Pulsar}
\end{subfigure}
\begin{subfigure}[b]{0.36\textwidth}
\includegraphics[width=\textwidth]{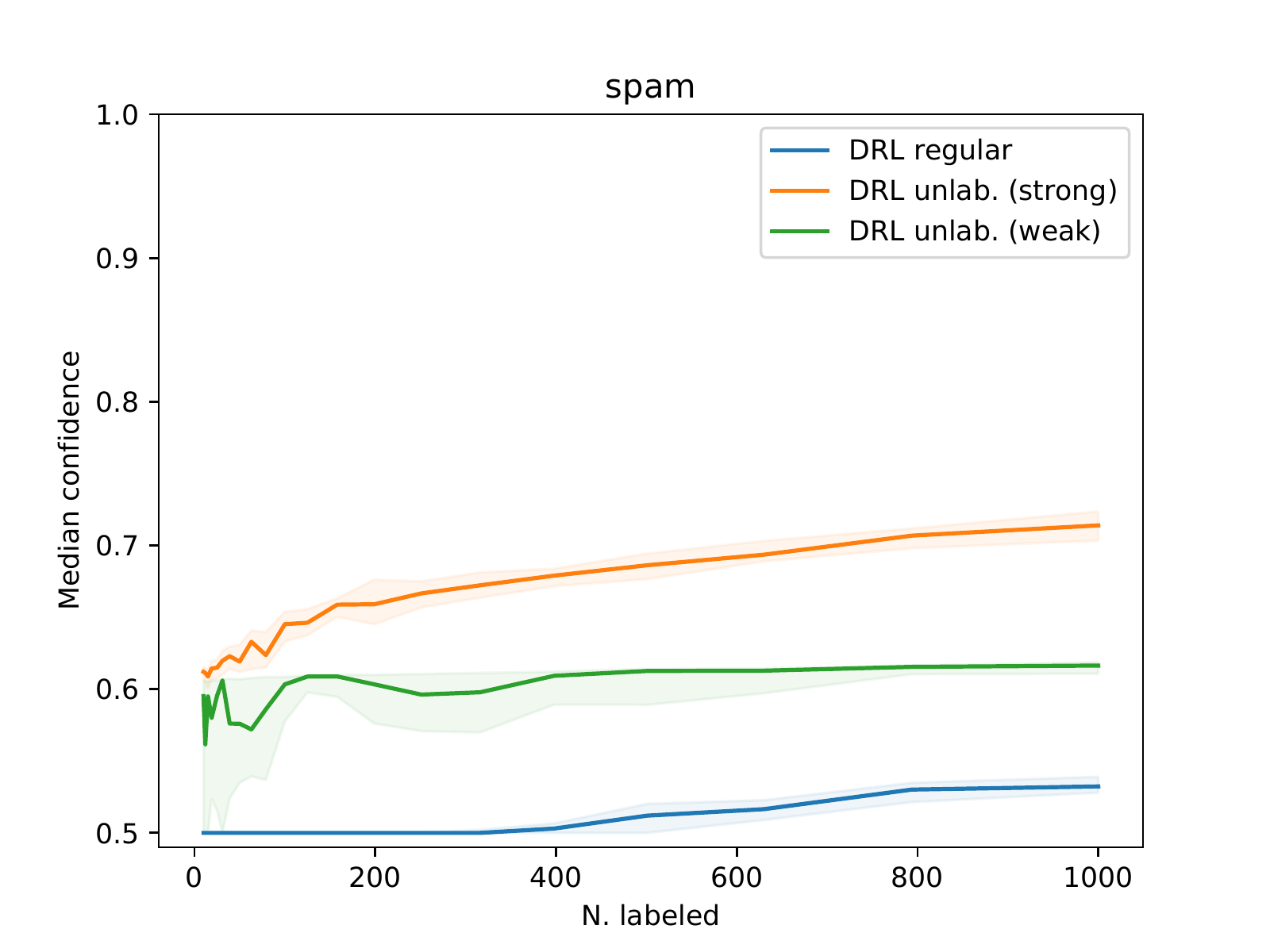}
\caption{Spam}
\end{subfigure}
\begin{subfigure}[b]{0.36\textwidth}
\includegraphics[width=\textwidth]{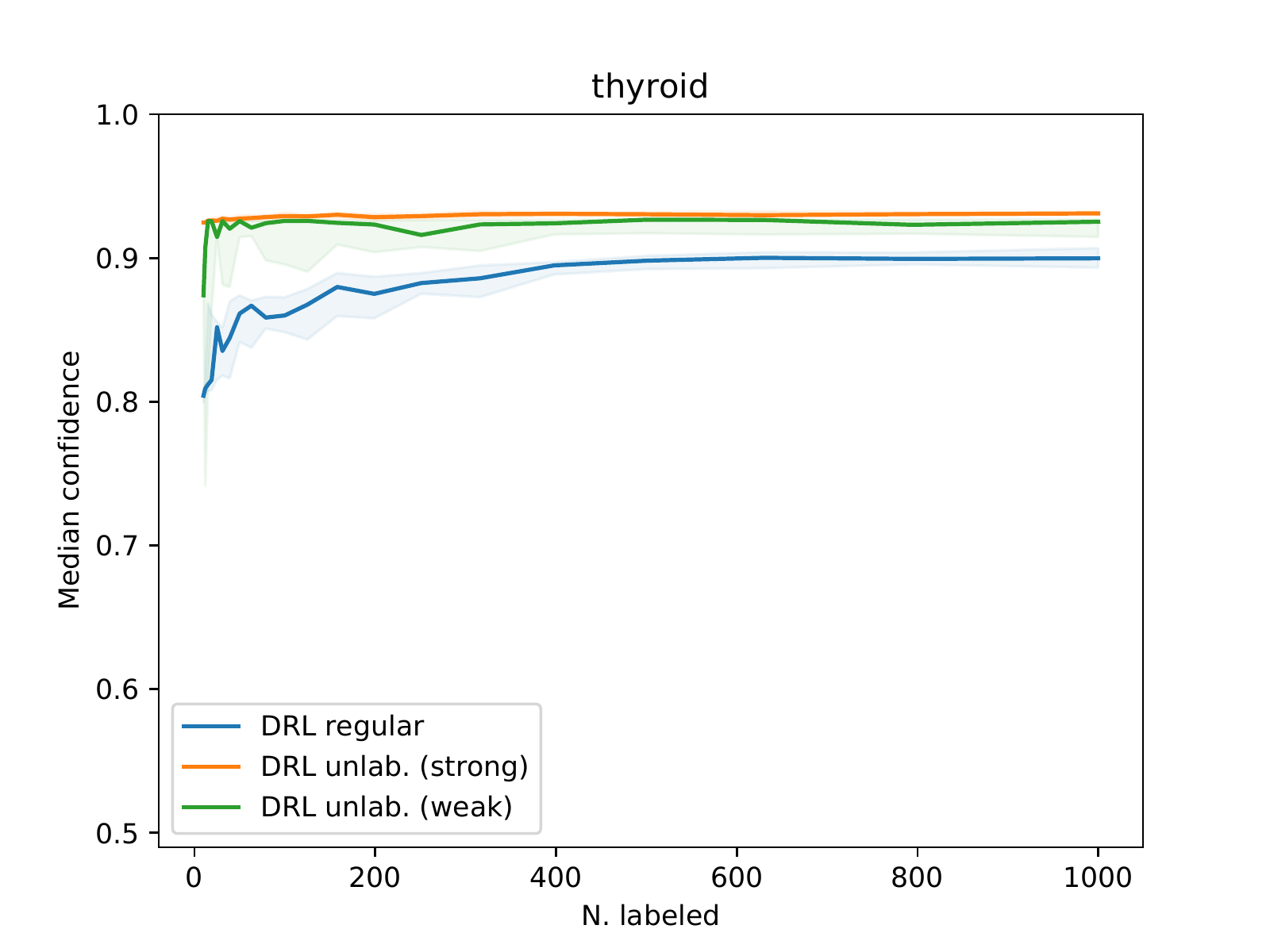}
\caption{Thyroid}
\end{subfigure}
\begin{subfigure}[b]{0.36\textwidth}
\includegraphics[width=\textwidth]{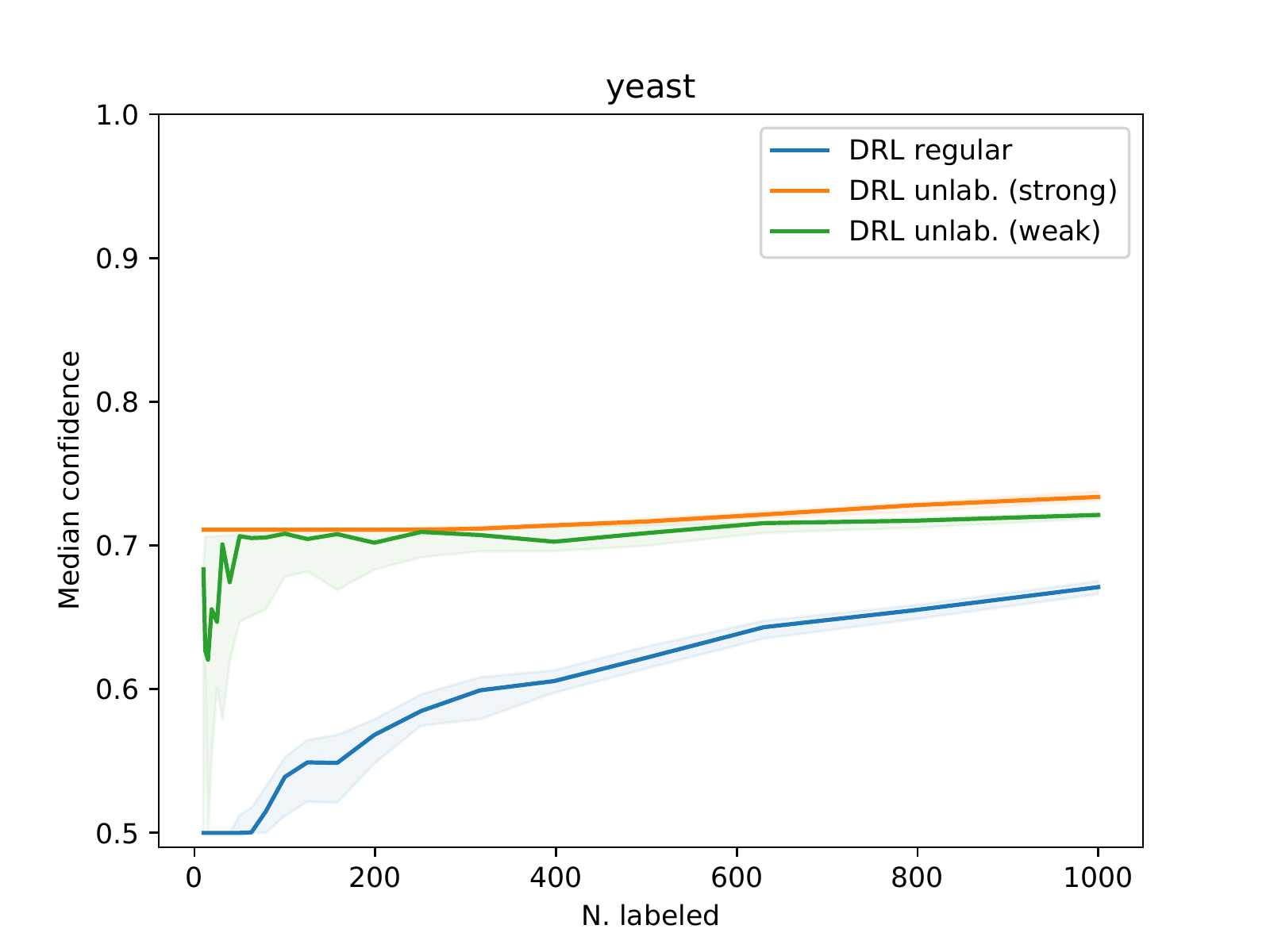}
\caption{Yeast}
\end{subfigure}
\caption{Median confidence vs. number of labeled data, setting $\varepsilon$ to include the true test distribution. The regular DRL predictor often has confidence close to $0.5$ in settings where both DRL methods using unlabeled data yield non-trivial predictors.}
\label{fig:appendix-true-radius-conf}
\end{figure}

\subsection{Traditional Wasserstein DRL performance as we vary $\varepsilon$}
\label{sec:appendix-regular-drl-generalization-vs-radius}

The following describes Figure \ref{fig:likelihood-and-confidence-vs-radius} in Section \ref{sec:how-important-is-radius} and Figure \ref{fig:appendix-likelihood-and-confidence-vs-radius} in the appendix. The figures show out-of-sample generalization performance as well as confidence of predictors trained using traditional Wasserstein DRL, as we vary the radius of robustness $\varepsilon$. In particular, for each trial, we sample a dataset of $N_l = 100$ labeled examples, which we use to compute the Wasserstein distributionally robust logistic regression \citep{abadeh2015distributionally}, setting the radius $\varepsilon$ to be a fixed percentage of the distance to the data distribution $\Ppr$, i.e., the empirical distribution of the full dataset. The log of this percentage is shown on the horizontal axis of the figure.

The figures show both the test set likelihood and the maximum over input examples of the confidence of the learned predictor, defined by $\max\{h_{\theta}(\xvec), 1 - h_{\theta}(\xvec)\}$ for each $\xvec \in \Xspace$. The solid line is the median over $100$ trials while the shaded region shows the $95\%$ interval of the median.

\begin{figure}
\centering
\begin{subfigure}[b]{0.36\textwidth}
\includegraphics[width=\textwidth]{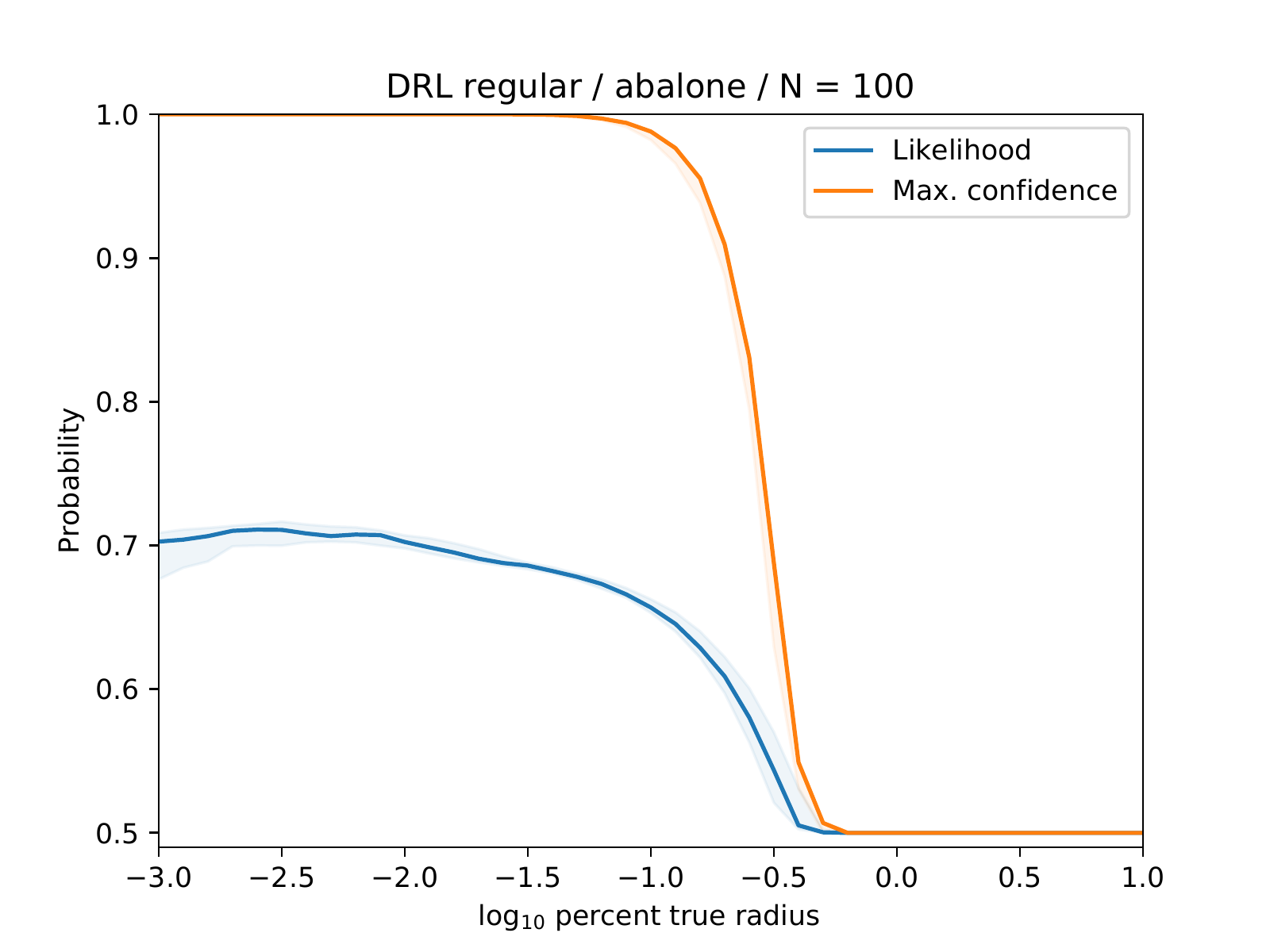}
\caption{Abalone}
\end{subfigure}
\begin{subfigure}[b]{0.36\textwidth}
\includegraphics[width=\textwidth]{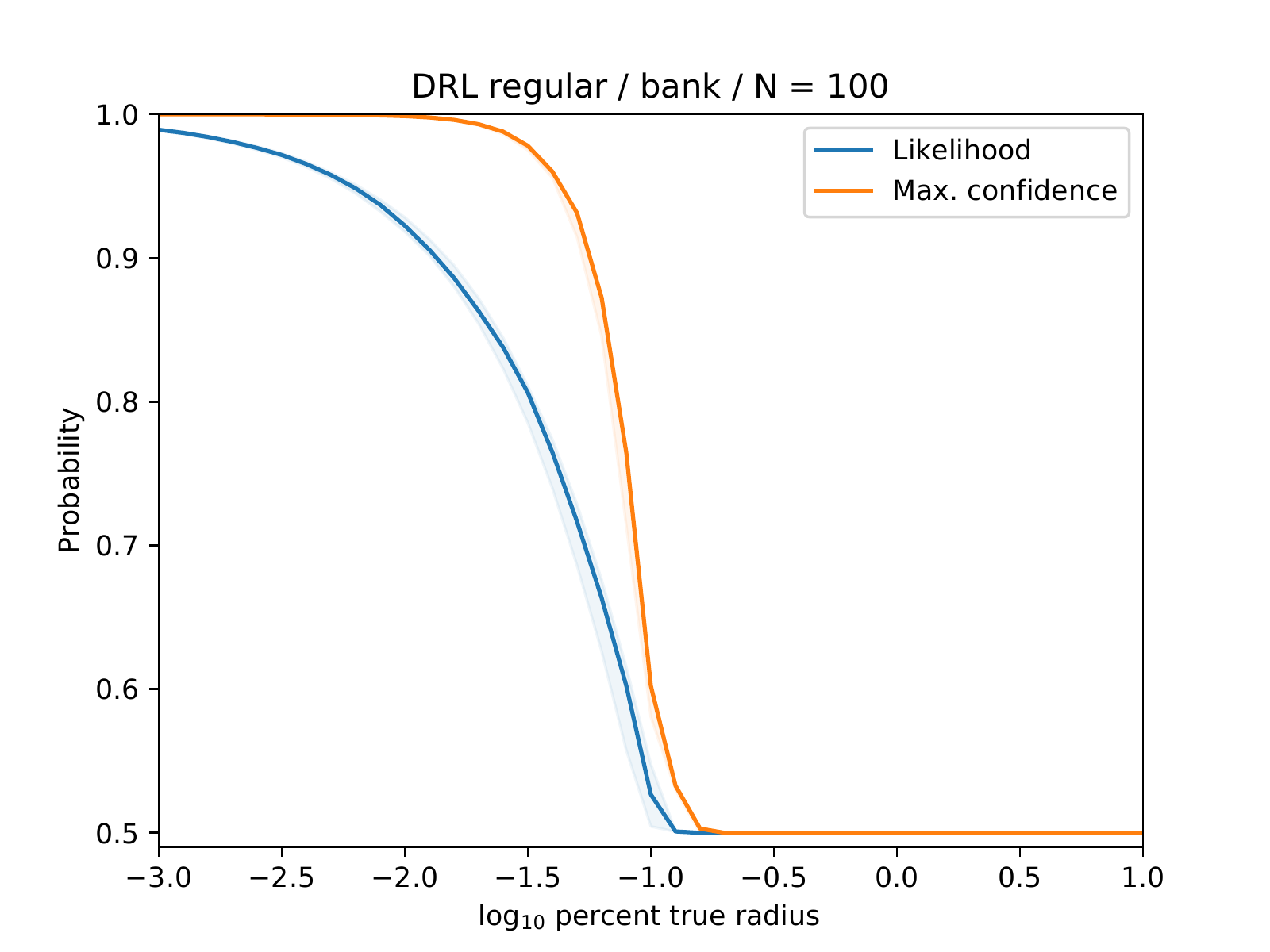}
\caption{Bank}
\end{subfigure}
\begin{subfigure}[b]{0.36\textwidth}
\includegraphics[width=\textwidth]{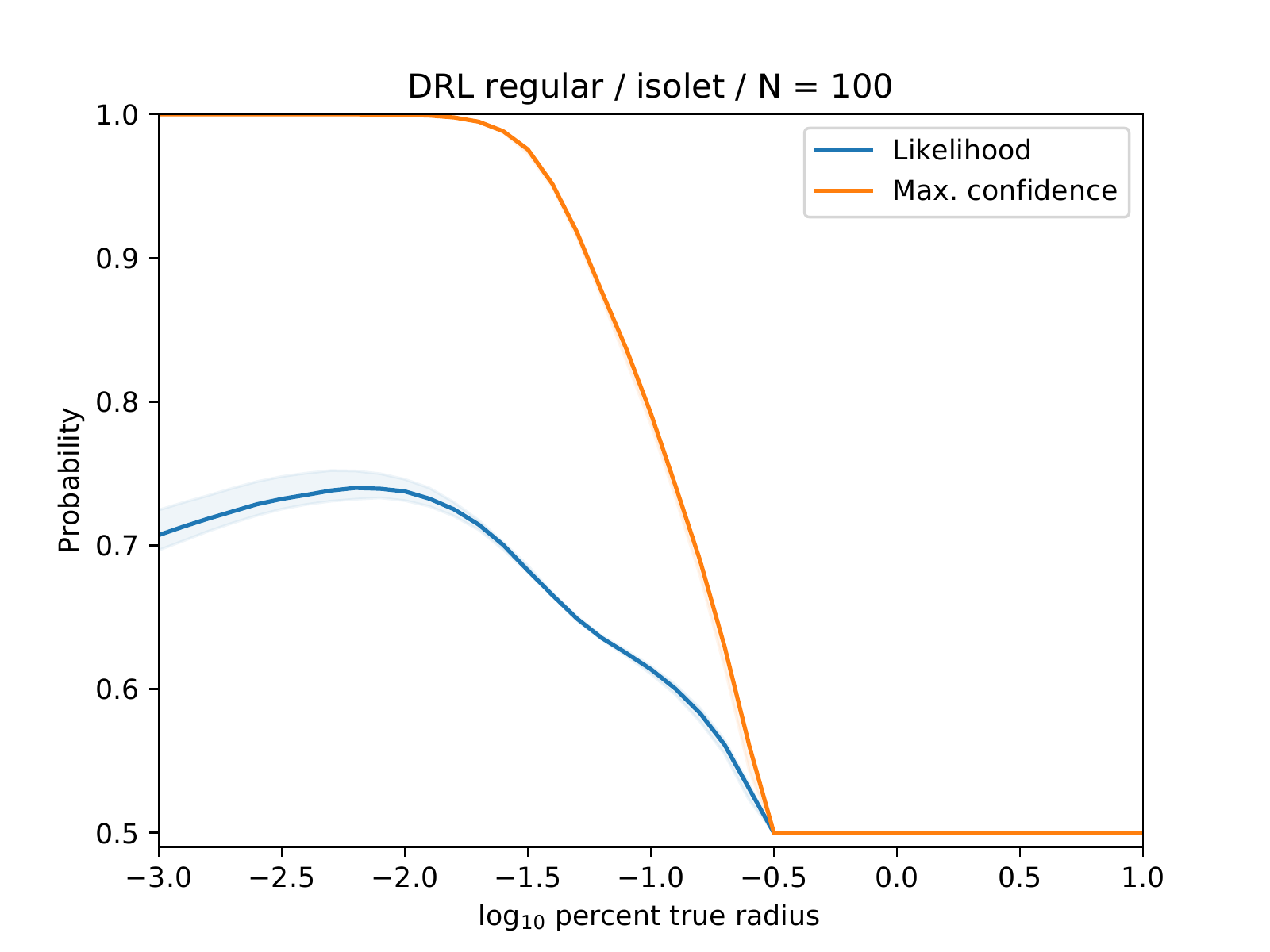}
\caption{Isolet}
\end{subfigure}
\begin{subfigure}[b]{0.36\textwidth}
\includegraphics[width=\textwidth]{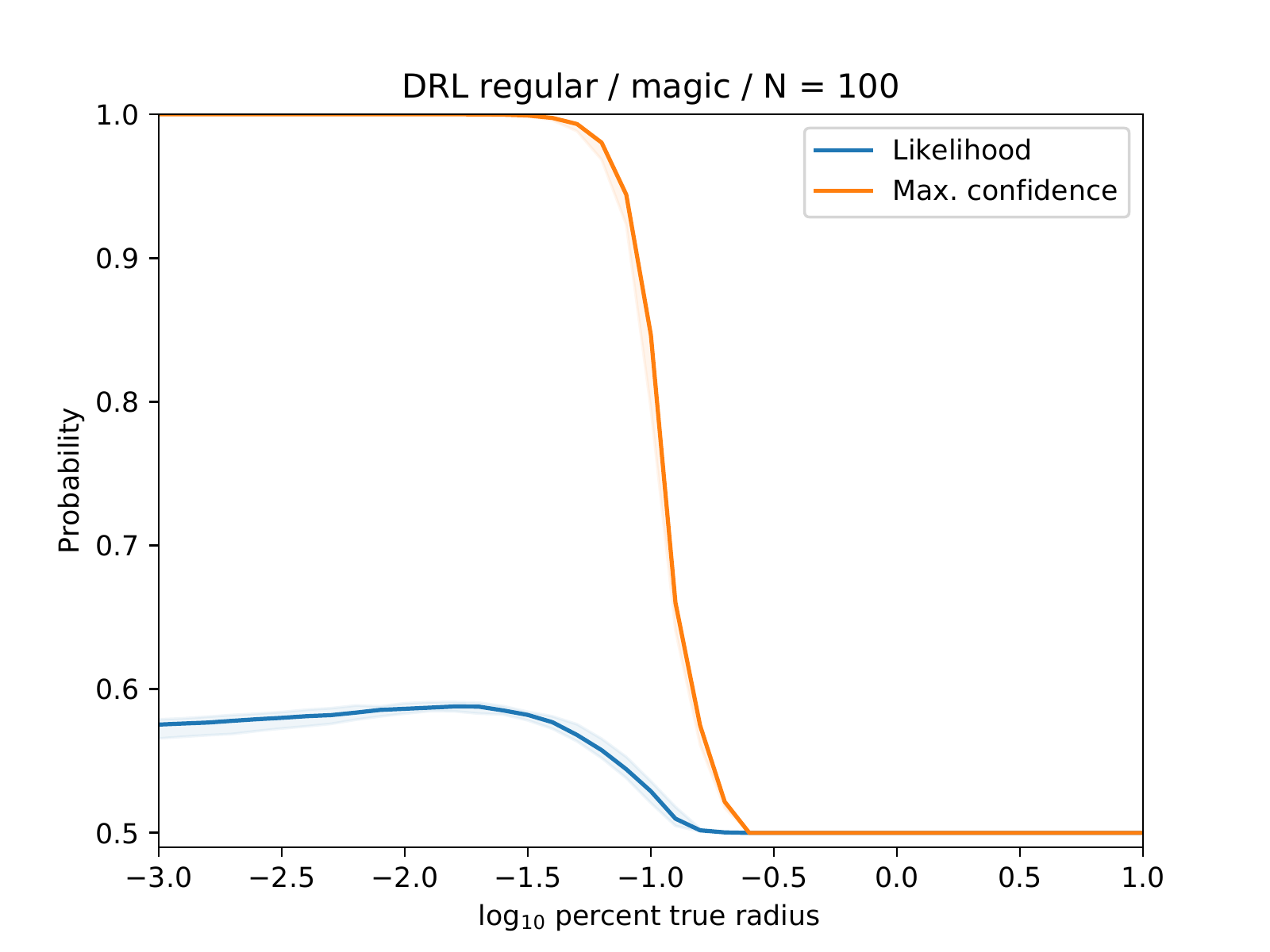}
\caption{Magic}
\end{subfigure}
\begin{subfigure}[b]{0.36\textwidth}
\includegraphics[width=\textwidth]{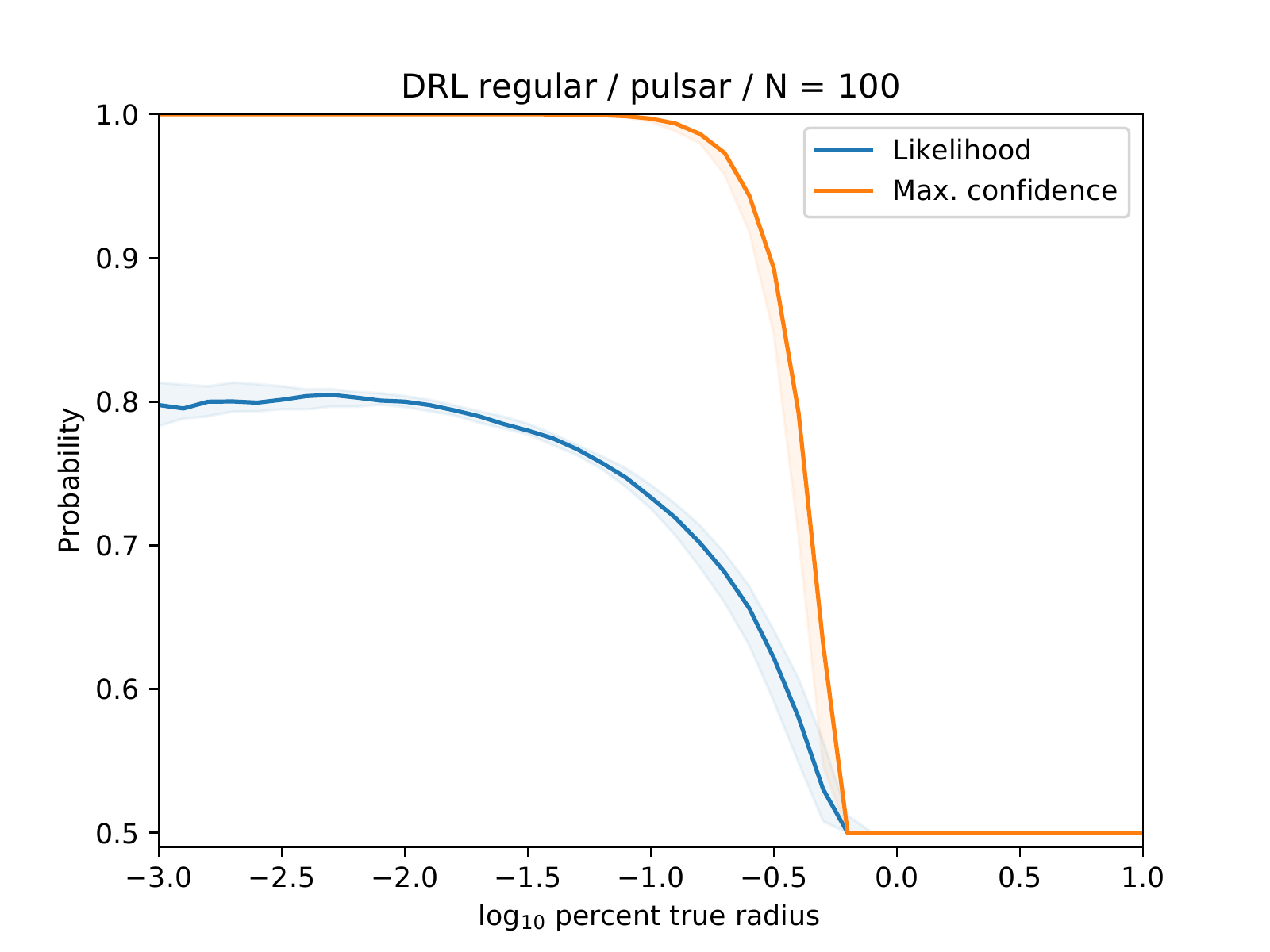}
\caption{Pulsar}
\end{subfigure}
\begin{subfigure}[b]{0.36\textwidth}
\includegraphics[width=\textwidth]{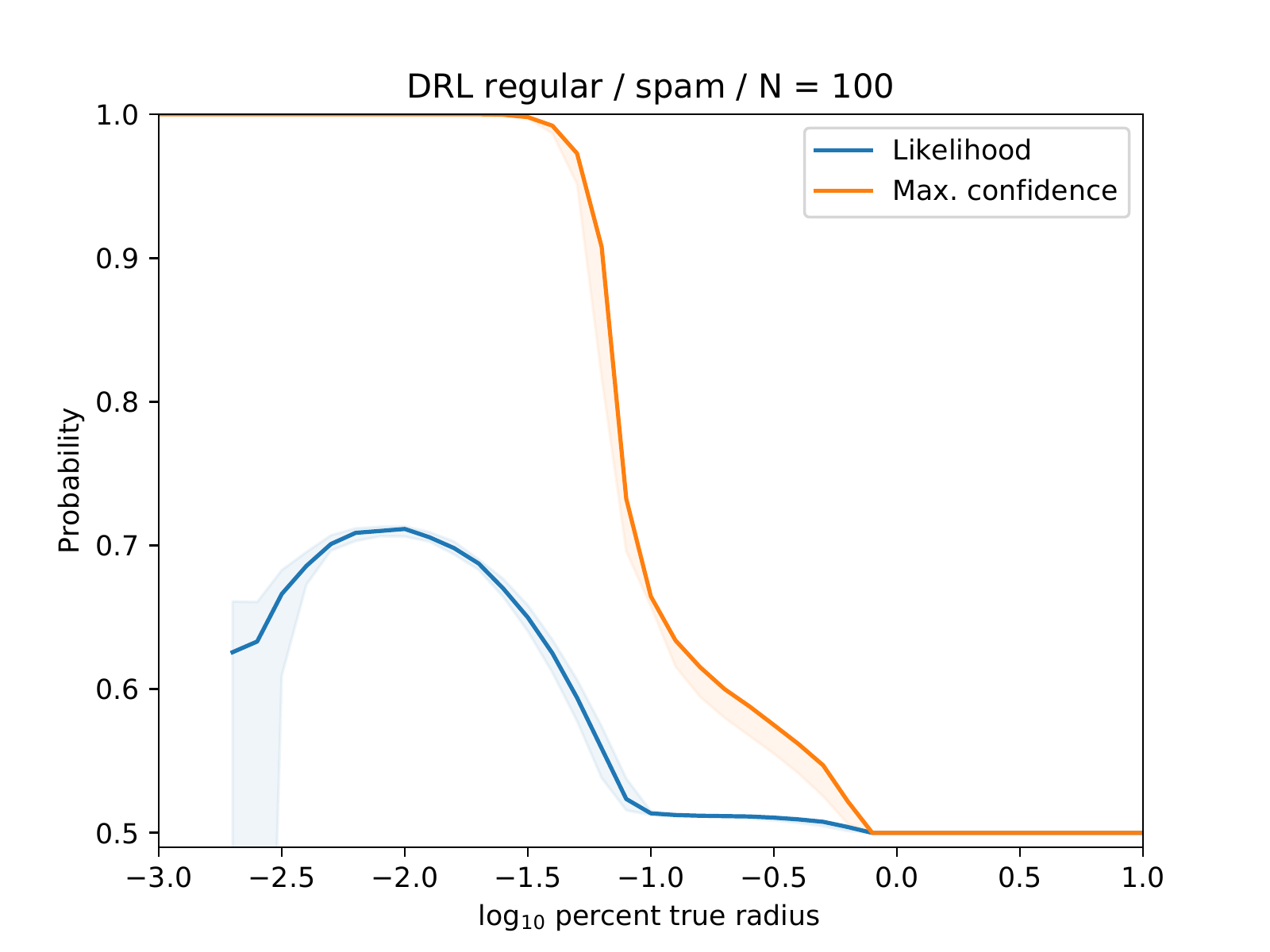}
\caption{Spam}
\end{subfigure}
\begin{subfigure}[b]{0.36\textwidth}
\includegraphics[width=\textwidth]{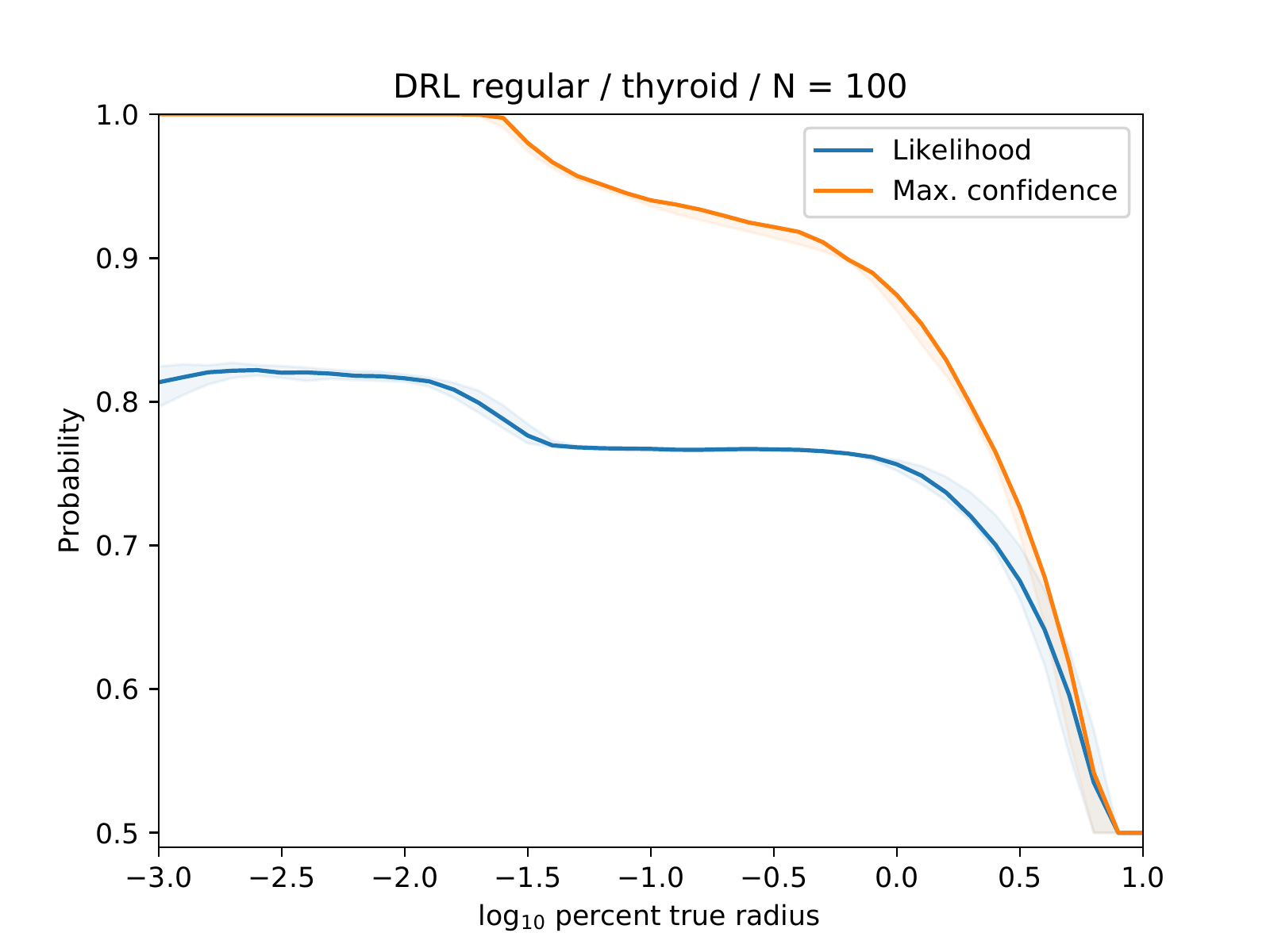}
\caption{Thyroid}
\end{subfigure}
\begin{subfigure}[b]{0.36\textwidth}
\includegraphics[width=\textwidth]{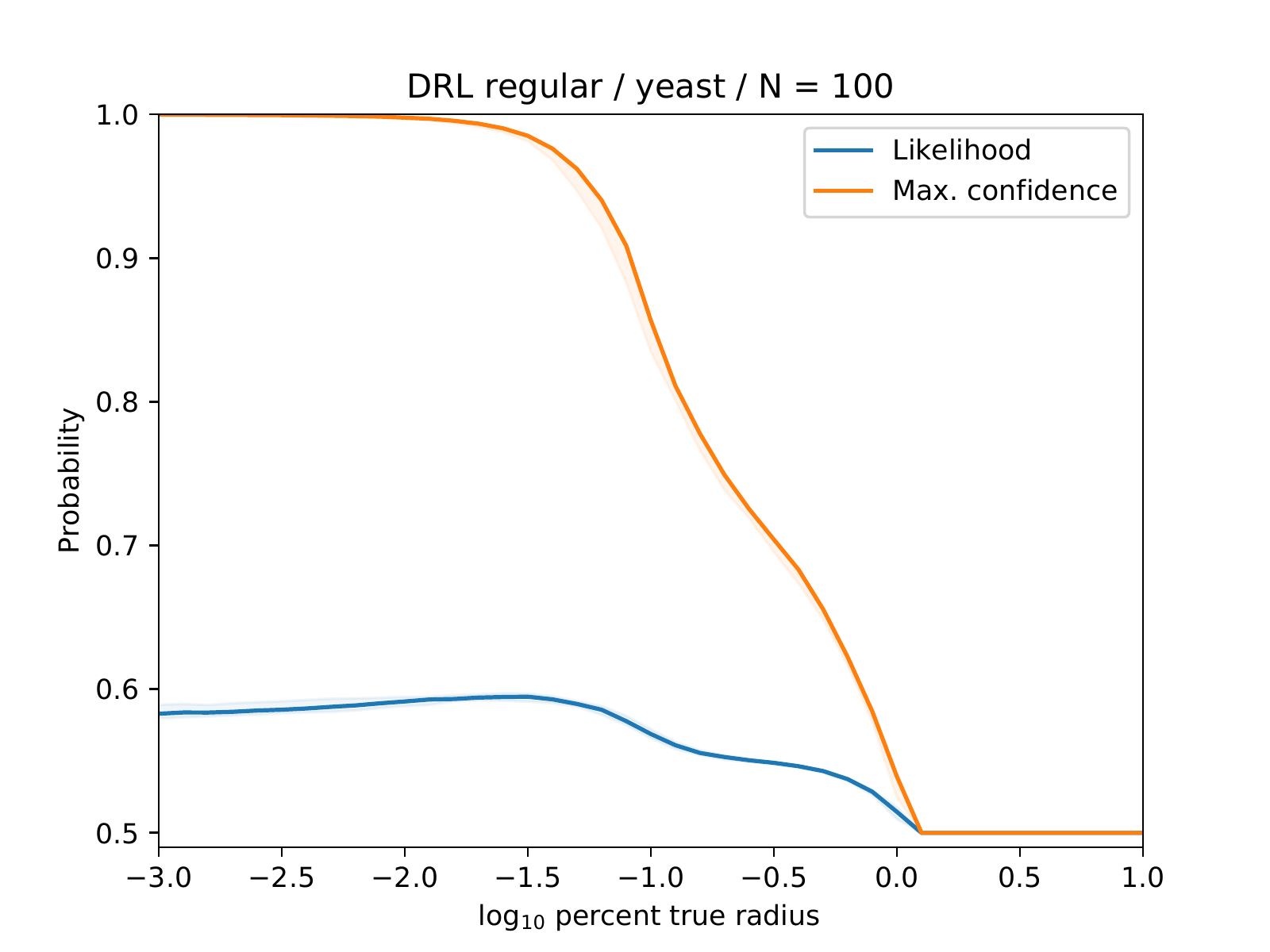}
\caption{Yeast}
\end{subfigure}
\caption{Traditional Wasserstein DRL. Out-of-sample performance (likelihood) and maximum confidence vs. radius of robustness $\varepsilon$ as a percentage of the distance to the true data distribution $\Ppr$. Performance shows a bias-variance tradeoff with peak at $\varepsilon$ much smaller than the distance to $\Ppr$. Confidence often drops sharply at a radius much smaller than the distance to $\Ppr$.}
\label{fig:appendix-likelihood-and-confidence-vs-radius}
\end{figure}

\subsection{Traditional Wasserstein DRL robustness beyond the decision set}
\label{sec:appendix-beyond-decision-set}

The following describes Figure \ref{fig:worst-case-likelihood-vs-radius-plus-delta} in Section \ref{sec:how-important-is-radius} and \ref{fig:appendix-worst-case-likelihood-vs-radius-plus-delta} in the appendix. The figures show the worst-case performance of a predictor trained by traditional Wasserstein DRL with a radius of robustness $\varepsilon$, when the test-time data distribution is allowed to come from a Wasserstein ball having the same center and slightly larger radius $\varepsilon + \Delta$. Specifically, for each trial, we sample a training set of $N_l = 100$ labeled examples, which we use to compute the Wasserstein distributionally robust logistic regression \citep{abadeh2015distributionally}, fixing the radius $\varepsilon$ of the underlying Wasserstein ball to a value between $10^{-3}$ and $10^0$. For each radius $\varepsilon$, we obtain a set of learned parameters $\hat{\theta}_{\varepsilon} \in \Theta$. We then evaluate the worst-case value of the negative log-likelihood, fixing these parameters $\hat{\theta}_{\varepsilon}$, but increasing the radius of the underlying Wasserstein ball to $\varepsilon + \Delta$, for $\Delta \in [10^{-3}, 10^0]$. This worst-case value can be written
\begin{equation}
f(\hat{\theta}_{\varepsilon}, \Delta) = \sup_{\mu \in \ball_{\varepsilon+\Delta}(\hat{\Ppr}_l)} \expect^{\mu} \Yrv \log(1 + \exp\{-\langle \Xrv, \hat{\theta}_{\varepsilon} \rangle\}) + (1 - \Yrv) \log(1 + \exp\{\langle \Xrv, \hat{\theta}_{\varepsilon} \rangle\}),
\end{equation}
with $\ball_{\varepsilon + \Delta}(\hat{\Ppr}_l)$ the Wasserstein ball of radius $\varepsilon + \Delta$ centered at the empirical distribution of the labeled data $\hat{\Ppr}_l$. This is exactly the inner problem of traditional Wasserstein DRL and is solved by the same mechanism, fixing the parameters $\hat{\theta}_{\varepsilon}$.

Each figure shows the median over trials of the resulting worst-case likelihood value, $\exp(-f(\hat{\theta}_{\varepsilon}, \Delta))$, as we vary $\varepsilon$ (vertical axis) and $\Delta$ (horizontal axis). Each axis shows the base-$10$ log of the respective value. The color encodes the likelihood value, with blue indicating value $0$, green value $0.5$, and yellow value $1$.

\begin{figure}
\centering
\begin{subfigure}[b]{0.36\textwidth}
\includegraphics[width=\textwidth]{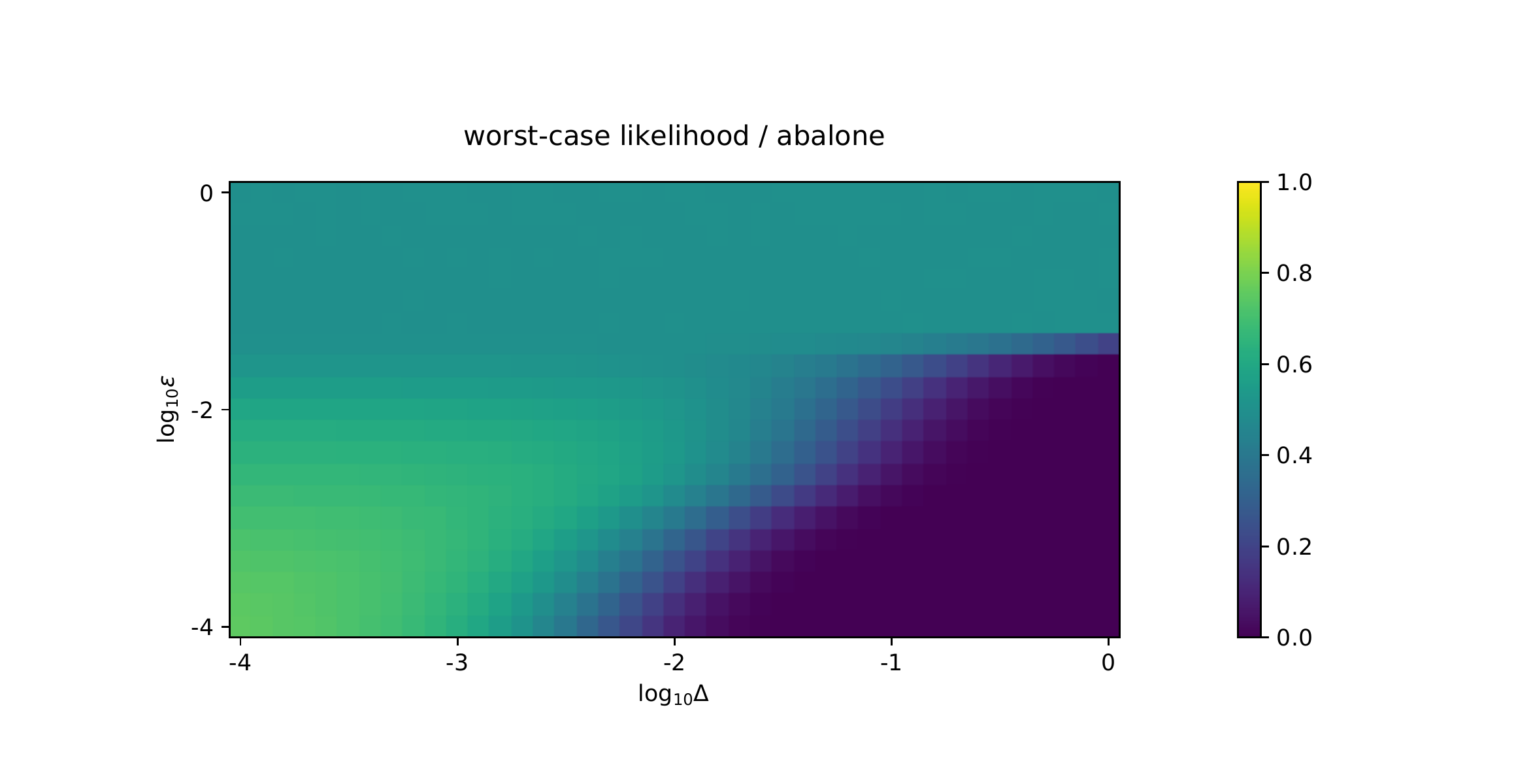}
\caption{Abalone}
\end{subfigure}
\begin{subfigure}[b]{0.36\textwidth}
\includegraphics[width=\textwidth]{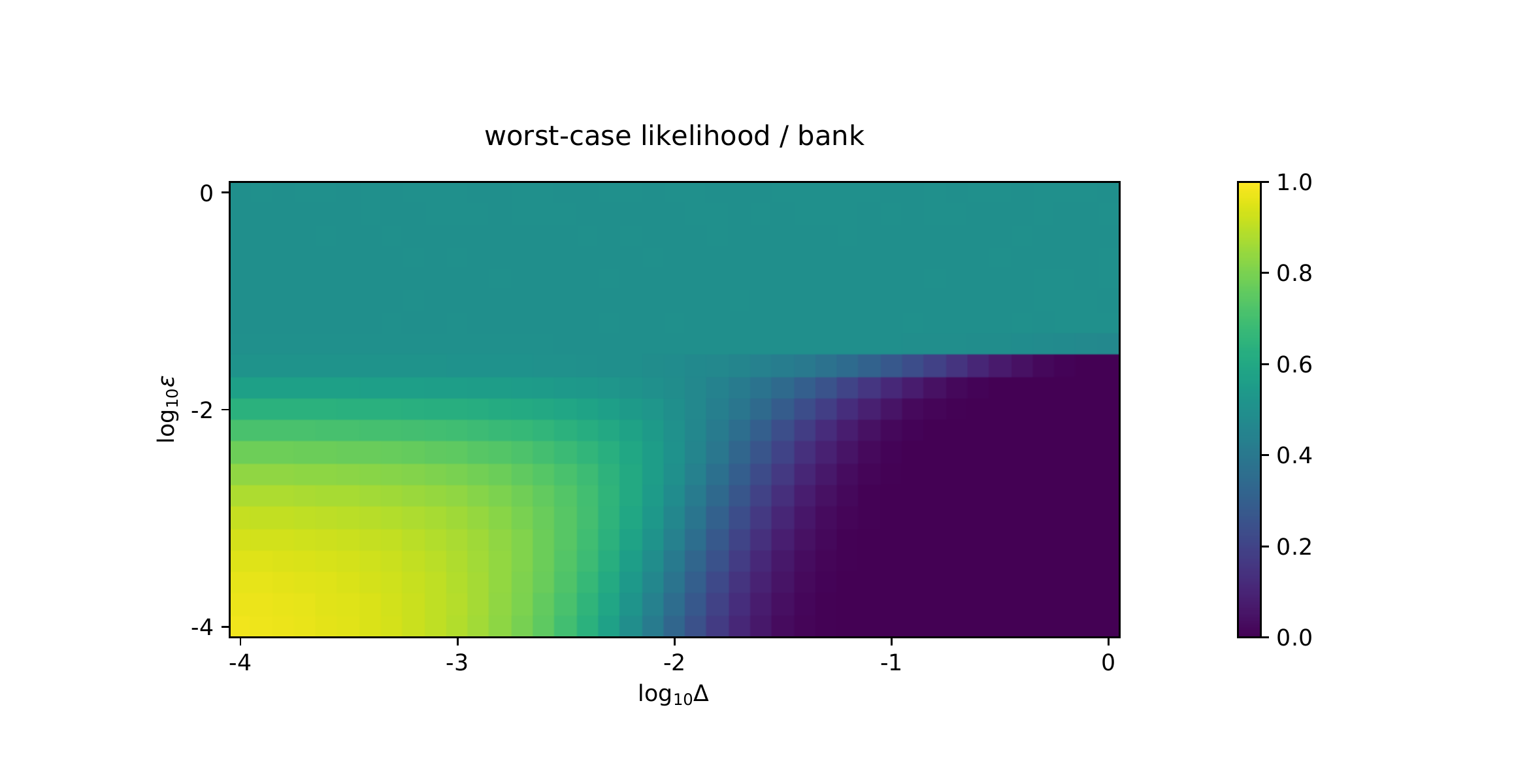}
\caption{Bank}
\end{subfigure}
\begin{subfigure}[b]{0.36\textwidth}
\includegraphics[width=\textwidth]{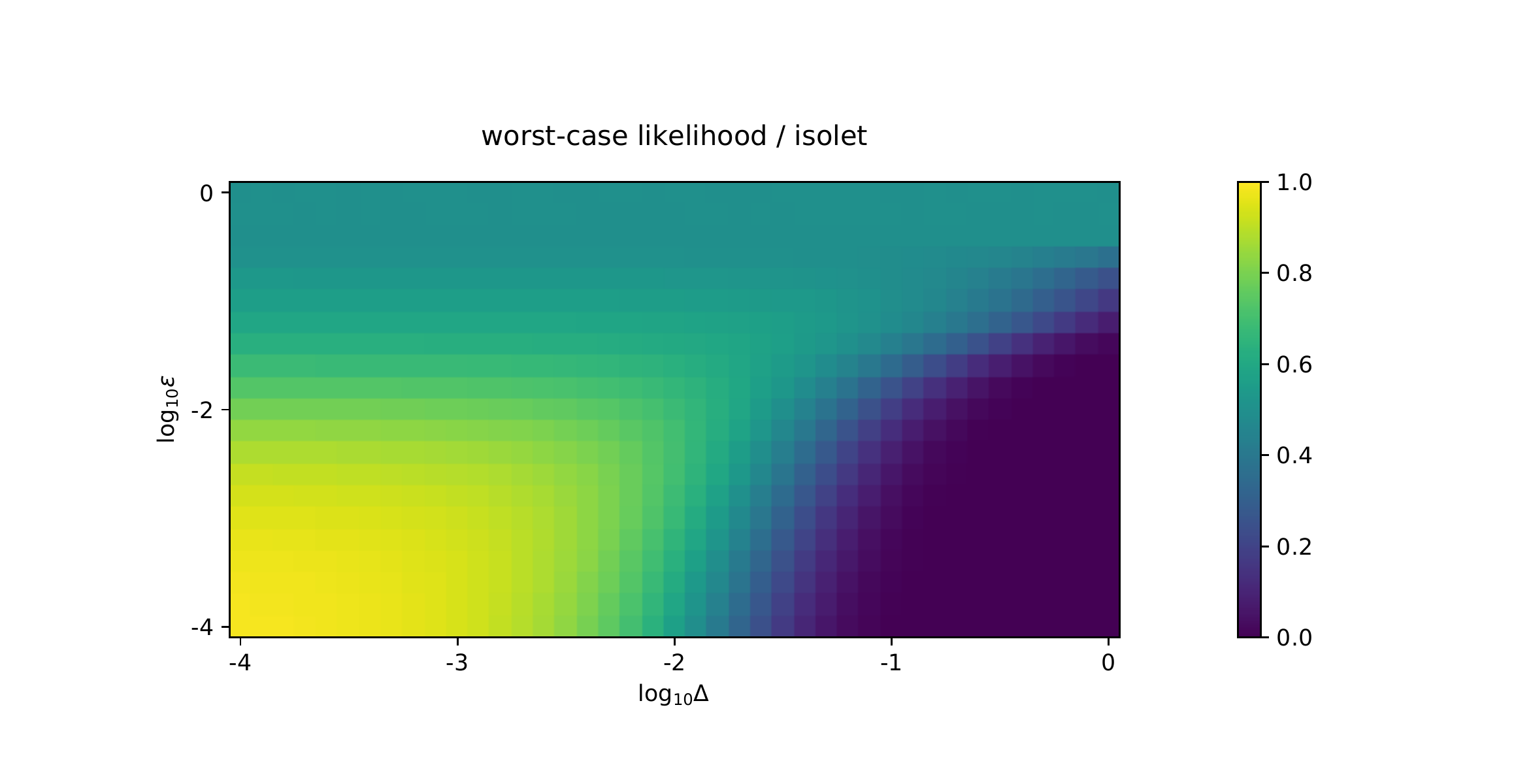}
\caption{Isolet}
\end{subfigure}
\begin{subfigure}[b]{0.36\textwidth}
\includegraphics[width=\textwidth]{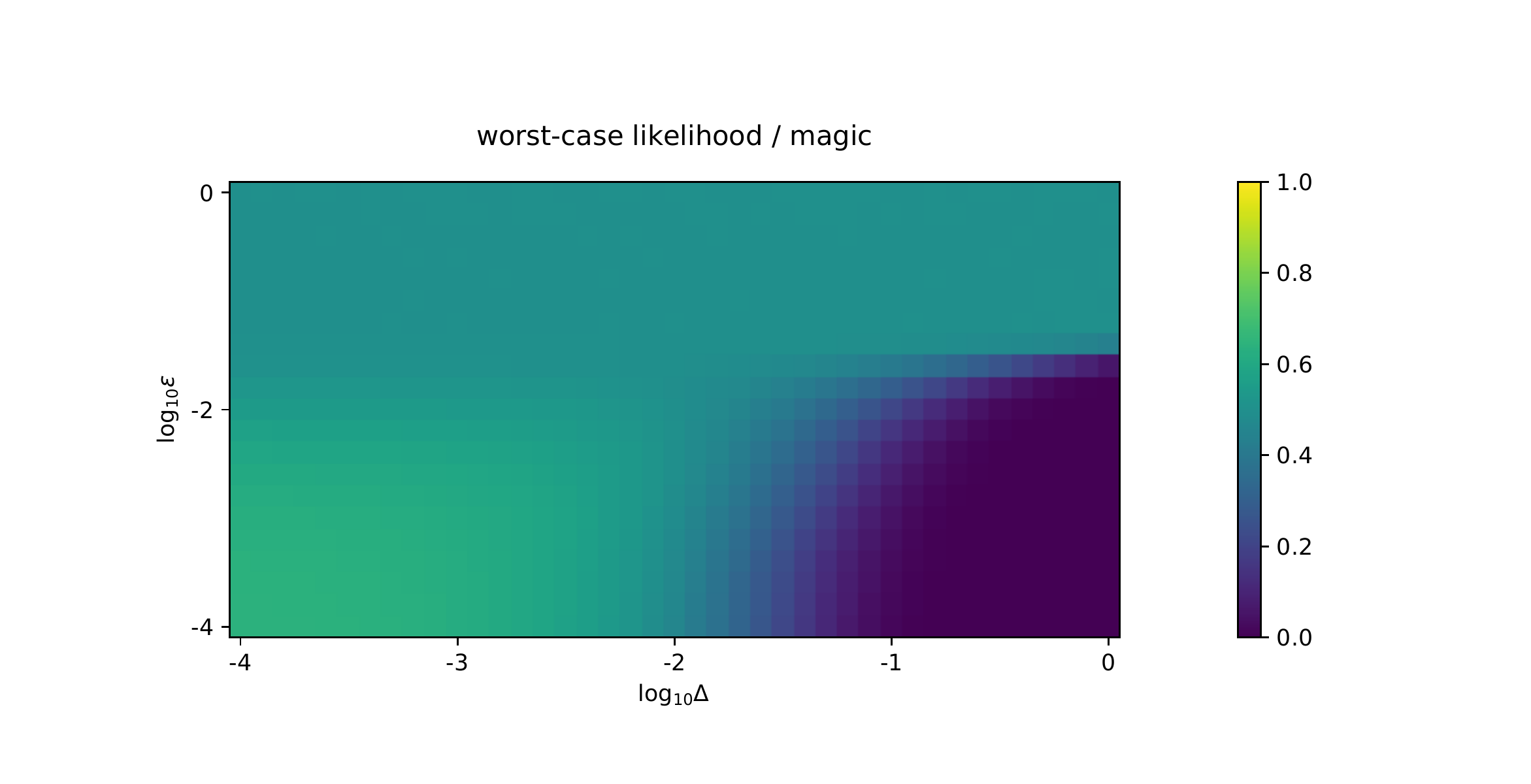}
\caption{Magic}
\end{subfigure}
\begin{subfigure}[b]{0.36\textwidth}
\includegraphics[width=\textwidth]{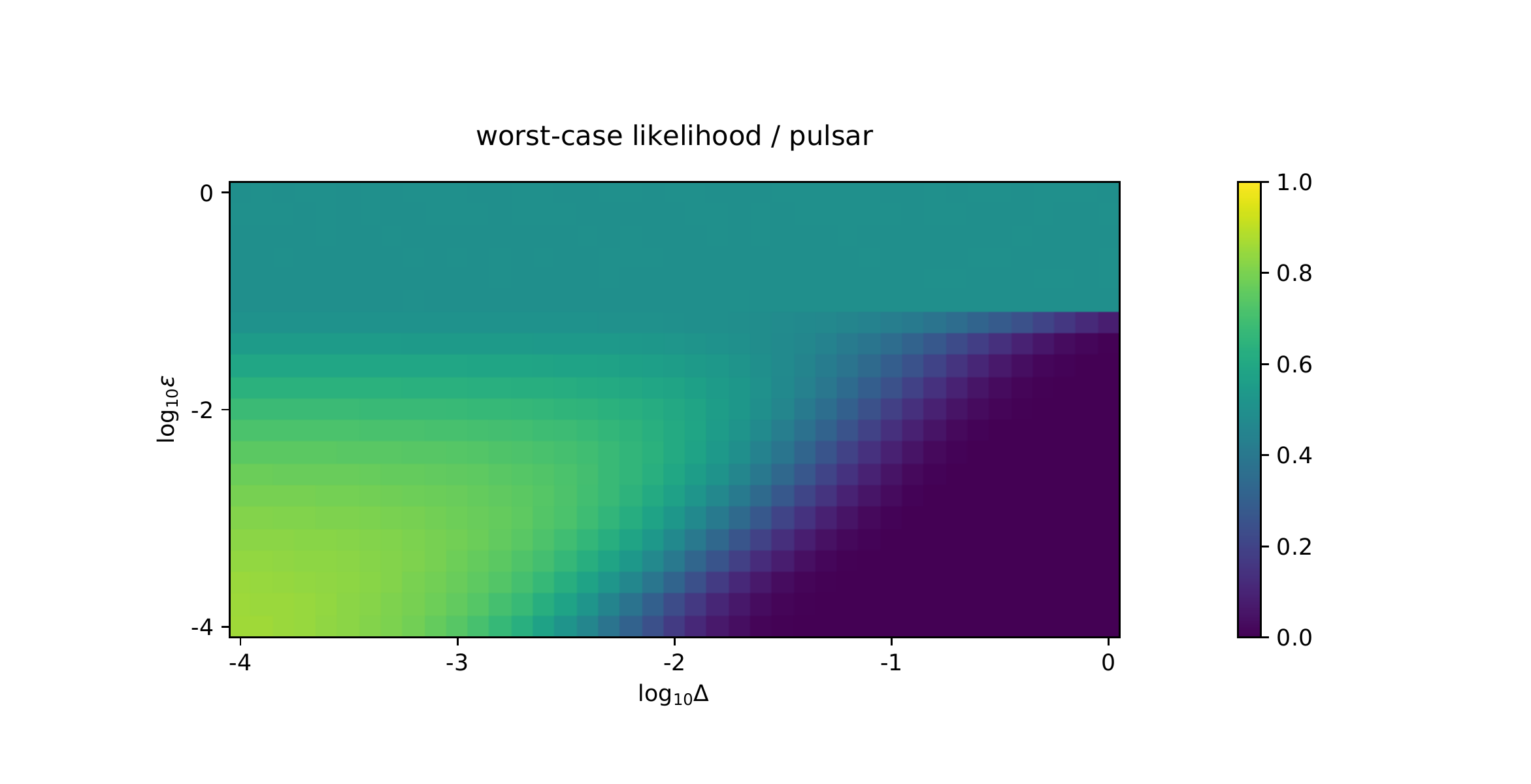}
\caption{Pulsar}
\end{subfigure}
\begin{subfigure}[b]{0.36\textwidth}
\includegraphics[width=\textwidth]{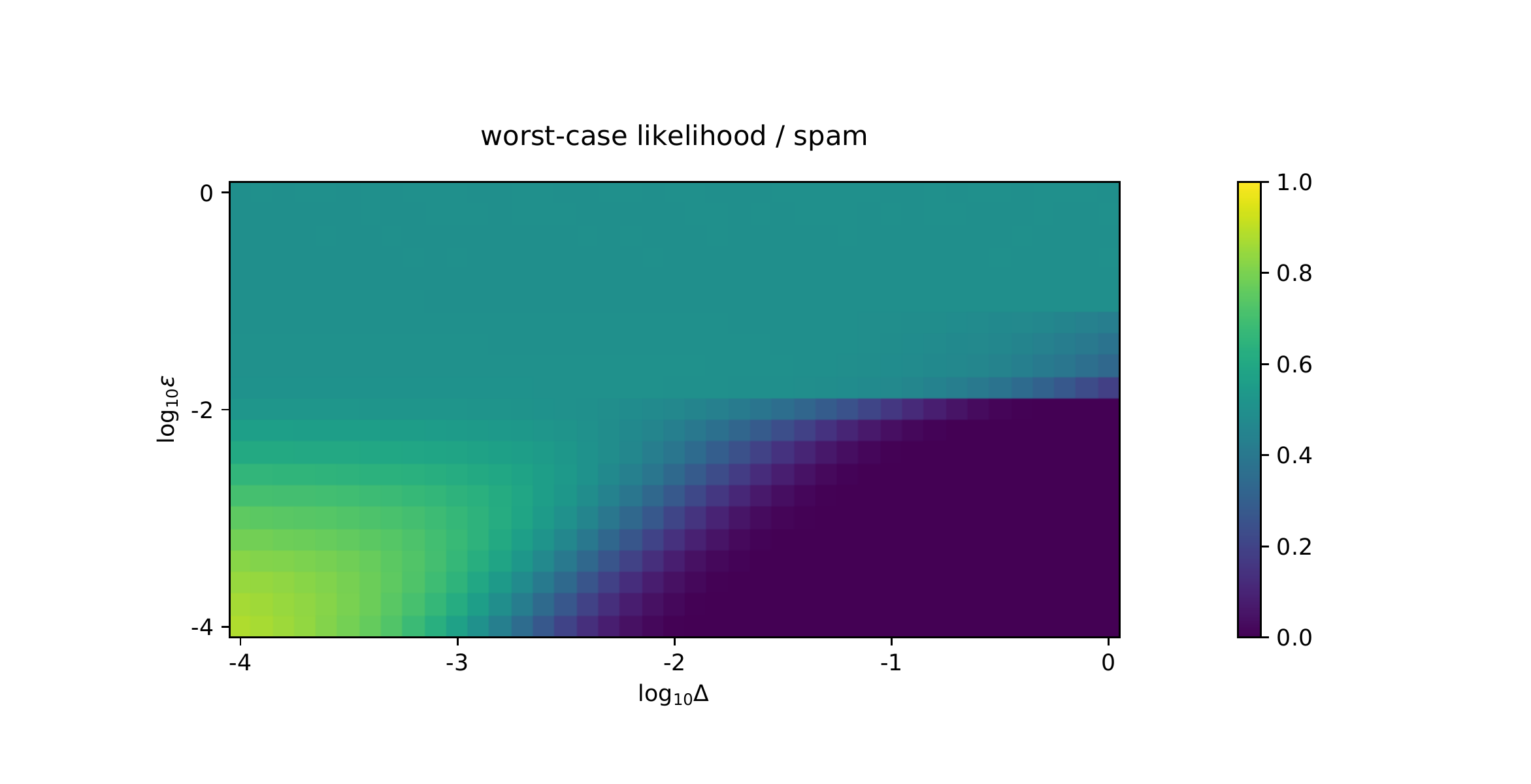}
\caption{Spam}
\end{subfigure}
\begin{subfigure}[b]{0.36\textwidth}
\includegraphics[width=\textwidth]{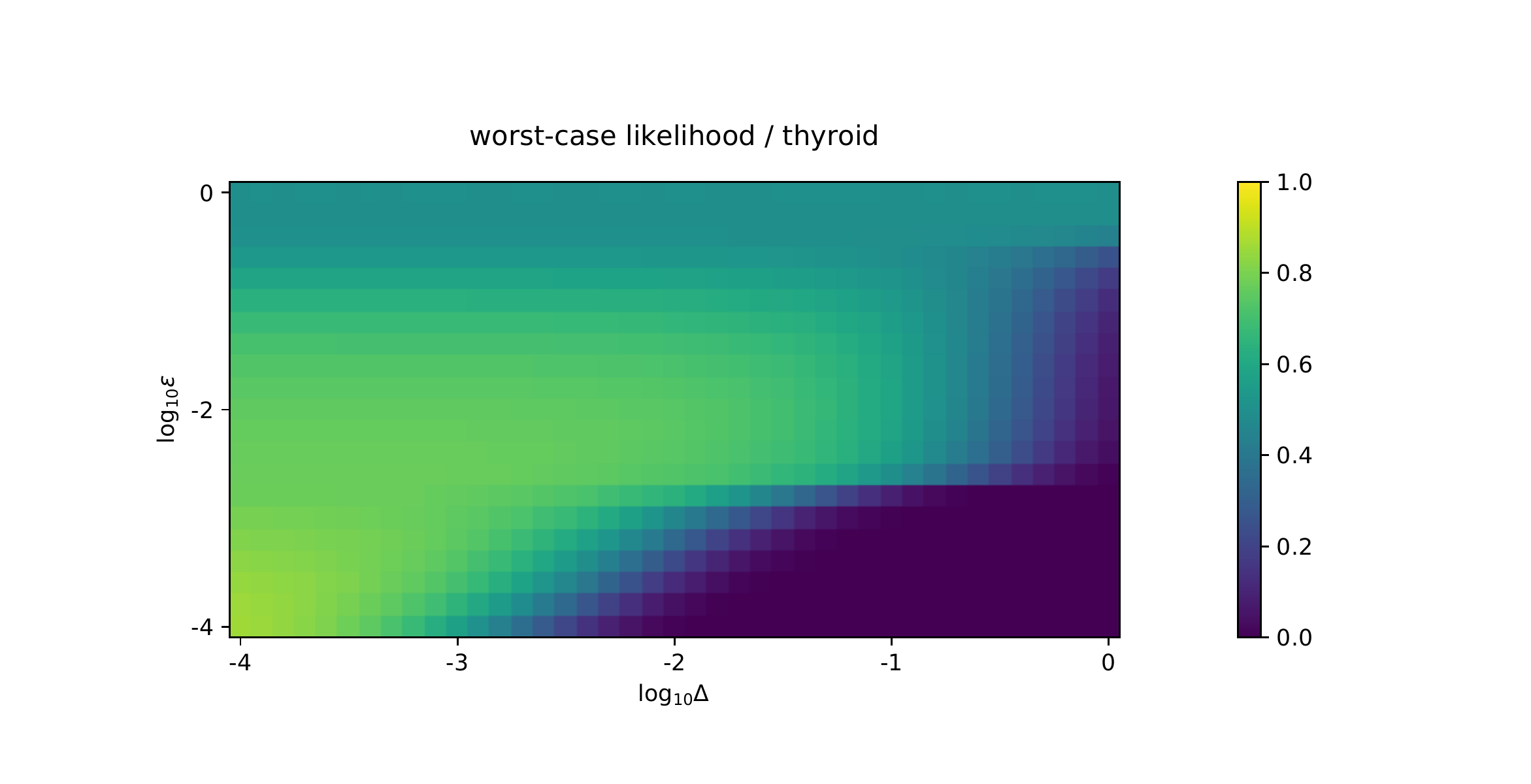}
\caption{Thyroid}
\end{subfigure}
\begin{subfigure}[b]{0.36\textwidth}
\includegraphics[width=\textwidth]{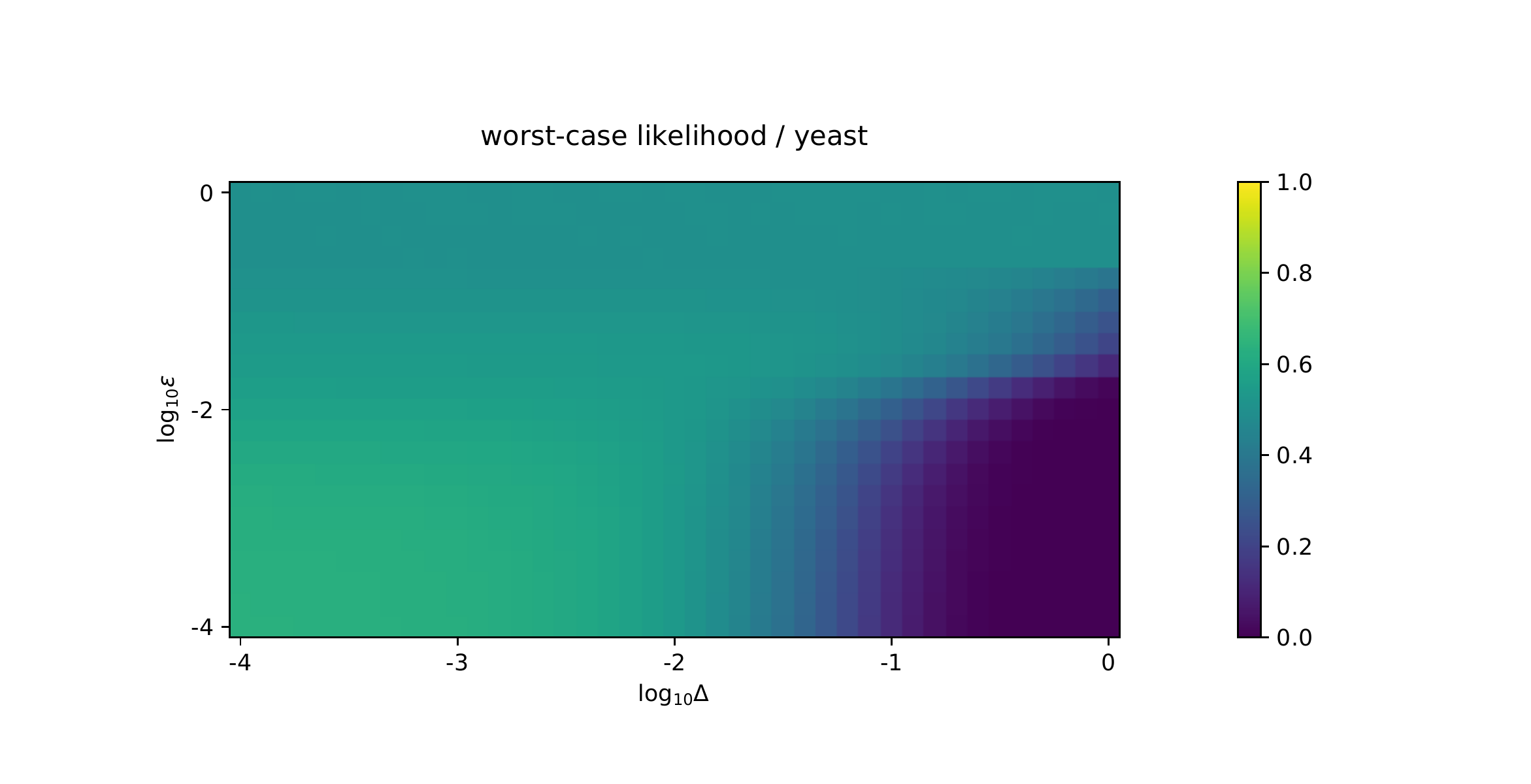}
\caption{Yeast}
\end{subfigure}
\caption{Traditional Wasserstein DRL. Worst-case performance (likelihood) vs. radius of robustness $\varepsilon$ and test-time data radius $\varepsilon + \Delta$. Yellow indicates perfectly correct prediction (likelihood $1$), blue perfectly incorrect (likelihood $0$), and green perfectly indecisive prediction (likelihood $0.5$). Training with radius $\varepsilon$ confers little robustness beyond $\varepsilon$.}% \justin{explain the color}}
\label{fig:appendix-worst-case-likelihood-vs-radius-plus-delta}
\end{figure}

\subsection{DRL with unlabeled data, lack of bias-variance tradeoff}
\label{sec:appendix-no-bias-variance}

The following describes Figure \ref{fig:likelihood-and-confidence-vs-additional-radius} in Section \ref{sec:how-important-is-radius} as well as Figure \ref{fig:appendix-likelihood-and-confidence-vs-additional-radius} in the appendix. The figures show the out-of-sample performance of linear logistic regression models learned by the proposed DRL method (Section \ref{sec:drl-with-unlabeled-data}). Specifically, for each trial we sample $100$ examples uniformly from the full dataset, to form the training set $\hat{\Zspace}_l$. We then find a minimal radius $\varepsilon_0$ such that the feasible set $\Pset = \ball_{\varepsilon_0}(\hat{\Ppr}_l) \cap \Uset(\Ppr_{\Xspace}, \overline{\pvec}_{\Yspace}, \underline{\pvec}_{\Yspace})$ is non-empty, by doing a binary search for the minimal radius for which the value of the objective $g(\theta)$ (Section \ref{sec:drl-problem-formulation}) is nonnegative. Here, $\Ppr_{\Xspace}$ is the $\Xspace$-marginal of the true data distribution $\Ppr$, which is taken to be the empirical distribution of the full dataset. Given this radius $\varepsilon_0$, then we select radius $\varepsilon = \varepsilon_0 + \Delta$, for $\Delta \in [10^{-4}, 10^1]$, and solve the DRL problem
\begin{equation}
\minimize_{\theta \in \Theta} \sup_{\mu \in \ball_{\varepsilon}(\hat{\Ppr}_l)} \expect^{\mu} \Yrv \log(1 + \exp\{-\langle \Xrv, \theta \rangle\}) + (1 - \Yrv) \log(1 + \exp\{\langle \Xrv, \theta \rangle\}),
\end{equation}
with $\Xspace = \reals^q \times \{1\}$ the feature space and $\Yspace = \{0,1\}$ the label space. Here, and when finding $\varepsilon_0$, we choose $\overline{\pvec}_{\Yspace} = \underline{\pvec}_{\Yspace} = \pvec_{\Yspace}$, the true label probabilities from $\Ppr$. This DRL problem is solved as described in Appendix \ref{sec:appendix-true-radius}.

The solid lines in Figure \ref{fig:likelihood-and-confidence-vs-additional-radius} in Section \ref{sec:how-important-is-radius} and Figure \ref{fig:appendix-likelihood-and-confidence-vs-additional-radius} in the appendix show the median over $100$ trials of likelihood on the test set for the learned model as well as the median confidence, defined as in Appendix \ref{sec:appendix-true-radius}. The shaded regions are $95\%$ confidence intervals for the median. The horizontal axis shows the base-$10$ log of the excess radius $\Delta$ beyond $\varepsilon_0$.

\begin{figure}
\centering
\begin{subfigure}[b]{0.36\textwidth}
\includegraphics[width=\textwidth]{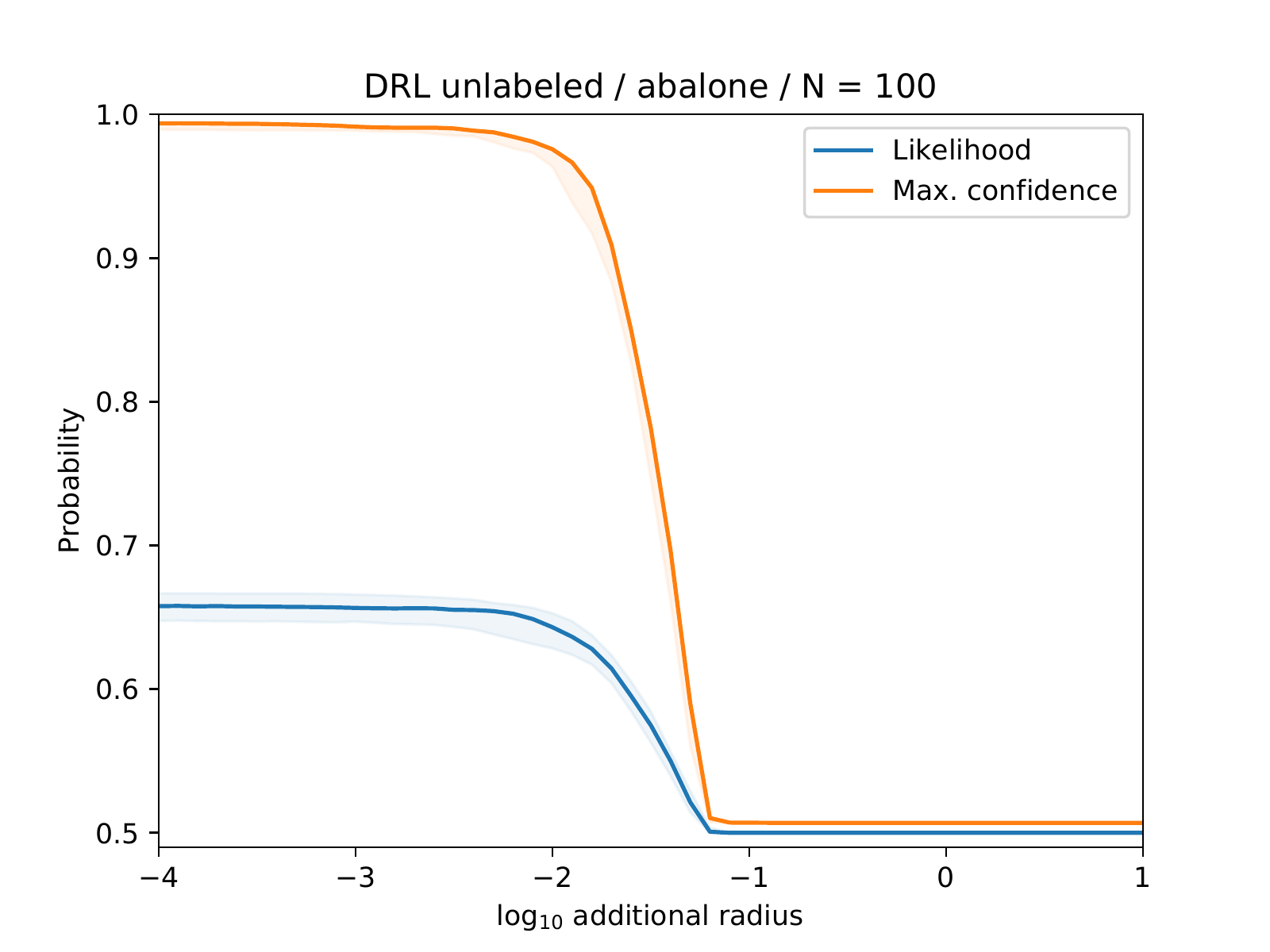}
\caption{Abalone}
\end{subfigure}
\begin{subfigure}[b]{0.36\textwidth}
\includegraphics[width=\textwidth]{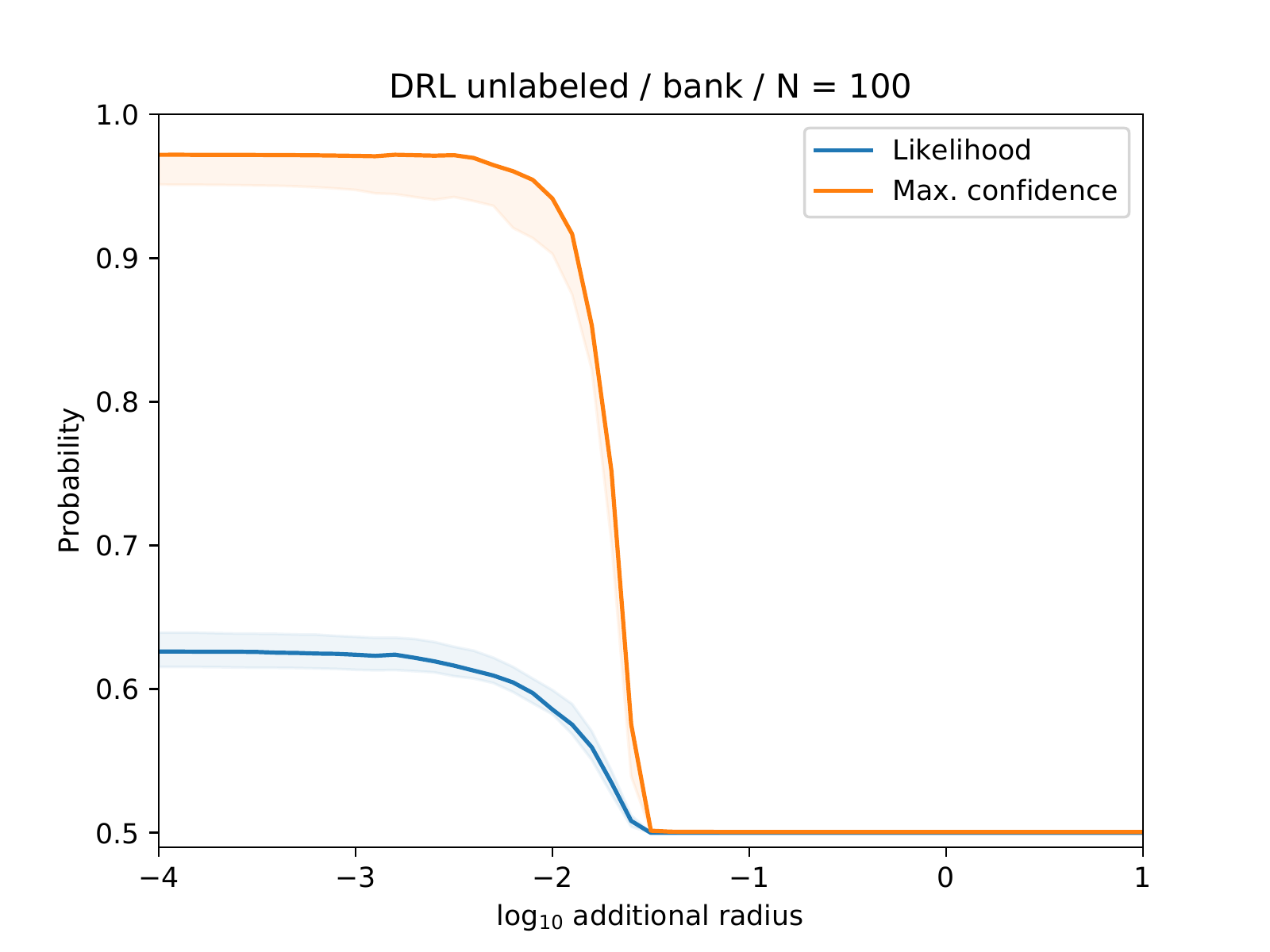}
\caption{Bank}
\end{subfigure}
\begin{subfigure}[b]{0.36\textwidth}
\includegraphics[width=\textwidth]{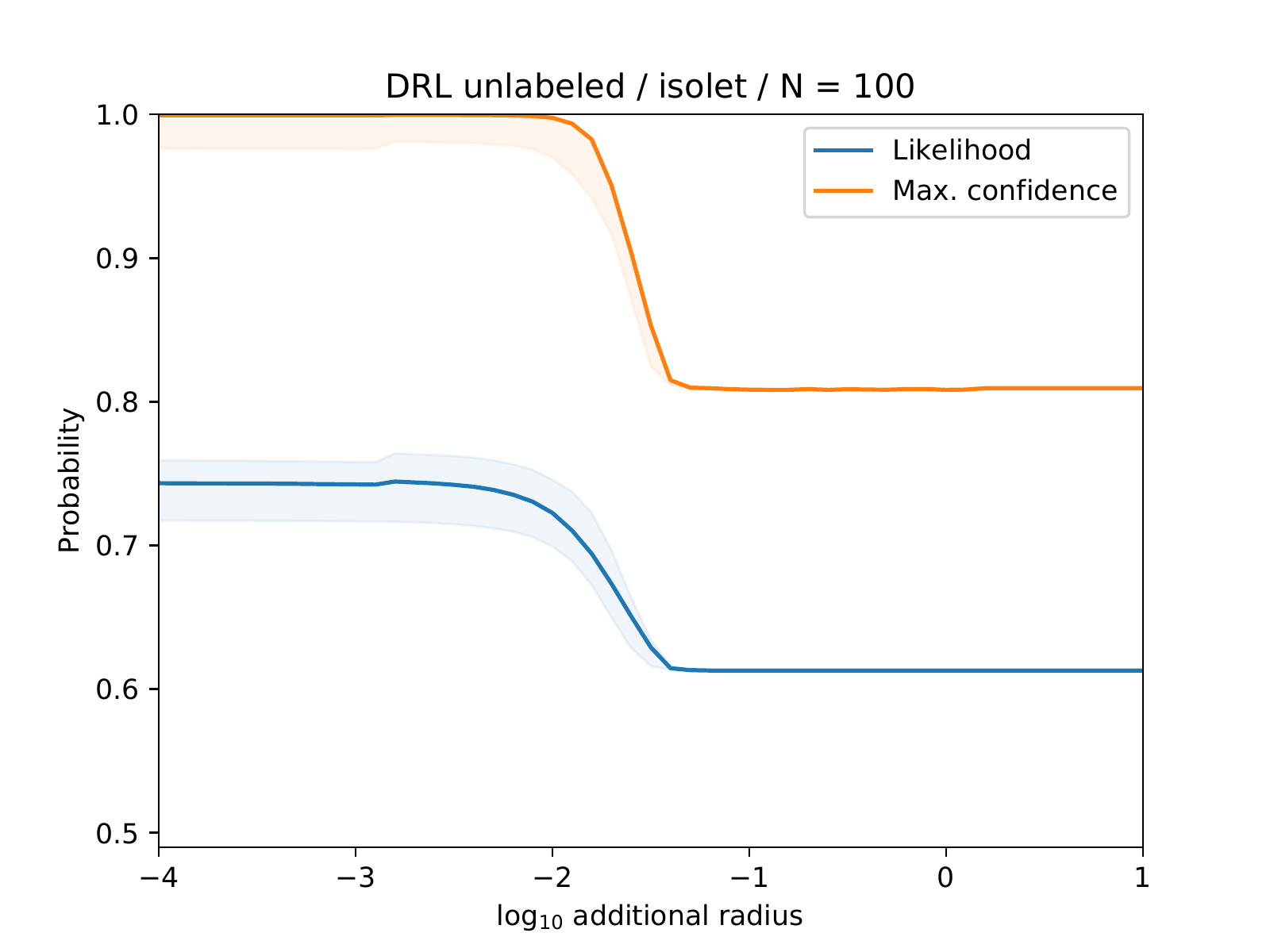}
\caption{Isolet}
\end{subfigure}
\begin{subfigure}[b]{0.36\textwidth}
\includegraphics[width=\textwidth]{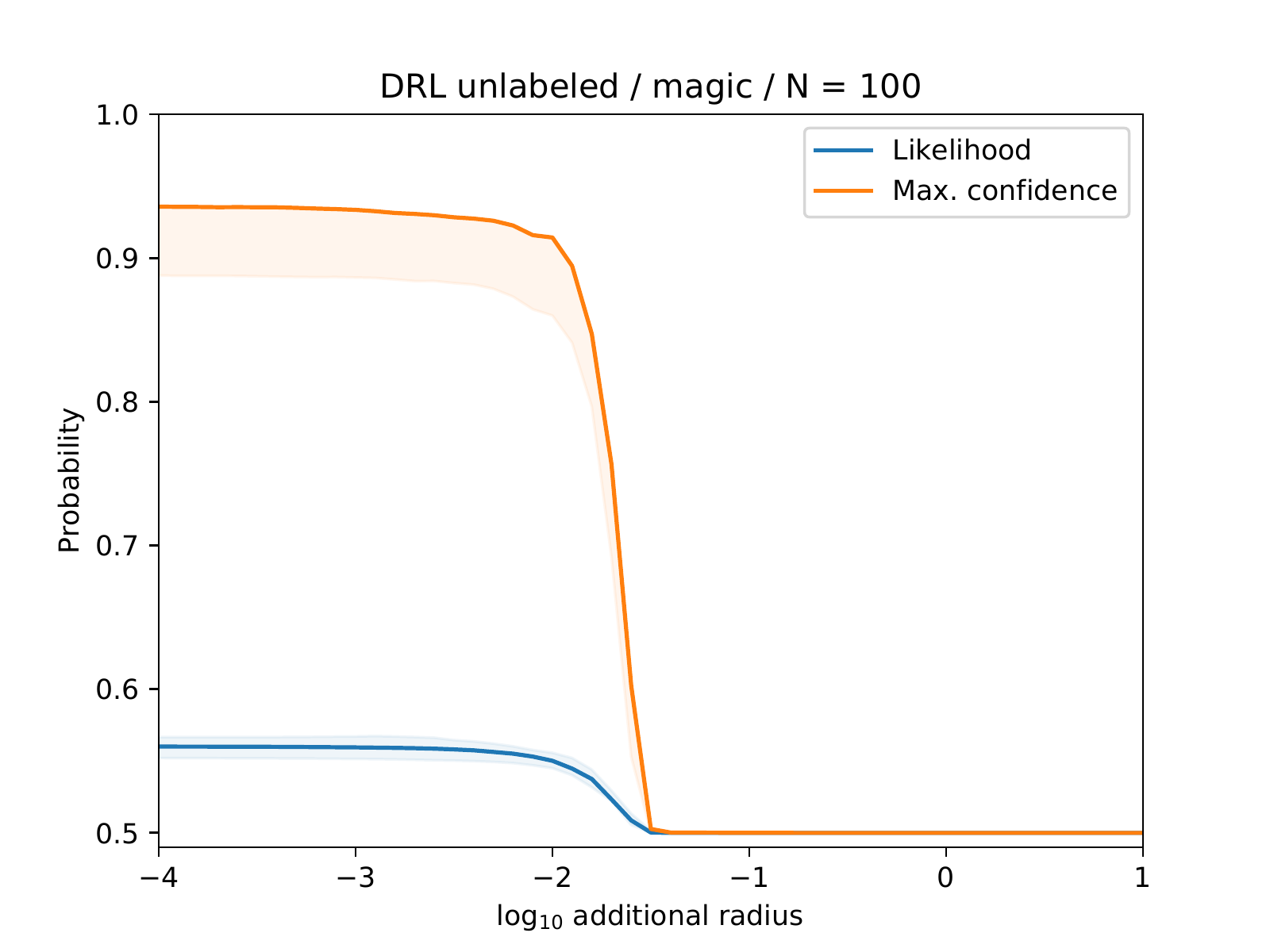}
\caption{Magic}
\end{subfigure}
\begin{subfigure}[b]{0.36\textwidth}
\includegraphics[width=\textwidth]{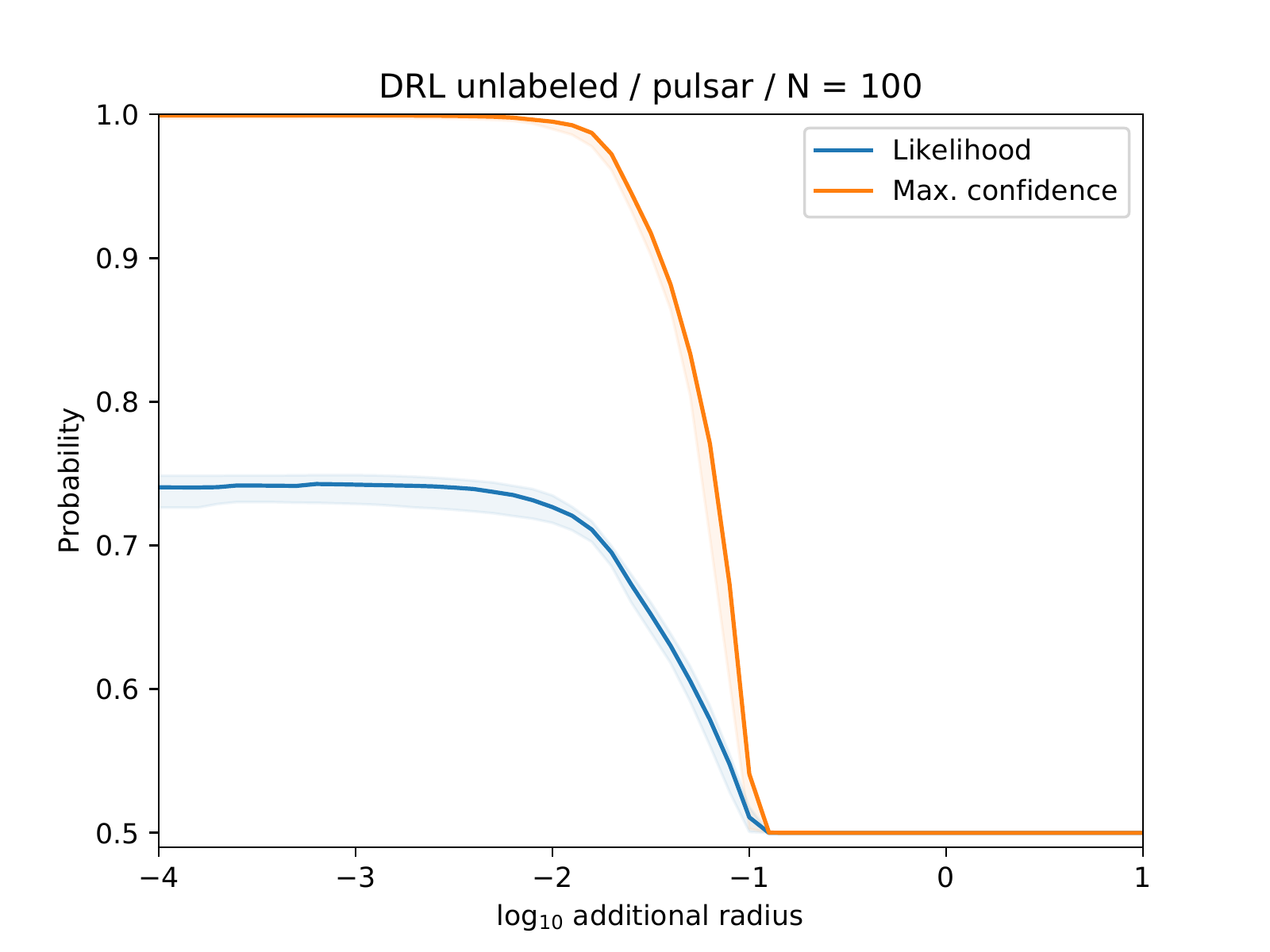}
\caption{Pulsar}
\end{subfigure}
\begin{subfigure}[b]{0.36\textwidth}
\includegraphics[width=\textwidth]{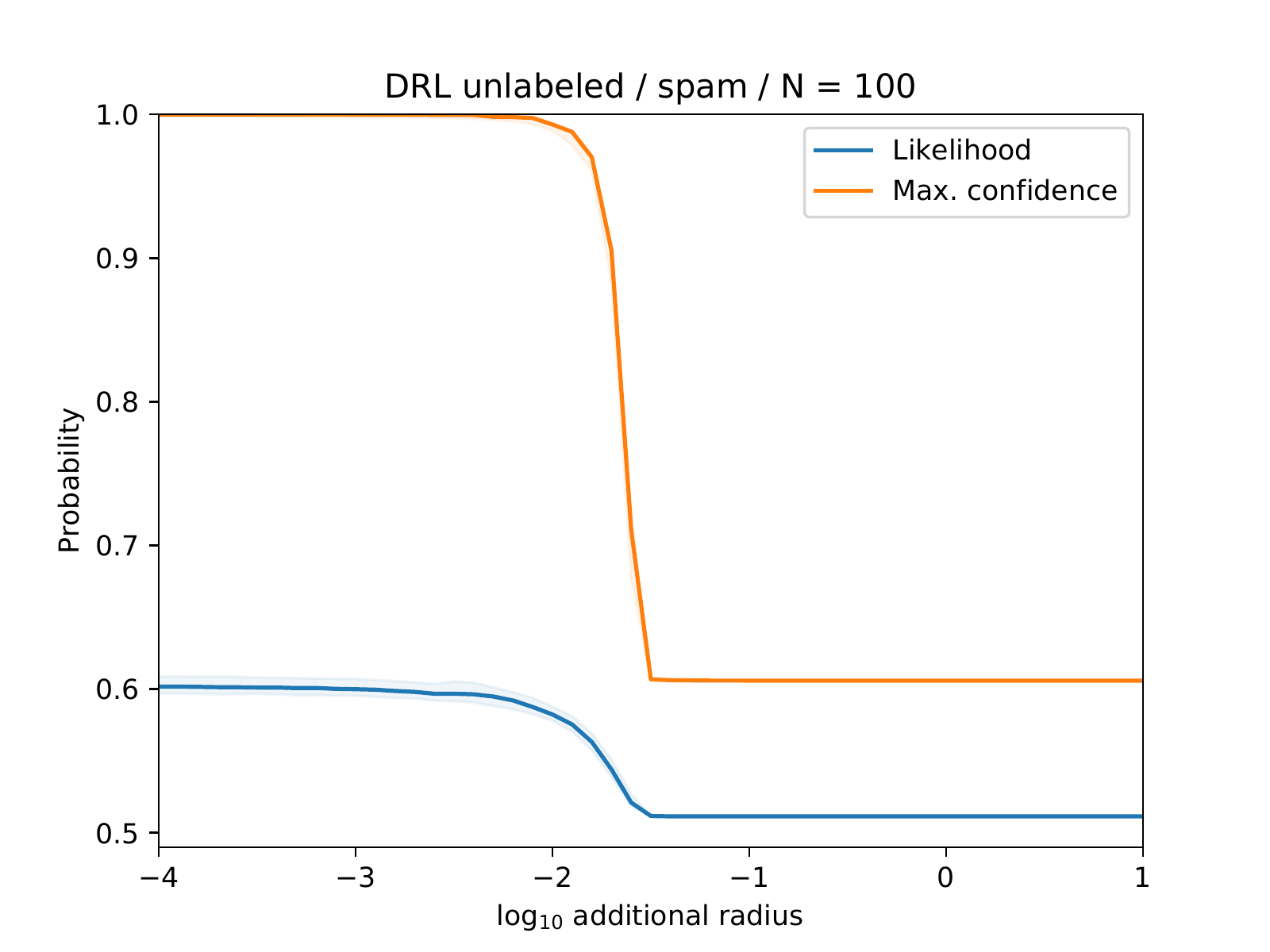}
\caption{Spam}
\end{subfigure}
\begin{subfigure}[b]{0.36\textwidth}
\includegraphics[width=\textwidth]{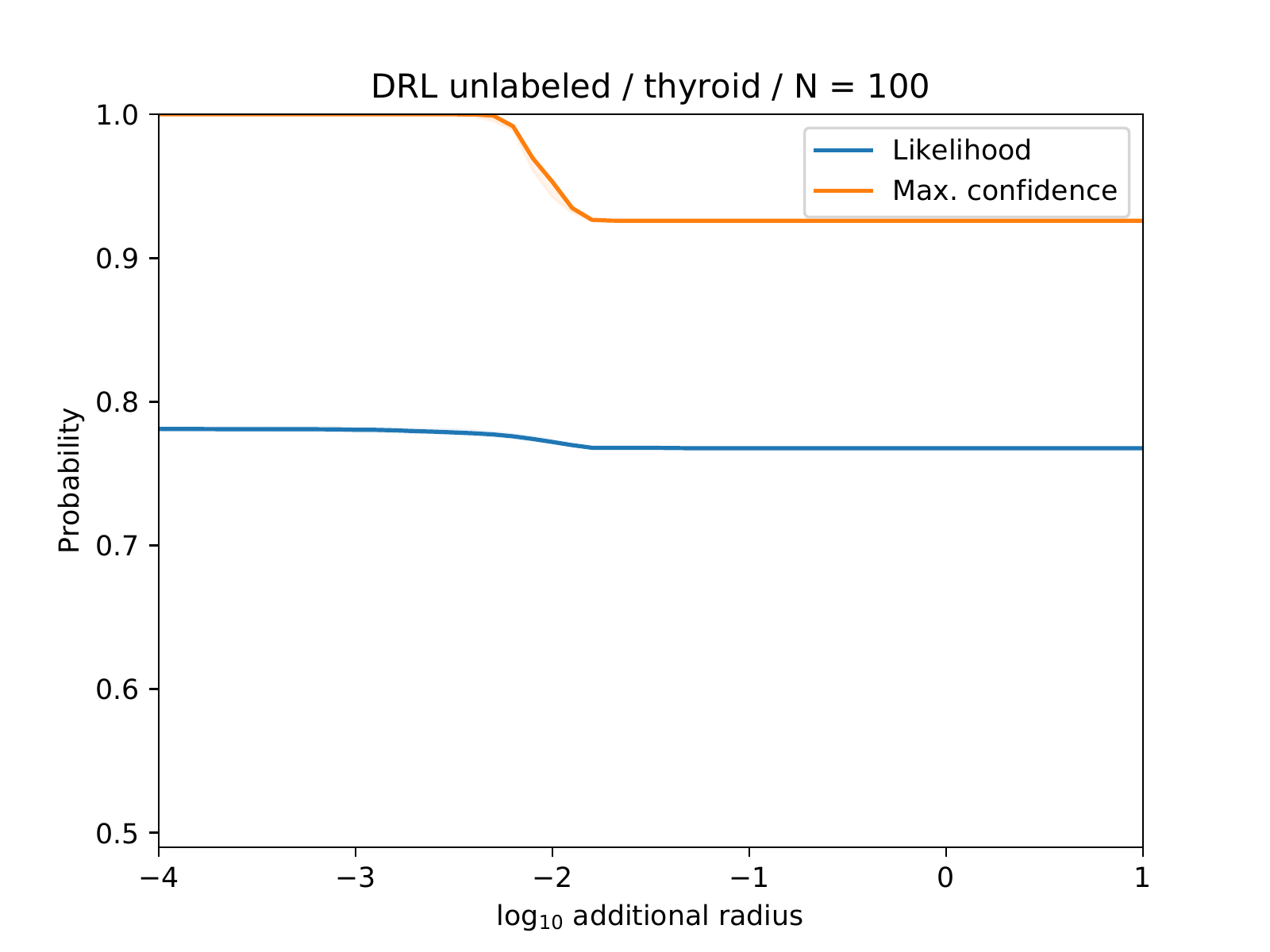}
\caption{Thyroid}
\end{subfigure}
\begin{subfigure}[b]{0.36\textwidth}
\includegraphics[width=\textwidth]{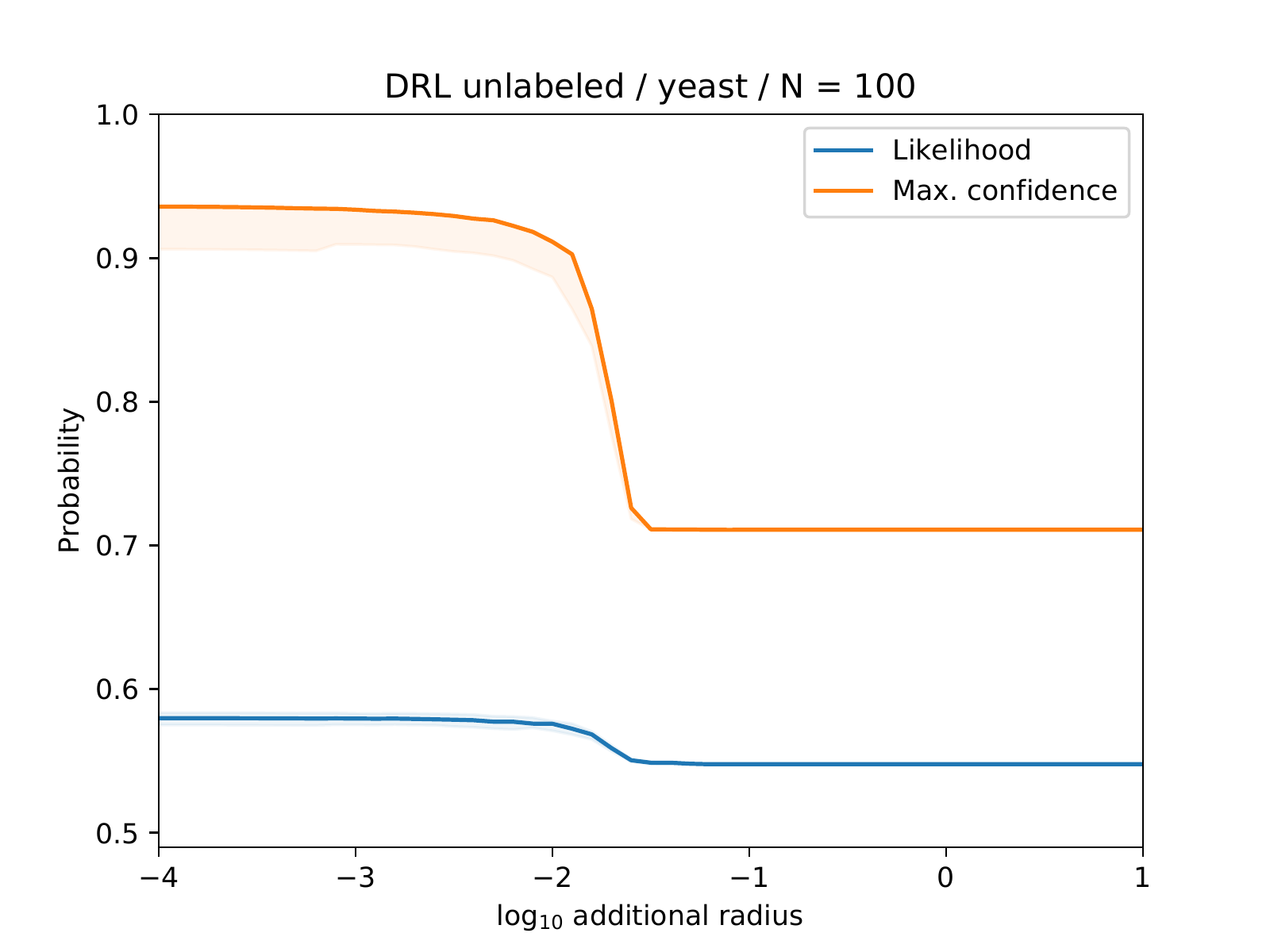}
\caption{Yeast}
\end{subfigure}
\caption{DRL with unlabeled data. Out-of-sample performance (likelihood) and maximum confidence vs. difference between the radius of robustness $\varepsilon$ and the minimal radius necessary for the decision set to be nonempty. Unlike with traditional Wasserstein DRL, here there is no apparent bias-variance tradeoff. Performance is flat out to a radius at which the confidence drops sharply.}
\label{fig:appendix-likelihood-and-confidence-vs-additional-radius}
\end{figure}

\subsection{Active learning}
\label{sec:supplement-active}

The following describes Table \ref{tbl:active-aulc} in Section \ref{sec:active-learning-empirical}. The table shows the area under the likelihood curve for model-change active learning heuristics and a random baseline applied to $14$ binary classification datasets. For each dataset and each trial, we sample an initial set $\hat{\Zspace}_l$ of $20$ labeled examples, with the remaining samples forming the unlabeled set $\hat{\Xspace}_u$. We then use this labeled set to learn a linear logistic regression model, solving
\begin{equation}
\hat{\theta} = \argmin_{\theta \in \Theta} \frac{1}{N_l} \sum_{i=1}^{N_l} \yvec_{\ell}^i \log(1 + \exp(-\langle \xvec_l^i, \theta\rangle)) + (1 - \yvec_{\ell}^i) \log(1 + \exp(\langle \xvec_l^i, \theta\rangle)) + \gamma \|\theta\|_2^2,
\end{equation}
where $\hat{\Zspace}_l = \{(\xvec_l^i, \yvec_{\ell}^i)\}_{i=1}^{N_l}$ and $\gamma = 0.001$.  Given $\hat{\theta}$, then, we apply each of the given active learning methods to select a new sample $\xvec_{\ast} \in \hat{\Xspace}_u$ to label. Specifically,
\begin{enumerate}
\item {\bf Random}. We sample $\xvec_{\ast}$ uniformly from $\hat{\Xspace}_u$.
\item {\bf EMC}. We compute for each $\xvec_u \in \hat{\Xspace}_u$
\begin{equation}
\hat{g}(\xvec_u) = \frac{2 \|\xvec_u\|_2}{(1 + \exp(-\langle \hat{\theta}, \xvec_u\rangle)) (1 + \exp(\langle \hat{\theta}, \xvec_u\rangle))},
\end{equation}
and select $\xvec_{\ast} = \argmax_{\xvec_u \in \hat{\Xspace}_u} \hat{g}(\xvec_u)$.
\item {\bf Min. MC}. We compute for each $\xvec_u \in \hat{\Xspace}_u$
\begin{equation}
\hat{g}(\xvec_u) = \min\left\{(1 + \exp(-\langle \xvec_u, \hat{\theta}\rangle))^{-1}, (1 + \exp(\langle \xvec_u, \hat{\theta}\rangle))^{-1}\right\},
\end{equation}
and select $\xvec_{\ast} = \argmax_{\xvec_u \in \hat{\Xspace}_u} \hat{g}(\xvec_u)$.
\item {\bf Max. MC}. We compute for each $\xvec_u \in \hat{\Xspace}_u$
\begin{equation}
\hat{g}(\xvec_u) = \max\left\{(1 + \exp(-\langle \xvec_u, \hat{\theta}\rangle))^{-1}, (1 + \exp(\langle \xvec_u, \hat{\theta}\rangle))^{-1}\right\}
\end{equation}
and select $\xvec_{\ast} = \argmax_{\xvec_u \in \hat{\Xspace}_u} \hat{g}(\xvec_u)$.
\item {\bf DR (strong)}. We take $100$ examples $\hat{\Xspace}_u^{\prime}$ uniformly at random from $\hat{\Xspace}_u$ and compute for each $\xvec_u \in \hat{\Xspace}_u^{\prime}$
\begin{equation}
\hat{g}(\xvec_u) = \inf_{\mu \in \Pset} \expect^{\mu} \frac{\delta_{\xvec_u}(\Xrv) \|\Xrv\|_2}{\hat{\varphi}(\Xrv) (1 + \exp(-\Yrv \langle \Xrv, \hat{\theta}\rangle))},
\end{equation}
with $\Pset = \ball_{\varepsilon}(\hat{\Ppr}_l) \cap \Uset(\Ppr_{\Xspace}, \overline{\pvec}_{\Yspace}, \underline{\pvec}_{\Yspace})$, $\hat{\varphi}(\xvec_u) = \frac{1}{N_u}$ for all $\xvec_u \in \hat{\Xspace}_u$, and $\overline{\pvec}_{\Yspace} = \underline{\pvec}_{\Yspace} = \pvec_{\Yspace}$ the true label probabilities from $\Ppr$. We select $\varepsilon$ via a binary search for the minimal radius at which the feasible set $\Pset$ is non-empty, setting $\varepsilon$ to be greater than this radius by a fixed margin $\Delta = 10^{-3}$. We  select $\xvec_{\ast} = \argmax_{\xvec_u \in \hat{\Xspace}_u} \hat{g}(\xvec_u)$.
\item {\bf DR (weak)}. We take $100$ examples $\hat{\Xspace}_u^{\prime}$ uniformly at random from $\hat{\Xspace}_u$ and compute for each $\xvec_u \in \hat{\Xspace}_u^{\prime}$
\begin{equation}
\hat{g}(\xvec_u) = \inf_{\mu \in \Pset} \expect^{\mu} \frac{\delta_{\xvec_u}(\Xrv) \|\Xrv\|_2}{\hat{\varphi}(\Xrv) (1 + \exp(-\Yrv \langle \Xrv, \hat{\theta}\rangle))},
\end{equation}
with $\Pset = \ball_{\varepsilon}(\hat{\Ppr}_l) \cap \Uset(\Ppr_{\Xspace}, \overline{\pvec}_{\Yspace}, \underline{\pvec}_{\Yspace})$, $\hat{\varphi}(\xvec_u) = \frac{1}{N_u}$ for all $\xvec_u \in \hat{\Xspace}_u$, and $[\underline{\pvec}_{\Yspace}^k, \overline{\pvec}_{\Yspace}^k]$ the $95\%$ confidence intervals estimated from the original set of $20$ labeled examples by the method of \citet{clopper1934use}. We select $\varepsilon$ as follows. We estimate $\varepsilon_{\ell}$ and $\varepsilon_h$ via a binary search for the minimal radius at which the feasible set $\Pset$ is non-empty, setting $\overline{\pvec}_{\Yspace}^k = \underline{\pvec}_{\Yspace}^k$, i.e. the lower bound of the interval, for $\varepsilon_{\ell}$, and vice versa, for $\varepsilon_h$. $\varepsilon$ is then $\max\{\varepsilon_{\ell}, \varepsilon_h\} + \Delta$ with $\Delta$ a fixed margin $\Delta = 10^{-3}$. We select $\xvec_{\ast} = \argmax_{\xvec_u \in \hat{\Xspace}_u} \hat{g}(\xvec_u)$.
\end{enumerate}
We then acquire the true label $\yvec_{\ast}$ corresponding to $\xvec_{\ast}$, remove $\xvec_{\ast}$ from $\hat{\Xspace}_u$, and add $(\xvec_{\ast}, \yvec_{\ast})$ to the labeled set $\hat{\Zspace}_l$. We repeat the process above (beginning by estimating $\hat{\theta}$) until $\Zspace_{\ell}$ contains $100$ samples.

Each time we estimate $\hat{\theta}$, we evaluate the performance of the learned classifier via the likelihood on the full dataset (including $\hat{\Zspace}_l$). Taking the median over $50$ trials, we obtain a likelihood curve (likelihood vs. number of labeled samples) for each active learning method, with the number of samples ranging from $20$ to $100$. The area under this curve is computed by the trapezoidal rule. The value shown in Table \ref{tbl:active-aulc} is $100$ times this area.

For the proposed distributionally robust heuristics, we solve the dual problem (Section \ref{sec:active-distributionally-robust}) using the Adam optimizer \citep{kingma2014adam}, setting $\beta_1 = 0.9$, $\beta_2 = 0.999$, $\epsilon = 10^{-8}$, and a batch size of $100$ and decreasing the learning rate by a factor of $10$ every $5000$ steps.

\vskip 0.2in
\bibliography{robust}

\end{document}